\documentclass[letterpaper]{article} 
\usepackage{arxiv}  
\usepackage[hyphens]{url}  
\usepackage{graphicx} 
\usepackage{natbib}  
\usepackage{caption} 
\usepackage{algorithm}
\usepackage{algorithmic}

\usepackage{newfloat}
\usepackage{listings}
\DeclareCaptionStyle{ruled}{labelfont=normalfont,labelsep=colon,strut=off} 
\floatstyle{ruled}
\newfloat{listing}{tb}{lst}{}
\floatname{listing}{Listing}

\usepackage{booktabs}

\usepackage{amsmath}
\usepackage{graphicx}
\usepackage{amssymb}
\usepackage{bbding}
\usepackage{booktabs}
\usepackage{tabularx}

\usepackage{colortbl}
\usepackage{xcolor}
\definecolor{mygray}{gray}{.9} 
\definecolor{myred}{rgb}{1,0,0} 
\definecolor{mygreen}{rgb}{0,1,0} 
\definecolor{myblue}{rgb}{0,0,1} 
\definecolor{mycyan}{rgb}{0,1,1} 
\definecolor{mymagenta}{rgb}{1,0,1} 
\definecolor{myyellow}{rgb}{1,1,0} 
\definecolor{myorange}{rgb}{1,0.647,0} 

\usepackage{subcaption}
\usepackage{multirow}
\usepackage{diagbox}
\usepackage{bm}
\usepackage{amsthm}

\definecolor{LightRed}{RGB}{255,182,193} 
\definecolor{LightBlue}{RGB}{173,216,230}

\theoremstyle{plain}
\newtheorem{theorem}{Theorem}
\newtheorem{proposition}{Proposition}
\newtheorem{lemma}{Lemma}

\theoremstyle{definition}

\newtheorem{assumption}{Assumption}
\theoremstyle{remark}
\newtheorem{remark}{Remark}

\usepackage{xspace}
\DeclareRobustCommand{\method}{FedACT\xspace}

\title{FedACT: Federated Adaptive Coordinate Trust Modulation for Robust Transformer Training under Data Heterogeneity}

\author{
    Shuai Li\textsuperscript{\rm 1,2,3},
    Qinglin Wang\textsuperscript{\rm 1,2,3},
    Ping Luo\textsuperscript{\rm 2,3},
    Jiahuan Wang\textsuperscript{\rm 2,3},
    Hongyang Hu\textsuperscript{\rm 1,2,3},
    Haotian Mo\textsuperscript{\rm 1,2,3}\\
    Yigui Feng\textsuperscript{\rm 1,2,3},
    Ziang Liu\textsuperscript{\rm 1,2,3},
    Qisong Xiao\textsuperscript{\rm 1,2,3},
    Jie Liu\textsuperscript{\rm 1,2,3}\thanks{Corresponding authors: Jie Liu (liujie@nudt.edu.cn) and Tao Sun (suntao.saltfish@outlook.com).},
    Tao Sun\textsuperscript{\rm 2,3}\footnotemark[1]
}
\affiliations{
    \textsuperscript{\rm 1}Laboratory of Digitizing Software for Frontier Equipment, National University of Defense Technology\\
    \textsuperscript{\rm 2}National Key Laboratory of Parallel and Distributed Computing, National University of Defense Technology\\
    \textsuperscript{\rm 3}College of Computer Science and Technology, National University of Defense Technology, Changsha 410073, China\\
    
}

\usepackage{bibentry}

\begin{document}

\maketitle

\begin{abstract}
Federated Transformer training increasingly relies on local AdamW, whose adaptive updates can provide much stronger local progress than SGD-based training. However, under heterogeneous client data, even globally corrected AdamW updates may remain highly uneven in coordinate-wise reliability. We refer to this phenomenon as \textbf{\textit{coordinate trust mismatch}}. Existing federated adaptive optimizers mainly address mismatch at the client-update or communication-round level, but still apply the corrected adaptive direction densely and uniformly across coordinates. In this paper, we propose \textbf{\method}, a global-aware coordinate trust modulation method for federated AdamW training. \method first forms a globally corrected adaptive direction and then reallocates update magnitudes according to a coordinate-wise trust score, assigning larger steps to coordinates jointly supported by local gradients and global correction, while preserving smaller non-zero updates on the remaining coordinates. Extensive experiments on federated vision Transformers, CNNs, LLM pre-training, and LLM fine-tuning show that \method consistently improves over strong federated adaptive baselines, with the largest gains on Transformer models under stronger data heterogeneity. Mechanism analyses further show that \method improves cross-client direction consistency, suggesting that coordinate-level trust allocation effectively complements round-level global-local correction. Code will be released.
\end{abstract}
\section{Introduction}
Adaptive optimization has become central to large-scale distributed training~\cite{kingma2014adam, loshchilov2017decoupled, zhou2020towards, zhang2024transformers, chen2023symbolic, jordan2024muon}, and this trend has naturally extended to federated learning~\cite{reddi2020adaptive, wang2022communication, sun2023efficient, liu2026fedadamw, liu2025fedmuon}. In particular, AdamW~\cite{loshchilov2017decoupled} is highly effective for Transformer architectures, whose optimization often benefits from heterogeneous effective learning rates rather than a single global learning-rate scale. Recent federated studies~\cite{sun2023efficient, liu2026fedadamw} further show that local AdamW can substantially outperform local SGD when training vision and language Transformers, motivating a shift from purely first-order local training toward adaptive local optimization in federated settings. This trend is also consistent with recent Hessian-based evidence showing that Transformers exhibit pronounced block heterogeneity, which favors coordinate- or block-wise adaptive scaling over single-rate first-order updates~\cite{zhang2024transformers, zhang2024adam}.

Recent works have begun to address adaptive local optimization in federated learning from the perspective of correction and stabilization~\cite{reddi2020adaptive, sun2023efficient, chen2025gradient, liu2026fedadamw}. Early adaptive federated optimization methods such as FedOpt~\cite{reddi2020adaptive} introduce server-side adaptivity, while later methods increasingly recognize that heterogeneity-aware correction becomes essential once adaptive optimization is used locally. In particular, FedLADA~\cite{sun2023efficient} amends local adaptive optimization with a global offset-like correction signal, FAdamGC~\cite{chen2025gradient} injects correction in a manner that respects the moment structure of adaptive methods, and FedAdamW~\cite{liu2026fedadamw} combines global-local alignment with second-moment synchronization to stabilize local AdamW under heterogeneous data. These developments suggest that global correction is a key ingredient for robust federated AdamW training.

However, existing adaptive FL methods still share an important limitation: after global correction, the adaptive local update remains \textit{\textbf{dense and uniform across coordinates}}. Once a global correction signal has been incorporated, all coordinates of the corrected local direction are implicitly treated as equally trustworthy. We argue that this assumption is unnecessarily strong under non-IID federated training, especially for Transformer models, where AdamW-style local updates are dense and highly non-uniform across coordinates, layers, and modules. In practice, some coordinates are jointly supported by the current stochastic gradient and the global correction signal, whereas others may be dominated by local noise, client-specific bias, or short-term overfitting. We refer to this fine-grained reliability imbalance as \textit{\textbf{coordinate trust mismatch}}.

This observation reveals a missing ingredient in federated AdamW training. Existing methods mainly correct local adaptive directions at the client-update or communication-round level, whereas recent centralized methods such as cautious optimizers~\cite{liang2024cautious} and MGUP~\cite{changmgup} show that update reliability can also vary substantially across coordinates. Nevertheless, these coordinate-wise modulation methods are developed for centralized stochastic optimization and rely only on local optimizer states or local gradients. They do not account for federated global information, client heterogeneity, or cross-client correction. This motivates a natural question: \textbf{\textit{once a globally corrected AdamW direction is formed in federated learning, is dense uniform updating still sufficient under non-IID data?}} We answer this question negatively and propose \textbf{\method}, an adaptive global-aware coordinate trust modulation method for federated AdamW training.

\method first forms a globally corrected AdamW direction and then reallocates update magnitudes across coordinates according to a global-aware trust score. Coordinates that are more strongly supported by both the current stochastic gradient and the global correction signal receive larger update magnitudes, while the remaining coordinates are still updated with smaller but non-zero magnitudes. In this way, \method does not replace round-level global-local correction; instead, it complements it with coordinate-level trust allocation. This design is particularly suitable for AdamW-style diagonal adaptive optimization, where coordinate-wise reliability is naturally meaningful because each coordinate has its own adaptive scale. Empirically, \method improves communication-round efficiency, training stability, and final performance over strong federated baselines, with especially clear gains on Transformer models under stronger data heterogeneity. 

Our main contributions are summarized as follows:
\begin{itemize}
    \item We identify \textbf{\textit{coordinate trust mismatch}} as a fine-grained failure mode in federated AdamW training: even after round-level global-local correction, different coordinates of the corrected adaptive update can remain highly uneven in reliability under heterogeneous data.
    \item We propose \textbf{\method}, an adaptive global-aware coordinate trust modulation method that first forms a globally corrected AdamW direction and then reallocates update magnitudes according to a coordinate-wise trust score. This design complements round-level correction with coordinate-level trust allocation.
    \item We conduct extensive experiments on federated vision Transformers, CNNs, LLM pre-training, and LLM fine-tuning. FedACT-AdamW consistently improves over strong federated adaptive baselines, with especially clear gains on Transformer models under stronger data heterogeneity. We further analyze direction consistency and coordinate-level trust diagnostics to explain the mechanism behind the improvement.
\end{itemize}
\section{Related Work}
\label{sec:related}

\begin{figure*}[!ht]
    \vskip -0.2cm
    \centering
    \begin{subfigure}{0.245\textwidth}
        \centering
        \includegraphics[width=1\textwidth]{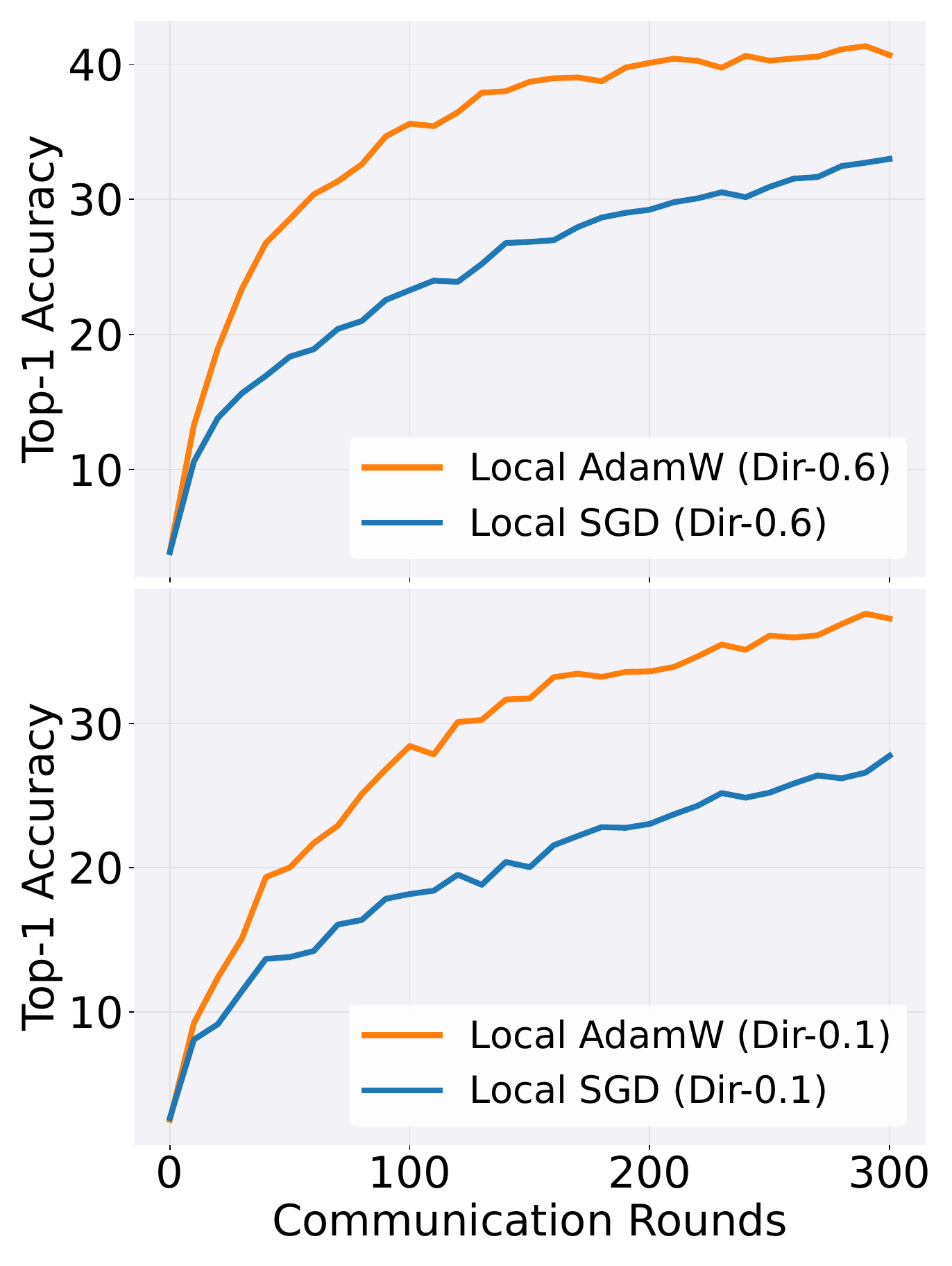}
        \vskip -0.2cm
        \caption{Top-1 Accuracy}
        \label{fig:motivation_adamw_vs_sgd}
    \end{subfigure}\hfill
    \begin{subfigure}{0.245\textwidth}
        \centering
        \includegraphics[width=1\textwidth]{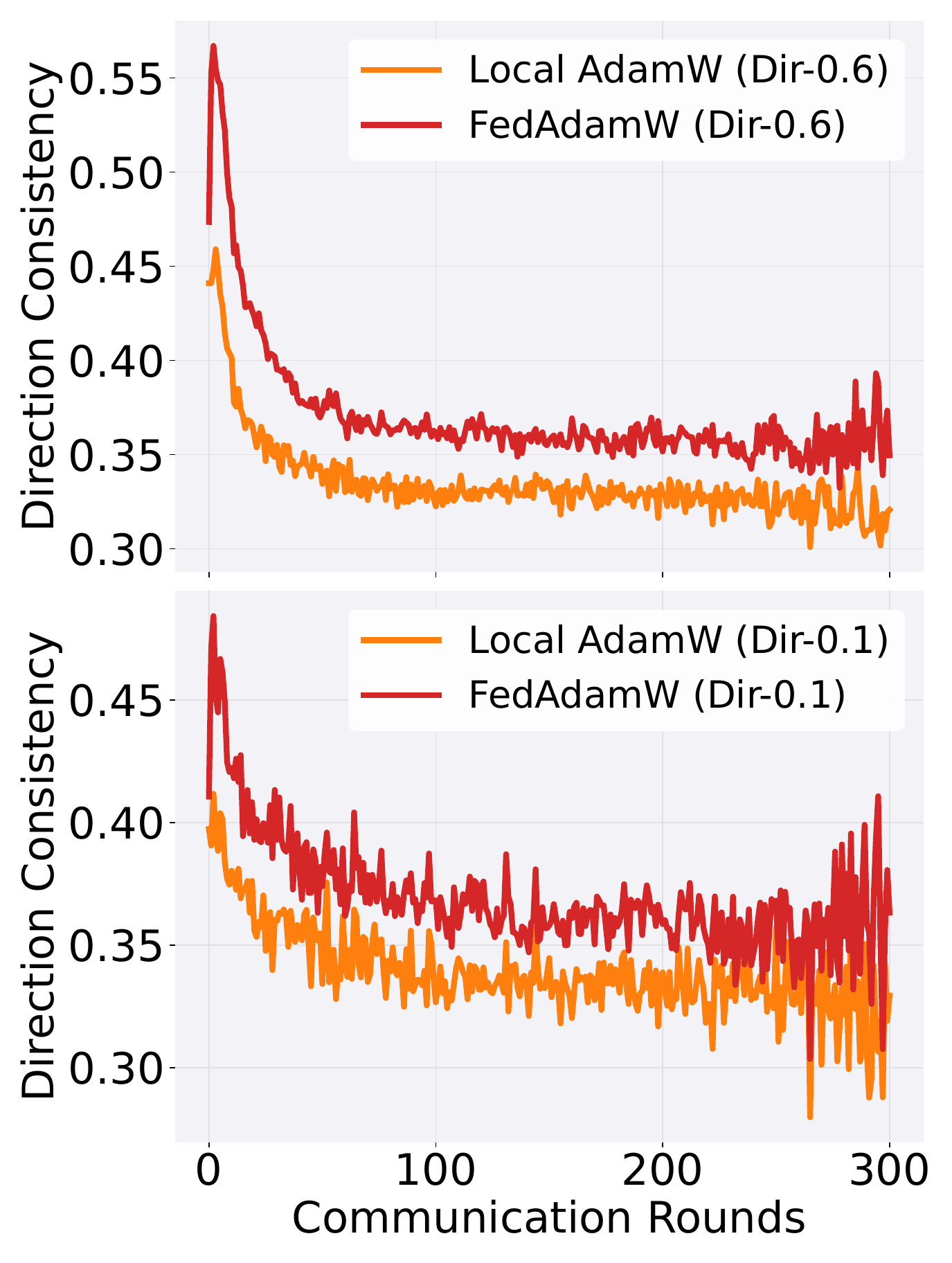}
        \vskip -0.2cm
        \caption{Direction Consistency}
        \label{fig:motivation_direction_consistency}
    \end{subfigure}\hfill
    \begin{subfigure}{0.245\textwidth}
        \centering
        \includegraphics[width=1\textwidth]{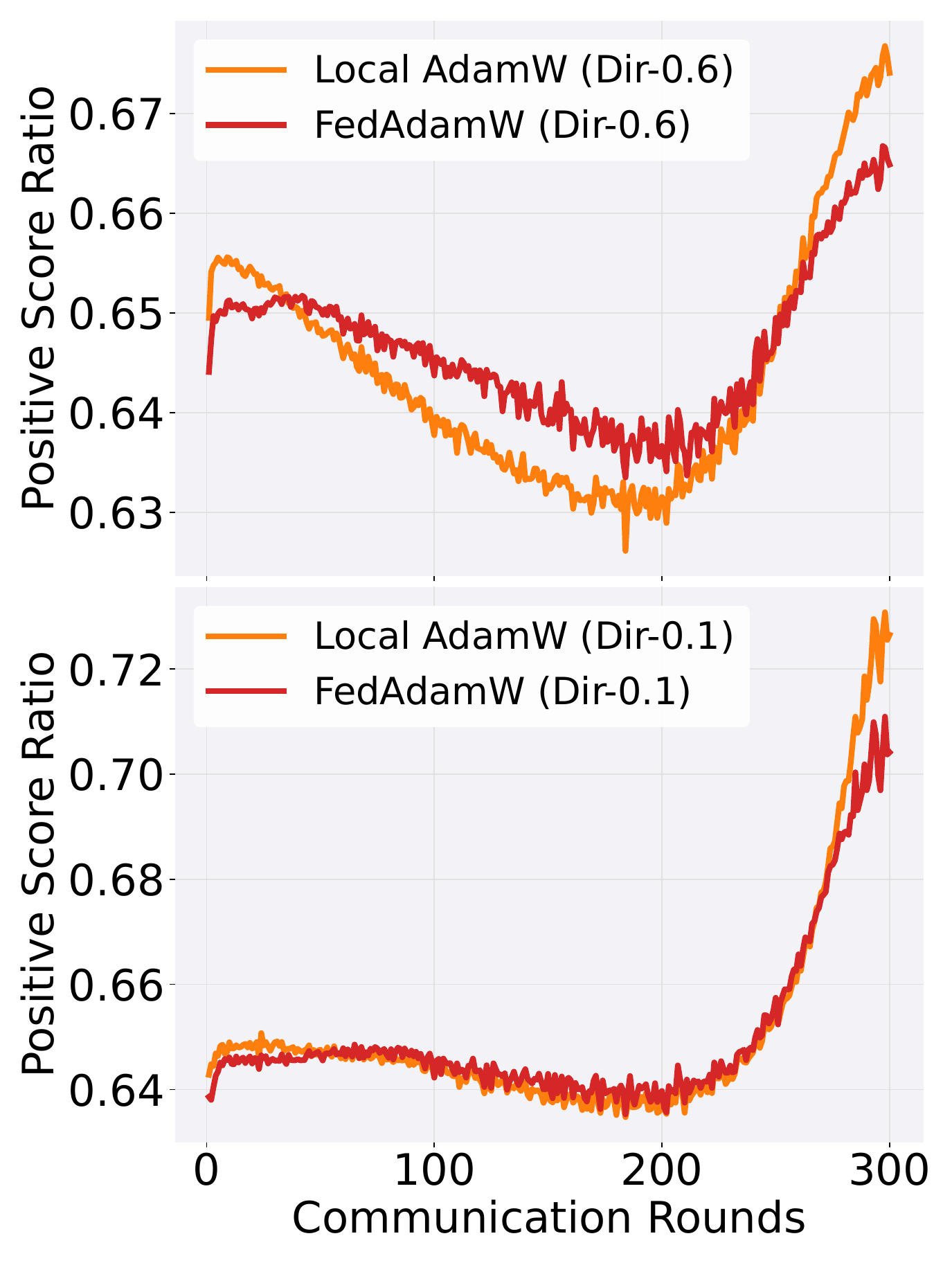}
        \vskip -0.2cm
        \caption{Positive Score Ratio}
        \label{fig:motivation_positive_score_ratio}
    \end{subfigure}
    \begin{subfigure}{0.245\textwidth}
        \centering
        \includegraphics[width=1\textwidth]{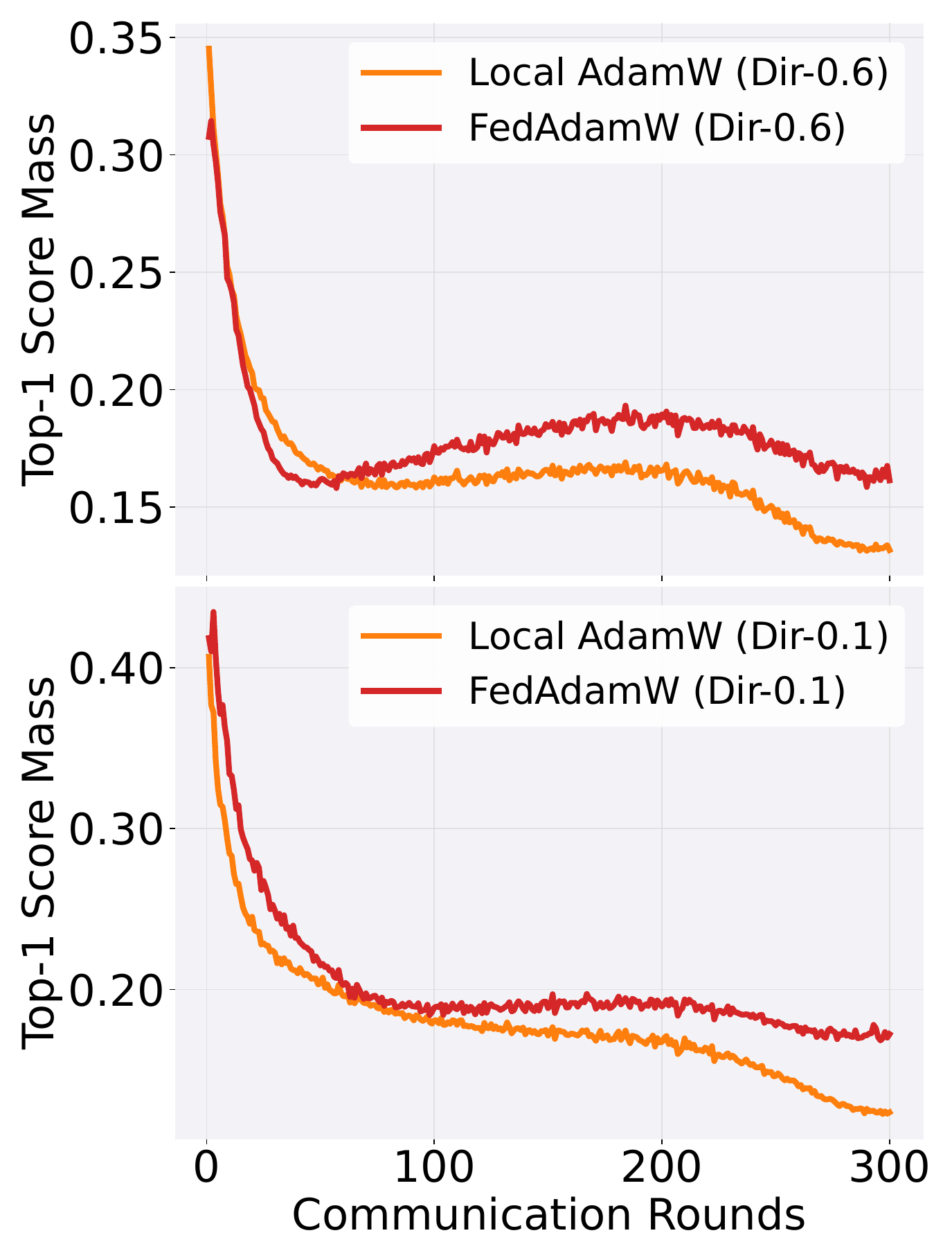}
        \vskip -0.2cm
        \caption{Top-1 Score Mass}
        \label{fig:motivation_top_1_score_mass}
    \end{subfigure}\hfill
    \vskip -0.2cm
    \caption{\textbf{Motivation for coordinate-level trust in federated AdamW training.} On CIFAR-100 with ViT-Tiny under 300 communication rounds, 100 clients, 10\% participation, and 50 local steps, (a) Local AdamW substantially outperforms Local SGD, confirming the effectiveness of adaptive local optimization for federated Transformer training; (b) FedAdamW improves round-level direction consistency over Local AdamW; however, (c) the positive score ratio and (d) the top-1 score mass show that coordinate trust remains highly non-uniform even after global correction.}
    \vskip -0.2cm
\end{figure*}

\paragraph{Federated Optimization under Data Heterogeneity.}
Federated learning reduces communication by allowing clients to perform multiple local updates before aggregation, as in FedAvg~\cite{mcmahan2017communication}. Under non-IID data, however, repeated local training can induce client drift and slow or destabilize convergence. Existing methods address this issue through local regularization, gradient correction, or improved aggregation. FedProx~\cite{li2020federated} uses a proximal term, SCAFFOLD~\cite{karimireddy2020scaffold} employs control variates, FedNova~\cite{wang2020tackling} normalizes heterogeneous local updates, and FedDyn~\cite{acar2021federated} and FedPD~\cite{zhang2021fedpd} improve local-global consistency through dynamic regularization or primal-dual formulations. Momentum- and lookahead-style methods such as SlowMo~\cite{wang2019slowmo}, FedADC~\cite{ozfatura2021fedadc}, FedCM~\cite{xu2021fedcm}, and FedACG~\cite{kim2024communication} further stabilize global progress using historical or aggregated update information. These methods mainly reduce mismatch at the client-update or communication-round level. In contrast, \method focuses on a finer-grained failure mode: even after round-level correction, different coordinates of a corrected local adaptive update can remain highly uneven in reliability under heterogeneous data.

\paragraph{Adaptive Optimization in Federated Learning.}
Adaptive optimizers such as Adam~\cite{kingma2014adam} and AdamW~\cite{loshchilov2017decoupled} are central to large-scale training because they provide coordinate-wise effective learning rates, which are particularly useful for Transformer models with heterogeneous optimization geometry across parameter blocks~\cite{zhang2024transformers, zhang2024adam}. In federated learning, early adaptive methods mainly focus on server-side adaptivity, such as the FedOpt family~\cite{reddi2020adaptive}, while MIME~\cite{karimireddy2020mime} shows that adapting centralized momentum or Adam to federated settings requires server-level statistics and correction. More recent methods study client-side adaptive optimization under data heterogeneity. FedLADA~\cite{sun2023efficient} amends local adaptive updates with a global offset-like correction, FAdamGC~\cite{chen2025gradient} injects drift compensation while respecting Adam's moment structure, and FedAdamW~\cite{liu2026fedadamw} stabilizes local AdamW through global-local alignment, second-moment synchronization, and decoupled weight decay. These methods show the importance of global correction and adaptive statistics, but they still apply the corrected adaptive direction densely and uniformly across coordinates. \method is complementary: it preserves the globally corrected AdamW direction while reallocating coordinate-wise update magnitudes according to global-aware trust.


\paragraph{Coordinate-Wise Trust Modulation in Centralized Optimization.}
Another line of work studies whether all coordinates of an optimizer update should be treated equally. Cautious optimizers~\cite{liang2024cautious} mask coordinates whose update direction conflicts with the current gradient, while MGUP~\cite{changmgup} softens this idea by assigning larger steps to reliable coordinates and smaller but non-zero steps to the remaining ones. More recently, Magma~\cite{joo2026surprising} extends update modulation to momentum-aligned stochastic masking at the block level. These methods show that update reliability can be highly non-uniform across coordinates or parameter blocks, and that selective modulation can improve optimization. However, they are developed for centralized stochastic optimization, where trust signals are computed from local optimizer states or local gradients. In contrast, federated training requires trust estimation to account for global information, client heterogeneity, and cross-client correction.

\section{Motivation: From Round-Level Alignment to Coordinate-Level Trust}
\label{sec:motivation}

Existing adaptive FL methods~\cite{sun2023efficient, liu2026fedadamw, chen2025gradient, liu2025fedmuon} have shown that global correction is important for mitigating heterogeneity-induced mismatch. In particular, FedAdamW~\cite{liu2026fedadamw} stabilizes local AdamW through global-local alignment and second-moment synchronization. However, this round-level correction leaves a finer question unanswered: \textbf{\textit{once a globally corrected AdamW direction is formed, is it still sufficient to apply it densely and uniformly across all coordinates under non-IID data?}} We answer this question by revisiting federated AdamW training from a coordinate-wise perspective.

\paragraph{Observation 1: AdamW is promising yet fragile in federated Transformer training.}
Figure~\ref{fig:motivation_adamw_vs_sgd} compares Local SGD and Local AdamW. Consistent with prior work~\cite{chen2025gradient, liu2026fedadamw}, Local AdamW achieves substantially faster progress than Local SGD, confirming that AdamW-style adaptive local optimization is worth preserving for federated Transformer training. However, its strong and dense local adaptivity can also amplify heterogeneity-induced mismatch: under non-IID data, each client repeatedly updates toward its private distribution, which may bias the local trajectory toward client-specific directions. Thus, the key issue is not whether local AdamW should be used, but how its adaptive direction should be corrected and trusted.

\paragraph{Observation 2: Round-level alignment is necessary but not sufficient.}
Figure~\ref{fig:motivation_direction_consistency} compares Local AdamW and FedAdamW using the \textit{direction consistency} metric 
{\small
\begin{equation}
    \operatorname{DC}^r = \frac{1}{S}\sum_{i \in S_r}  \cos(\Delta_i^r, \bar{\Delta}^r) = \frac{1}{S}\sum_{i \in S_r} \cos \left(\Delta_i^r, \frac{1}{s}\sum_{j \in S_r} \Delta_j^r\right)
\end{equation}
}
where $\Delta_i^r$ is the accumulated local update of client $i$ in round $r$, and $\bar{\Delta}^r$ is the aggregated round-level update. 

FedAdamW improves direction consistency over Local AdamW, confirming the value of global-local alignment. Nevertheless, residual inconsistency remains non-negligible, especially under stronger heterogeneity. This suggests that even after round-level correction, the corrected AdamW updates may still contain unreliable components that differ across clients. However, round-level direction consistency only measures the aggregate update direction; it does not reveal which parts of the corrected update remain unreliable. This motivates a finer coordinate-wise diagnosis.

\paragraph{Observation 3: Coordinate trust remains highly non-uniform even after global correction.}
To inspect the corrected AdamW direction at a finer granularity, we examine the coordinate-wise trust score
\begin{equation}
    s_{i,j}^{r,k}=u_{i,j}^{r,k}g_{i,j}^{r,k}
\end{equation}
where $u_{i,j}^{r,k}$ is the $j$-th coordinate of the globally corrected AdamW direction on client $i$ at local step $k$ of round $r$, and $g_{i,j}^{r,k}$ is the corresponding stochastic gradient coordinate. A positive score indicates local alignment between the corrected direction and the current gradient, while a larger positive value suggests stronger coordinate-level support. 

We summarize the score distribution using two statistics over the observed client-step pairs $\mathcal{O}_r$: the \textit{positive score ratio}, which measures the fraction of coordinates with $s_{i,j}^{r,k}>0$, and the \textit{top-$p$ positive score mass}, which measures how much positive support is concentrated in the top-scoring coordinates. As shown in Figure~\ref{fig:motivation_positive_score_ratio} and~\ref{fig:motivation_top_1_score_mass}, many coordinates are positively aligned, but the positive support is highly concentrated in a small subset of coordinates. This concentration remains pronounced even under FedAdamW, indicating that global correction reduces round-level mismatch but does not eliminate coordinate-level reliability imbalance.

These observations reveal a fine-grained failure mode, which we call \textit{\textbf{coordinate trust mismatch}}: even after round-level global-local correction, different coordinates of a corrected AdamW update remain highly uneven in reliability under non-IID data. Some coordinates are jointly supported by the current stochastic gradient and the global correction signal, whereas others are weakly supported or dominated by client-specific noise. Applying the corrected direction uniformly across all coordinates may therefore over-trust a broad tail of unreliable coordinates.

This motivates \textbf{\method}, which complements round-level global-local correction with coordinate-level trust allocation. Instead of replacing the corrected AdamW direction, \method assigns larger update magnitudes to high-trust coordinates and smaller but non-zero magnitudes to the remaining coordinates, thereby preserving adaptive local optimization while reducing the influence of unreliable coordinate components.
\section{Methodology: FedACT}


\subsection{From Corrected Adaptive Updates to Coordinate Trust}
\label{subsec:fedact_core}

We consider the standard federated optimization problem
\begin{equation}
    f(x)=\frac{1}{N}\sum_{i=1}^{N} f_i(x), \quad f_i(x)=\mathbb{E}_{\xi_i\sim\mathcal D_i}[F_i(x;\xi_i)]
    \label{eq:objective}
\end{equation}
where $N$ is the number of clients, $f_i$ is the local objective on client $i$, $\mathcal D_i$ is the local data distribution, and $x\in\mathbb R^d$ is the model parameter. In each communication round $r$, the server samples clients $S_r$, broadcasts the global model and correction statistics, and each selected client performs $K$ local updates before aggregation.

At local step $k$ on client $i$, let $g_i^{r,k}$ denote the stochastic gradient. Following AdamW-style local optimization, the client maintains
\begin{equation}
\begin{aligned}
    m_i^{r,k} &= \beta_1 m_i^{r,k-1} + (1-\beta_1) g_i^{r,k}, \\
    v_i^{r,k} &= \beta_2 v_i^{r,k-1} + (1-\beta_2)(g_i^{r,k}\odot g_i^{r,k}),
\end{aligned}
    \label{eq:adam_moments}
\end{equation}
with bias-corrected moments $\hat m_i^{r,k}$ and $\hat v_i^{r,k}$. The local AdamW direction is
\begin{equation}
    u_{i,\mathrm{loc}}^{r,k} = \frac{\hat m_i^{r,k}}{\sqrt{\hat v_i^{r,k}}+\epsilon}.
    \label{eq:local_adaptive_direction}
\end{equation}


Under non-IID data, this direction can be biased toward the client's private distribution. We therefore incorporate the server-provided global correction signal $\Delta_G^{r-1}$ and form the corrected AdamW direction
\begin{equation}
    u_i^{r,k} = (1-\rho)\,u_{i,\mathrm{loc}}^{r,k} + \rho\,\Delta_G^{r-1}, \quad \rho\in[0,1],
    \label{eq:global_aware_direction}
\end{equation}
where $\rho$ controls the strength of global guidance.

To estimate coordinate-level reliability, \method computes the trust score
\begin{equation}
    s_i^{r,k} = u_i^{r,k}\odot g_i^{r,k}.
    \label{eq:trust_score}
\end{equation}
A large positive score indicates that the globally corrected AdamW direction is well aligned with the current stochastic gradient. Since the score is computed from the corrected direction, \method favors coordinates that are jointly supported by local adaptive information and the federated global correction signal.


\subsection{Global-Aware Coordinate Trust Modulation---ACT}
\label{subsec:trust_modulation}

Given $s_i^{r,k}\in\mathbb R^d$, let $I_{\mathrm{top}}(s_i^{r,k})$ denote the indices of the top $\lfloor \tau d\rfloor$ entries, where $\tau\in(0,1)$ is the trust ratio. \method constructs
\begin{equation}
    \phi_{i,j}^{r,k} =
    \begin{cases}
        \alpha, & j\in I_{\mathrm{top}}(s_i^{r,k}),\\
        \gamma, & j\notin I_{\mathrm{top}}(s_i^{r,k}),
    \end{cases}
    \qquad \alpha>\gamma>0.
    \label{eq:fedact_modulation}
\end{equation}
The resulting local update is
\begin{equation}
    x_i^{r,k+1} = x_i^{r,k} - \eta\big(\phi_i^{r,k}\odot u_i^{r,k}\big) - \eta\lambda x_i^{r,k},
    \label{eq:fedact_update}
\end{equation}
where $\eta$ is the local learning rate and $\lambda$ is the decoupled weight decay coefficient.

Thus, \method preserves the dense globally corrected AdamW direction, but reallocates its coordinate-wise magnitudes: high-trust coordinates receive stronger updates, while the remaining coordinates are softly attenuated rather than discarded. This avoids the instability of hard masking while reducing the influence of weakly supported coordinates. Unless otherwise stated, we use the MGUP-style default $\alpha=1/\tau$ and $\gamma=\tau$.


\paragraph{Special cases.} When $\rho=0$, \method reduces to a local-only coordinate trust policy. When $\phi_i^{r,k}=\mathbf{1}$, it reduces to a dense corrected AdamW update. If $\alpha=\gamma$, the modulation becomes a uniform rescaling that can be absorbed into the learning rate. Algorithm~\ref{alg:algorithm_fedact} summarizes the full procedure.


\begin{algorithm}[!ht]
	\caption{\texttt{\method} Algorithm}
    \label{alg:algorithm_fedact}
	\begin{algorithmic}[1]
		\STATE {\textbf{Initial} model $\boldsymbol{x}^0$, $\beta_1=0.9, \beta_2=0.999, \epsilon=10^{-8}$, time step $t \leftarrow 0$, global correction $\boldsymbol{\Delta}_G^0 \leftarrow \boldsymbol{0}$, the number of all clients $N$, selected clients $S$, weight decay $\lambda$, \texttt{ACT} ratio $\tau$.}
		\FOR{$r = 1, \dots, R$}
		\FOR{each selected client $i \in \{1, \dots, S\}$ in parallel}
		\STATE $\boldsymbol{x}_{i}^{r,0} \gets \boldsymbol{x}^r$, $\boldsymbol{m}^{r,0}_{i} \gets \boldsymbol{0}$, $\boldsymbol{v}^{r,0}_{i} \gets \boldsymbol{\bar{v}}^{r}$;
		\FOR{$k = 1, \dots, K$}
		\STATE $t \gets t+1$;
		\STATE $\boldsymbol{g}^{r,k}_i \gets \nabla f_i(\boldsymbol{x}_i^{r,k-1} ; \xi_i)$;
		\STATE $\boldsymbol{m}^{r,k}_i = \beta_1 \boldsymbol{m}^{r,k-1}_i + \left(1-\beta_1\right)\boldsymbol{g}^{r,k}_i$;
		\STATE $\boldsymbol{v}^{r,k}_i = \beta_2 \boldsymbol{v}^{r,k-1}_i + \left(1-\beta_2\right)\boldsymbol{g}^{r,k}_i \odot \boldsymbol{g}^{r,k}_i$;
		\STATE \textbf{Bias correction}
		\STATE $\hat{\boldsymbol{m}}^{r,k}_i = \boldsymbol{m}^{r,k}_i /\left(1-\beta_1^{k}\right)$;
		\STATE $\hat{\boldsymbol{v}}^{r,k}_i = \boldsymbol{v}^{r,k}_i /\left(1-\beta_2^{t}\right)$;
		\STATE $\vartheta_{i}^{r,k} = 1 / ( \sqrt{\hat{\boldsymbol{v}}_{i}^{r,k}}+\epsilon)$;
		\STATE \textbf{Global-Aware Coordinate Trust Modulation}
		\STATE $\boldsymbol{u}_i^{r,k} \gets (1-\rho)\hat{\boldsymbol{m}}^{r,k}_i \odot \vartheta_{i}^{r,k} + \rho \boldsymbol{\Delta}_G^{r-1}$;
		\STATE $\boldsymbol{s}_i^{r,k} \gets \boldsymbol{u}_i^{r,k} \odot \boldsymbol{g}_i^{r,k}$;
		\STATE $\boldsymbol{\phi}_i^{r,k} \gets \texttt{ACT}(\boldsymbol{s}_i^{r,k}, \tau)$;

	 	\STATE \textbf{Update model parameters }\\
    	$\boldsymbol{x}_i^{r,k} = (1-\eta\lambda)\boldsymbol{x}_i^{r,k-1} - \eta(\boldsymbol{\phi}_i^{r,k} \odot \boldsymbol{u}_i^{r,k})$;
		\ENDFOR
		\STATE Communicate $( \boldsymbol{x}^{r,K}_i-\boldsymbol{x}^{r,0}_i, \boldsymbol{\bar{v}}_i=\boldsymbol{v}^{r,K}_i )$ to Server;
		\ENDFOR
        \STATE $\boldsymbol{\Delta}_G^r = \frac{-1}{SK\eta} \sum_{i=1}^S (\boldsymbol{x}^{r,K}_i-\boldsymbol{x}^{r,0}_i)$;
		\STATE $\boldsymbol{x}^{r+1} = \boldsymbol{x}^{r} + \frac{1}{S} \sum_{i=1}^S (\boldsymbol{x}^{r,K}_i-\boldsymbol{x}^{r,0}_i)$;
		\STATE $\boldsymbol{\bar{v}}^{r+1} = \frac{1}{S} \sum_{i=1}^S \boldsymbol{\bar{v}}_i$;
		\STATE Communicate $(\boldsymbol{x}^{r+1}, \boldsymbol{\bar{v}}^{r+1}, \boldsymbol{\Delta}_G^r )$ to Clients.
		\ENDFOR
	\end{algorithmic}
\end{algorithm}


\section{Convergence Analysis}
\label{sec:convergence}

We provide a nonconvex convergence analysis for \method, with the full proof deferred to the Appendix. Following standard analyses of local adaptive federated optimization, we consider $\lambda=0$ and full client participation. Let $T=RK$ be the total number of local update steps. We analyze the virtual averaged trajectory
{\small
\begin{equation}
    \bar{x}_t = \frac{1}{N}\sum_{i=1}^{N} x_{i,t}, \quad \bar{p}_t = \frac{1}{N}\sum_{i=1}^{N} p_{i,t}, \quad
    \bar{x}_{t+1} = \bar{x}_t - \eta \bar{p}_t ,
\label{eq:virtual_average}
\end{equation}
}where $p_{i,t}=\phi_{i,t}\odot u_{i,t}$ is the effective update direction on client $i$. The corrected AdamW direction is
\begin{equation}
    u_{i,t}=(1-\rho)H_t g_{i,t}+\rho \Delta_t ,
\label{eq:corrected_direction_theory}
\end{equation}
Here $H_t$ is the synchronized diagonal adaptive scaling matrix and $\Delta_t$ is the global correction signal.

We assume the global objective is lower bounded, each local objective is $L$-smooth, stochastic gradients are unbiased with bounded variance $\sigma^2$, client heterogeneity is bounded by $\zeta^2$, and stochastic gradients are uniformly bounded. We also use the predictable-preconditioner convention, where $H_t$ is formed before sampling current stochastic gradients and satisfies $\mu I\preceq H_t\preceq M I$.

\begin{theorem}[Convergence of FedACT]
\label{thm:fedact_convergence}
Suppose the above assumptions hold. Let $\lambda=0$ and let $T=RK$ be the total number of local update steps. Choose $T$-dependent parameters $\rho_T$ and $\tau_T$ such that
$0<\rho_T<\tau_T\le 1$, $\rho_T=\mathcal{O}(T^{-1/4})$, and $1-\tau_T=\mathcal{O}(T^{-1/4})$, and assume that $\tau_T-\rho_T$ is bounded away from zero. For
$\tau_T\le \phi_{i,t,j}\le 1/\tau_T$, if the stepsize satisfies the conditions specified in Appendix and $\eta=\Theta(1/\sqrt{T})$, then
\begin{equation}
    \frac{1}{T}\sum_{t=0}^{T-1}
    \mathbb{E}\left\|
    \nabla f(\bar{x}_t)
    \right\|^2
    =
    \mathcal{O}\left(\frac{1}{\sqrt{T}}\right)
    =
    \mathcal{O}\left(\frac{1}{\sqrt{RK}}\right).
\label{eq:main_convergence_bound}
\end{equation}
\end{theorem}

Theorem~\ref{thm:fedact_convergence} shows that FedACT achieves the standard nonconvex convergence rate $\mathcal{O}(1/\sqrt{RK})$, which is comparable to classical distributed SGD under the same smooth nonconvex setting.
\begin{table*}[!ht]
    \vskip -0.2cm
    \centering
    \caption{Accuracy (\%) of all methods on ViT-Tiny and Swin-Lite over 300 communication rounds under different datasets and heterogeneity settings (100 clients, 10\% participation, batch size 50, $K=50$).}
    \label{tab:main_results_vit_swin}
    \vskip -0.2cm
    \resizebox{1.0\linewidth}{!}{

{
\scriptsize
\setlength{\tabcolsep}{2.5pt}
\begin{tabular}{lcccccccccccccccccc}
  \toprule
  \multirow{3}{*}{\textbf{Method}}
  & \multicolumn{9}{c}{\textbf{ViT-Tiny}}
  & \multicolumn{9}{c}{\textbf{Swin-Lite}} \\
  \cmidrule(lr){2-10} \cmidrule(lr){11-19}
  & \multicolumn{3}{c}{\textbf{CIFAR-10}}
  & \multicolumn{3}{c}{\textbf{CIFAR-100}}
  & \multicolumn{3}{c}{\textbf{Tiny-ImageNet}}
  & \multicolumn{3}{c}{\textbf{CIFAR-10}}
  & \multicolumn{3}{c}{\textbf{CIFAR-100}}
  & \multicolumn{3}{c}{\textbf{Tiny-ImageNet}} \\
  \cmidrule(lr){2-4} \cmidrule(lr){5-7} \cmidrule(lr){8-10}
  \cmidrule(lr){11-13} \cmidrule(lr){14-16} \cmidrule(lr){17-19}
  & Dir-0.6 & Dir-0.3 & Dir-0.1
  & Dir-0.6 & Dir-0.3 & Dir-0.1
  & Dir-0.6 & Dir-0.3 & Dir-0.1
  & Dir-0.6 & Dir-0.3 & Dir-0.1
  & Dir-0.6 & Dir-0.3 & Dir-0.1
  & Dir-0.6 & Dir-0.3 & Dir-0.1 \\
  \midrule
  FedAvg
    & 58.24 & 55.90 & 44.48
    & 33.24 & 32.99 & 27.77
    & 18.00 & 16.90 & 15.21
    & 53.89 & 51.63 & 39.25
    & 31.99 & 29.86 & 22.91
    & 13.99 & 13.45 & 10.72 \\
  SCAFFOLD
    & 57.47 & 54.71 & 45.10
    & 32.90 & 32.61 & 27.65
    & 17.84 & 17.64 & 15.41
    & 53.89 & 51.47 & 39.79
    & 31.70 & 29.46 & 23.15
    & 14.01 & 13.07 & 10.74 \\
  FedProx
    & 57.28 & 54.05 & 46.13
    & 32.78 & 32.49 & 27.82
    & 18.00 & 16.96 & 15.35
    & 54.08 & 51.78 & 38.94
    & 31.91 & 29.74 & 23.29
    & 14.04 & 13.35 & 10.71 \\
  LocalAdam
    & 63.26 & 59.89 & 49.01
    & 38.78 & 36.72 & 29.64
    & 19.34 & 17.59 & 15.35
    & 71.04 & 67.28 & 56.92
    & 47.49 & 46.36 & 40.22
    & 27.54 & 25.25 & 21.79 \\
  FedAdam
    & 53.24 & 51.97 & 43.57
    & 30.17 & 30.12 & 25.52
    & 16.80 & 16.62 & 14.50
    & 51.90 & 50.68 & 37.39
    & 27.66 & 25.39 & 18.65
    & 11.47 & 10.49 & 7.65 \\
  FedLADA
    & 65.69 & 61.64 & 50.32
    & 37.68 & 37.01 & 33.87
    & 23.39 & 22.35 & 20.33
    & 59.35 & 53.06 & 39.24
    & 34.26 & 31.03 & 22.60
    & 15.74 & 15.04 & 10.55 \\
  LocalAdamW
    & 70.21 & 68.41 & 59.70
    & 41.34 & 40.62 & 37.62
    & 26.07 & 25.19 & 23.59
    & 70.47 & 67.54 & 56.19
    & 47.46 & 46.32 & 40.22
    & 27.25 & 25.62 & 21.34 \\
  FedAdamW
    & 73.06 & 72.39 & 63.70
    & 41.76 & 41.54 & 39.08
    & 26.67 & 25.18 & 22.82
    & 73.50 & 68.95 & 58.54
    & 49.79 & 48.59 & 42.75
    & 31.32 & 29.10 & 24.96 \\
  \rowcolor{mygray}
  \textbf{Ours}
    & \textbf{74.04} & \textbf{73.39} & \textbf{68.05}
    & \textbf{43.58} & \textbf{42.83} & \textbf{42.03}
    & \textbf{27.18} & \textbf{26.00} & \textbf{24.66}
    & \textbf{78.08} & \textbf{74.57} & \textbf{61.94}
    & \textbf{52.23} & \textbf{51.40} & \textbf{48.50}
    & \textbf{35.62} & \textbf{33.64} & \textbf{30.15} \\
  \bottomrule
\end{tabular}
}
    }
    \vskip -0.2cm
\end{table*}

\section{Experiments}



\subsection{Experimental Setup}

\textbf{Datasets.}
We evaluate \method on vision and language tasks. For image classification, we use CIFAR-10, CIFAR-100~\cite{krizhevsky2009learning}, and Tiny ImageNet~\cite{le2015tiny}, with non-IID client data simulated by Dirichlet partitioning~\cite{hsu2019measuring}. For LLM tasks, we use C4-en~\cite{raffel2020exploring} for pre-training, Alpaca-GPT4 for federated instruction tuning (FedIT), and HH-RLHF for federated value alignment (FedVA).

\textbf{Model Architectures.}
For vision tasks, we use ResNet-18~\cite{he2016deep} as a representative CNN, and lightweight Vision Transformer~\cite{dosovitskiy2020image} and Swin Transformer~\cite{liu2021swin} variants for federated Transformer training. For LLM pre-training, we use Llama2~\cite{touvron2023llama} at 60M/130M/250M scales. For LLM fine-tuning, we use LoRA~\cite{hu2022lora} with Llama2-7B for FedIT and a Wizard-Vicuna-style 7B model for FedVA.


\textbf{Comparison Methods.}
We compare \method with FedAvg~\cite{mcmahan2017communication}, FedProx~\cite{li2020federated}, SCAFFOLD~\cite{karimireddy2020scaffold}, LocalAdam, FedAdam~\cite{reddi2020adaptive}, FedLADA~\cite{sun2023efficient}, LocalAdamW, and FedAdamW~\cite{liu2026fedadamw}, covering first-order, drift-corrected, server-side adaptive, local adaptive, and corrected local AdamW training.

More training details are provided in the Appendix due to space limitations.

\subsection{Experimental Results}


\paragraph{Results on Vision Transformers.} Table~\ref{tab:main_results_vit_swin} reports the results on vision Transformer training. \method achieves the best accuracy across all settings, improving over FedAdamW by about 3.1 percentage points on average over the 12 CIFAR settings. The gains are particularly clear under strong heterogeneity: under Dir-0.1, \method improves over FedAdamW by 4.35 points on ViT-Tiny/CIFAR-10, 2.95 points on ViT-Tiny/CIFAR-100, 3.40 points on Swin-Lite/CIFAR-10, and 5.75 points on Swin-Lite/CIFAR-100. On Tiny-ImageNet, the improvement is especially pronounced for Swin-Lite, where \method outperforms FedAdamW by 4.30/4.54/5.19 points under Dir-0.6/0.3/0.1, respectively. Figure~\ref{fig:acc_vit_swin} further shows that \method converges faster and reaches higher final accuracy than FedAdamW, especially under stronger heterogeneity. These results support our hypothesis that coordinate-level trust allocation is particularly useful for federated Transformer training, where dense local AdamW updates remain highly non-uniform in coordinate reliability under heterogeneous data.

\begin{figure*}[!t]
    \centering
    \begin{subfigure}{0.248\textwidth}
        \centering
        \includegraphics[width=1\textwidth]{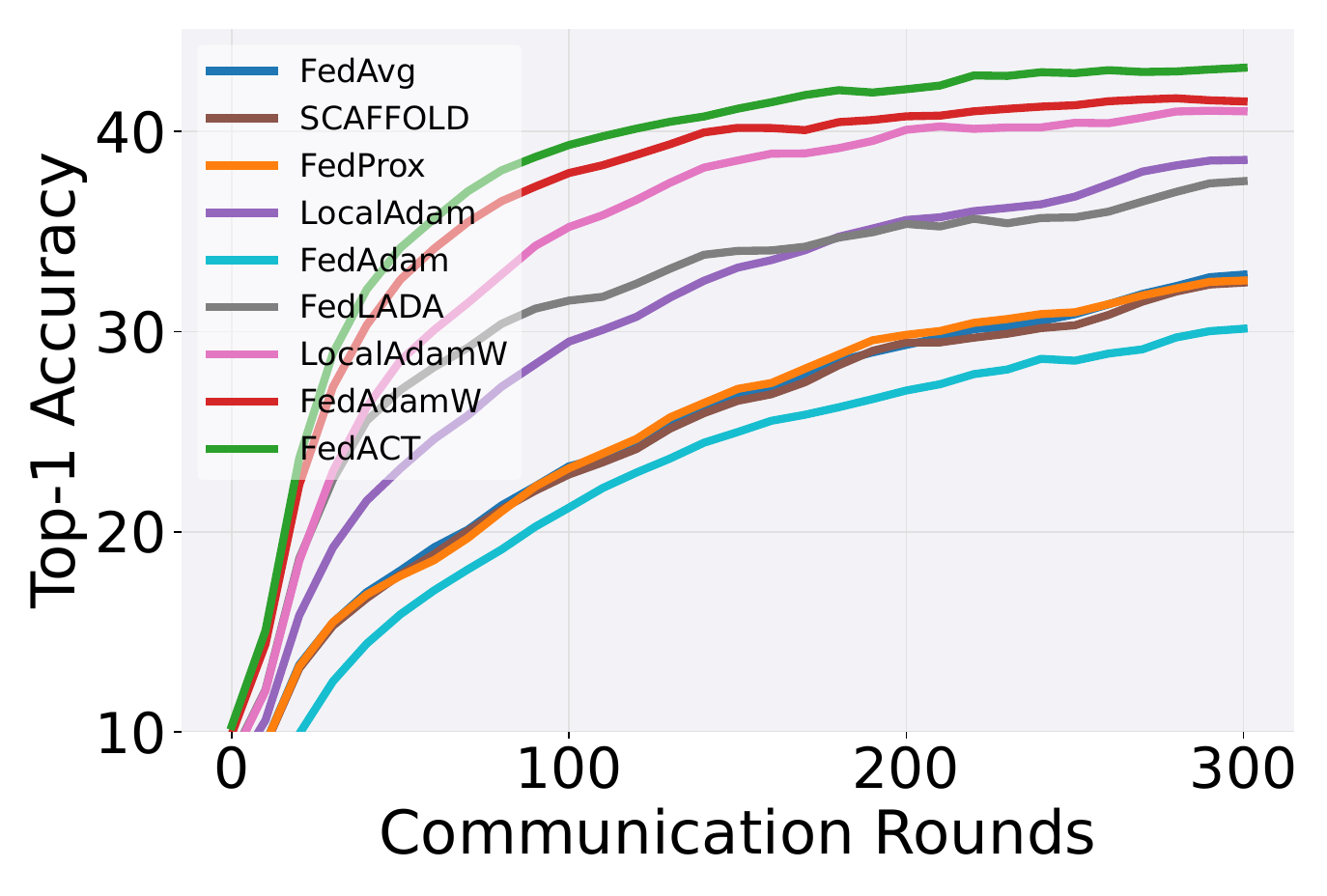}
        \vskip -0.2cm
        \caption{CIFAR-100, Dir-0.6}
    \end{subfigure}\hfill
    \begin{subfigure}{0.248\textwidth}
        \centering
        \includegraphics[width=1\textwidth]{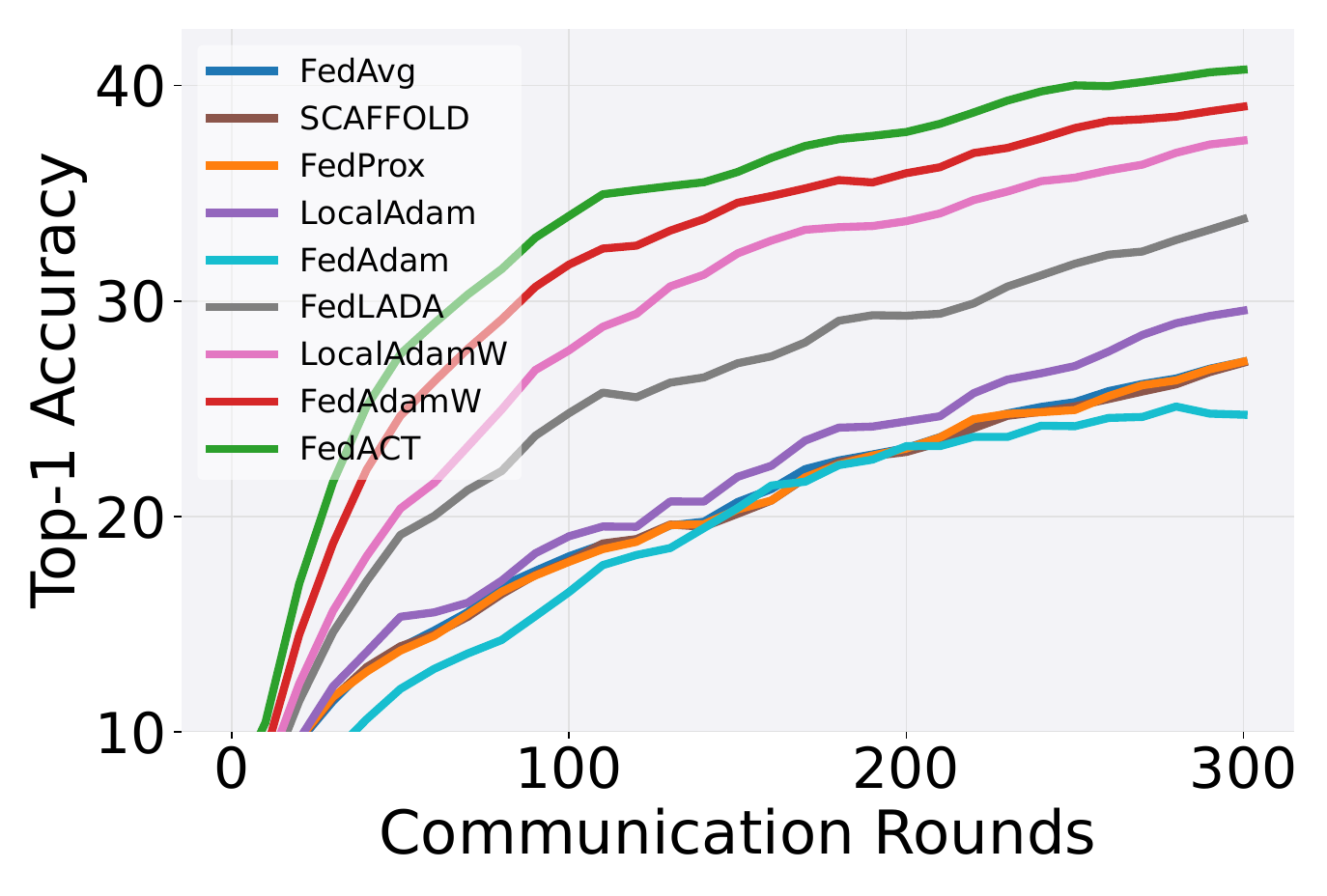}
        \vskip -0.2cm
        \caption{CIFAR-100, Dir-0.1}
    \end{subfigure}\hfill
    \begin{subfigure}{0.248\textwidth}
        \centering
        \includegraphics[width=1\textwidth]{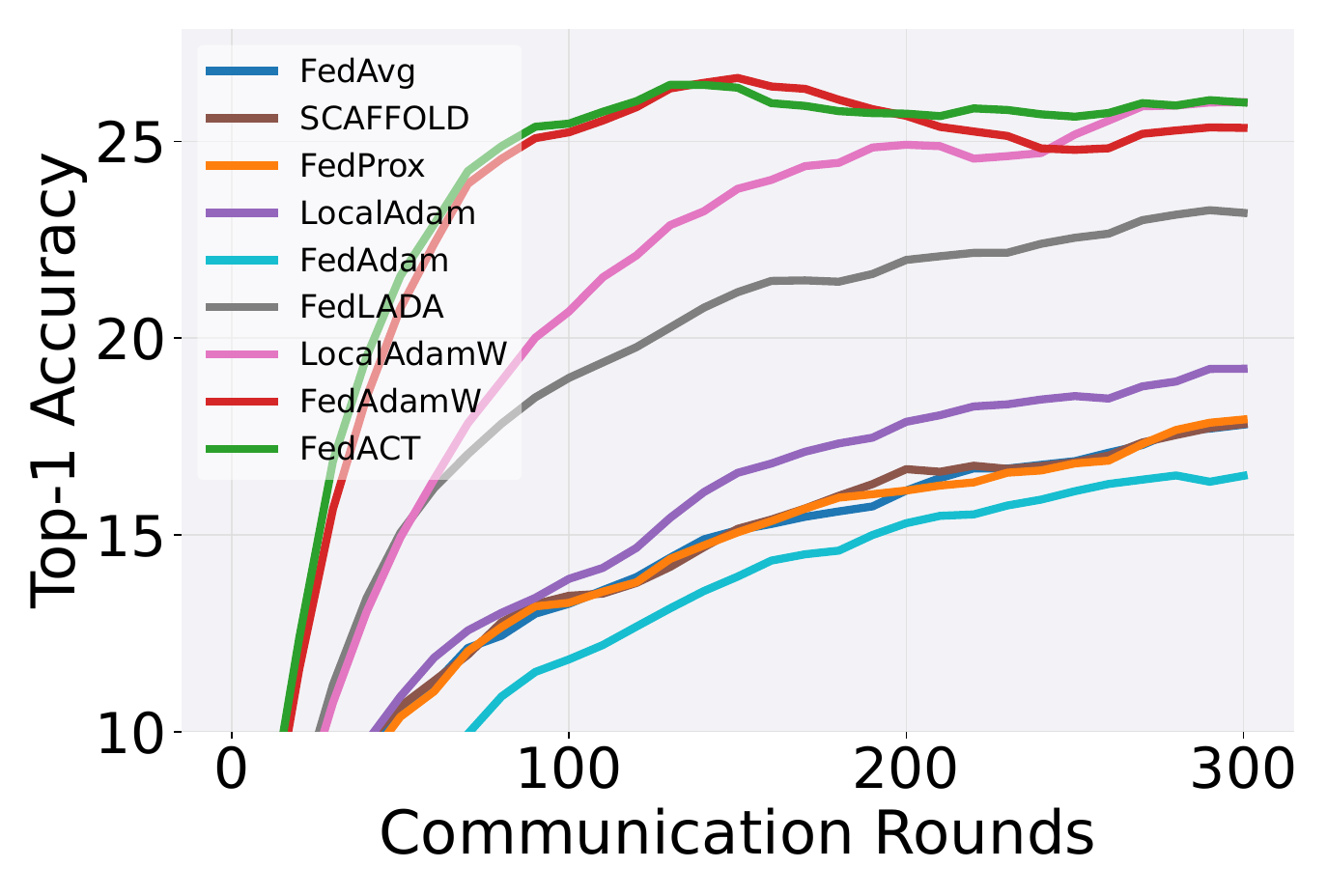}
        \vskip -0.2cm
        \caption{Tiny-ImageNet, Dir-0.6}
    \end{subfigure}
    \begin{subfigure}{0.248\textwidth}
        \centering
        \includegraphics[width=1\textwidth]{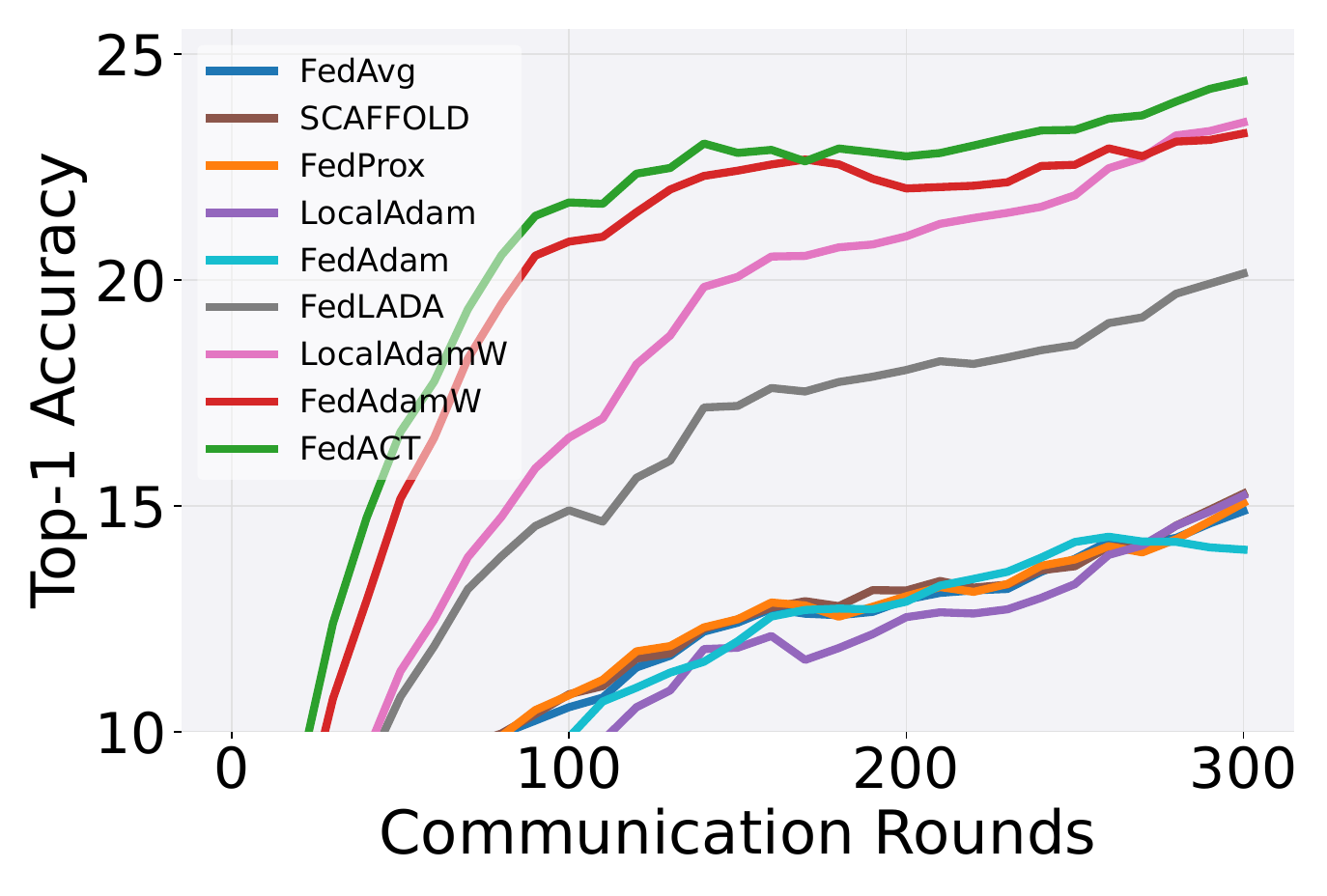}
        \vskip -0.2cm
        \caption{Tiny-ImageNet, Dir-0.1}
    \end{subfigure}\vfill
    \begin{subfigure}{0.248\textwidth}
        \centering
        \includegraphics[width=1\textwidth]{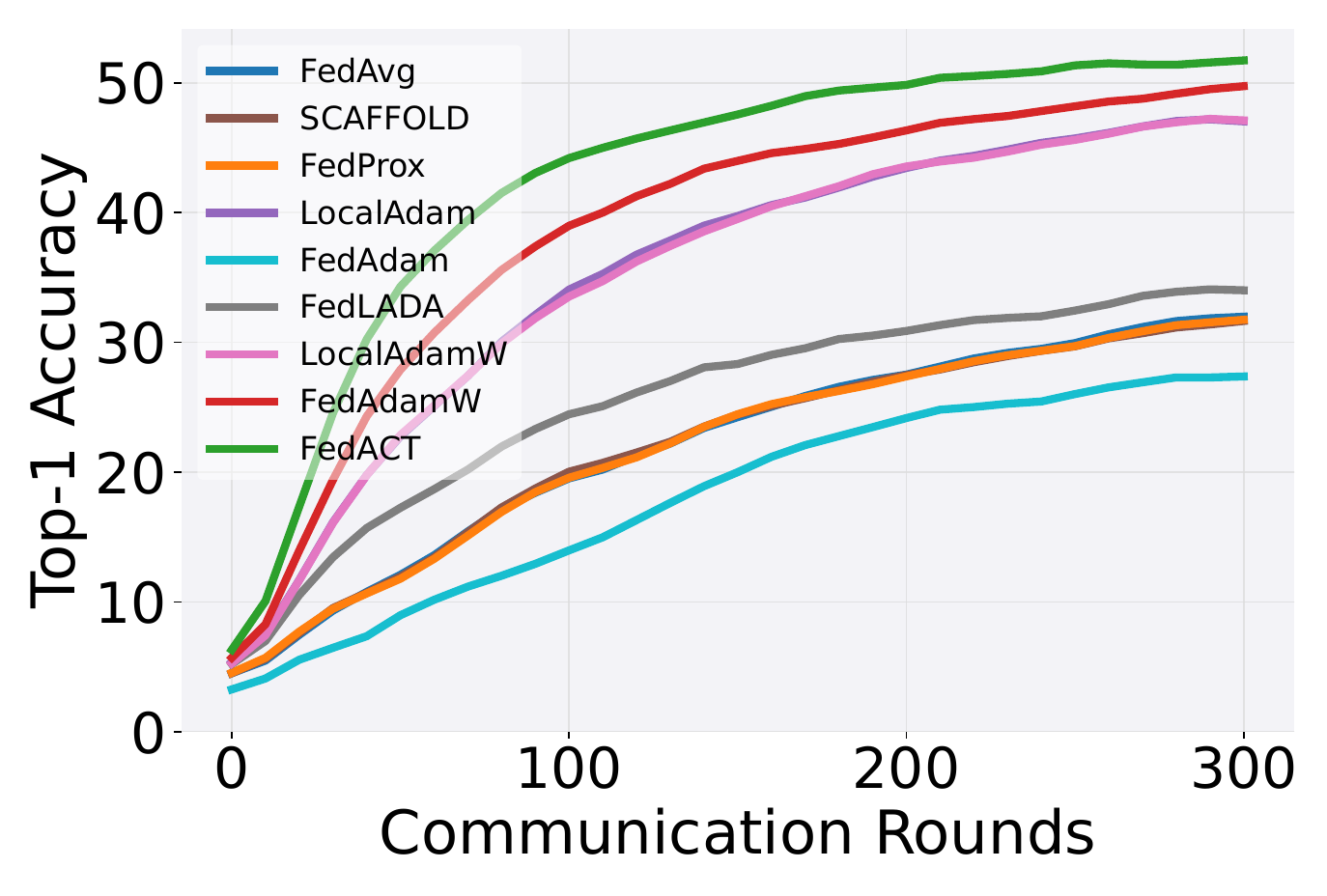}
        \vskip -0.2cm
        \caption{CIFAR-100, Dir-0.6}
    \end{subfigure}\hfill
    \begin{subfigure}{0.248\textwidth}
        \centering
        \includegraphics[width=1\textwidth]{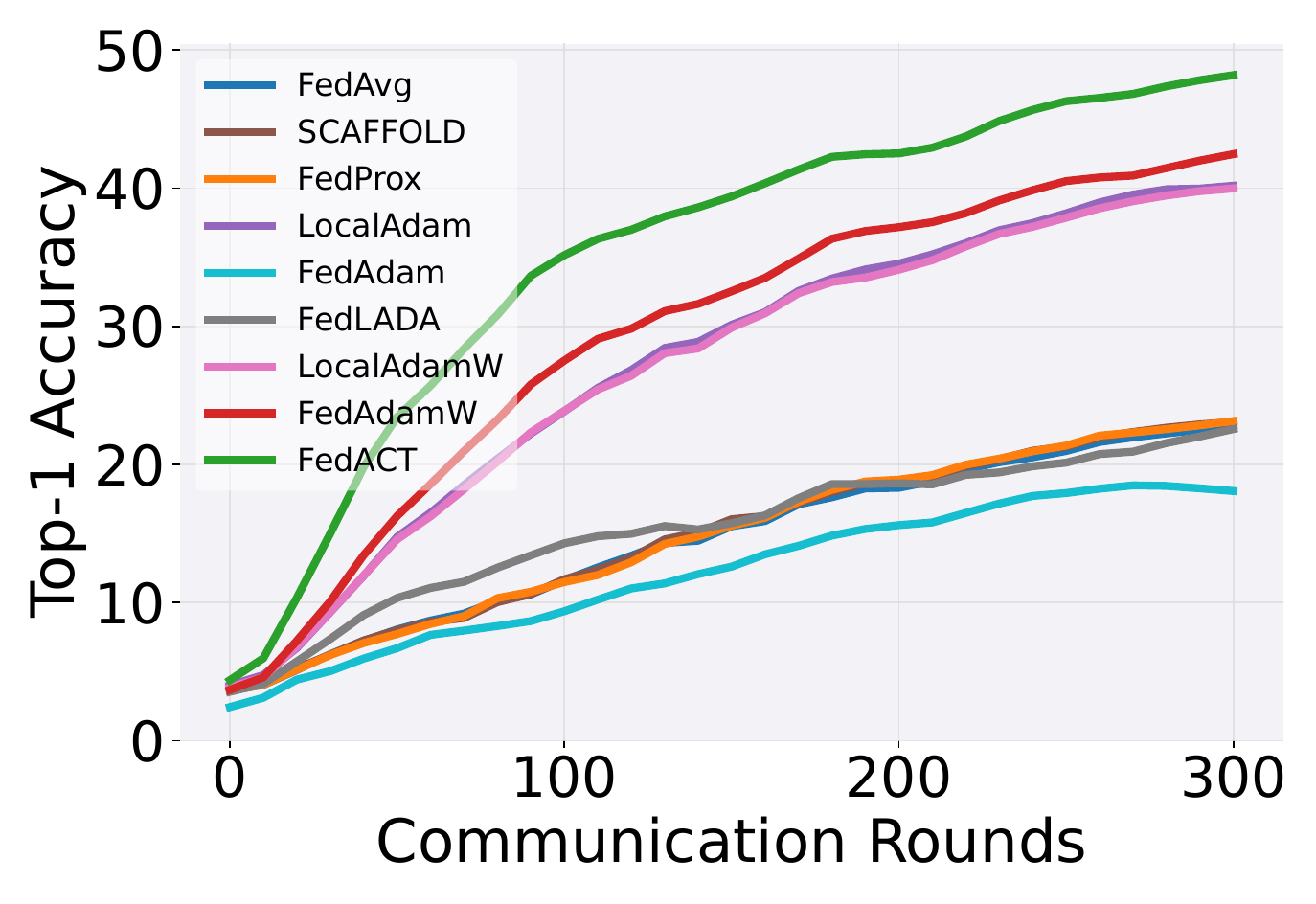}
        \vskip -0.2cm
        \caption{CIFAR-100, Dir-0.1}
    \end{subfigure}\hfill
    \begin{subfigure}{0.248\textwidth}
        \centering
        \includegraphics[width=1\textwidth]{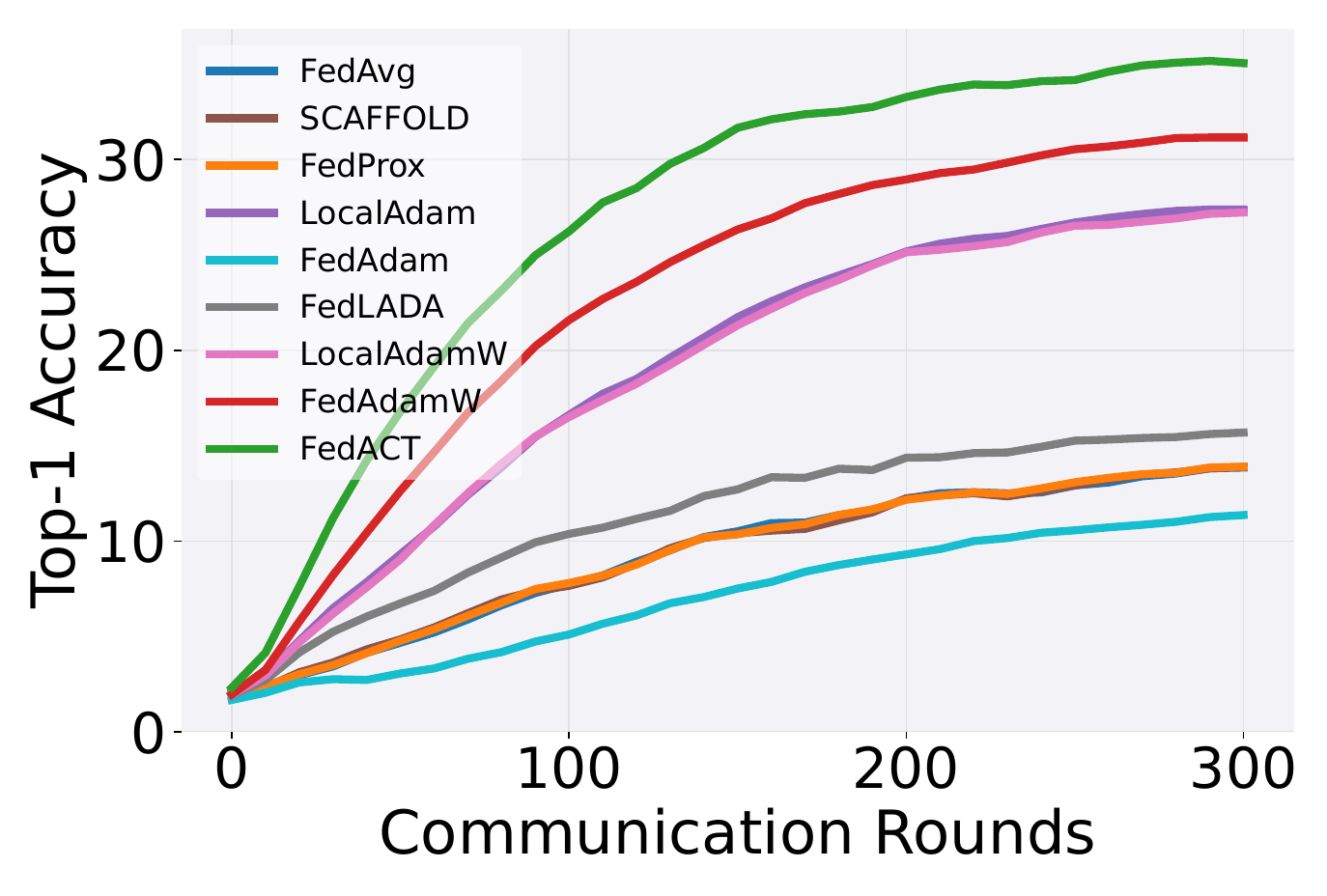}
        \vskip -0.2cm
        \caption{Tiny-ImageNet, Dir-0.6}
    \end{subfigure}
    \begin{subfigure}{0.248\textwidth}
        \centering
        \includegraphics[width=1\textwidth]{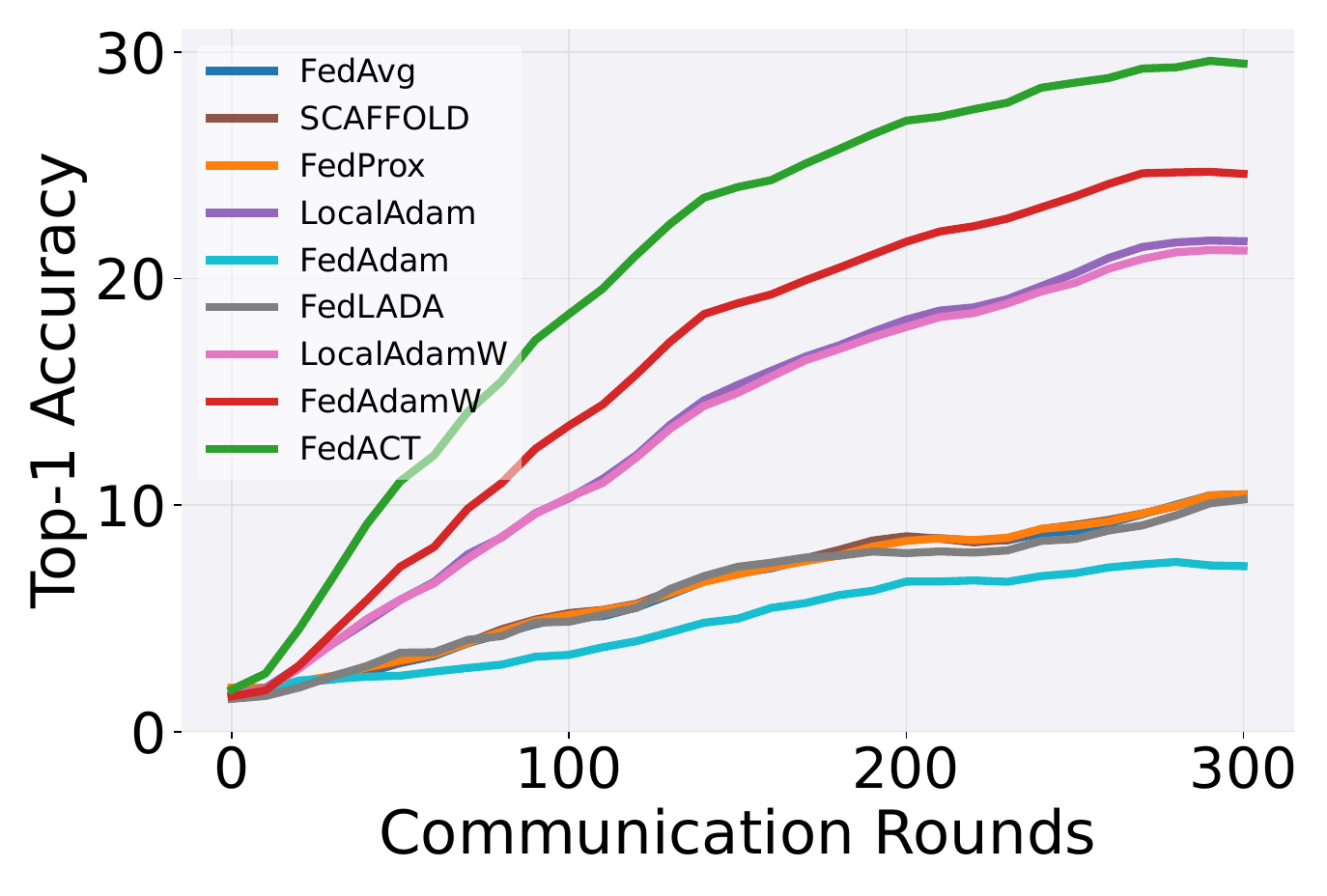}
        \vskip -0.2cm
        \caption{Tiny-ImageNet, Dir-0.1}
    \end{subfigure}
    \vskip -0.2cm
    \caption{Test Top-1 accuracy over communication rounds on federated vision Transformer training. The top row shows results for ViT-Tiny, and the bottom row shows results for Swin-Lite.}
    \label{fig:acc_vit_swin}
    \vskip -0.2cm
\end{figure*}

\begin{figure*}[!ht]
    \centering
    \begin{subfigure}{0.495\textwidth}
        \centering
        \includegraphics[width=1\textwidth]{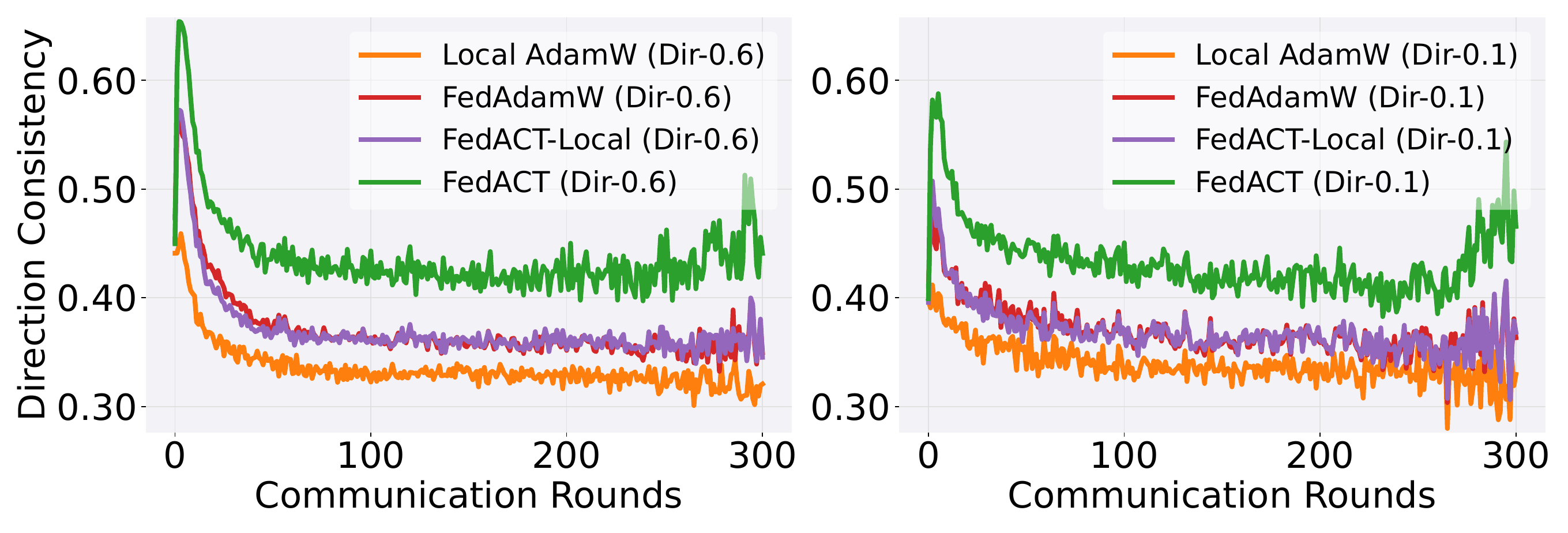}
        \vskip -0.2cm
        \caption{ViT-Tiny}
    \end{subfigure}\hfill
    \begin{subfigure}{0.495\textwidth}
        \centering
        \includegraphics[width=1\textwidth]{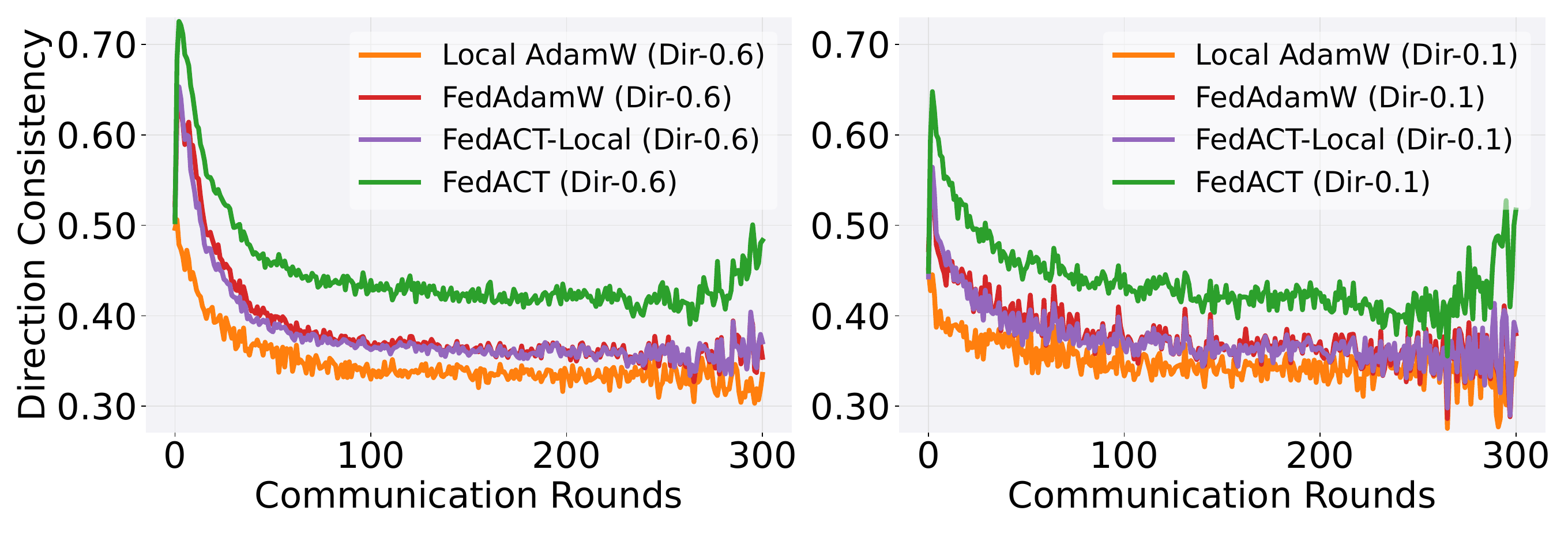}
        \vskip -0.2cm
        \caption{Swin-Lite}
    \end{subfigure}
    \vskip -0.2cm
    \caption{\textbf{Direction Consistency}. On CIFAR-100 with ViT-Tiny and Swin-Lite under 300 communication rounds, 100 clients, 10\% participation, and 50 local steps, FedACT consistently achieves higher direction consistency than Local AdamW, FedAdamW, and FedACT-Local under both Dirichlet heterogeneity levels.}
    \label{fig:mechanism_analysis}
    \vskip -0.2cm
\end{figure*}

\paragraph{Results on CNNs.} CNN results are deferred to the Appendix due to space limitations. \method consistently improves over FedAdamW across all nine settings, but the gains are smaller than those on ViT-Tiny and Swin-Lite. For example, the improvements are within 0.22--0.46 points on CIFAR-10, 0.76--1.41 points on CIFAR-100, and 0.40--0.91 points on Tiny-ImageNet. These results suggest that \method remains effective for CNN training, while its advantage is more pronounced in federated Transformer training, where local AdamW updates are stronger and more heterogeneous.


\paragraph{Results on Federated LLM Pre-Training and Fine-Tuning.} Table~\ref{tab:main_results_llama} reports the LLM pre-training results on C4-en with Llama2-60M/130M/250M under Chinchilla-matched token budgets. \method achieves the lowest training loss and validation perplexity across all three model sizes. Compared with FedAdamW, it reduces the validation perplexity from 35.64 to 32.57 on Llama2-60M, from 22.82 to 21.48 on Llama2-130M, and from 17.95 to 16.80 on Llama2-250M. Figure~\ref{fig:loss_llama} further shows that \method descends faster and reaches comparable loss levels with about $1.54\times$, $1.39\times$, and $1.30\times$ fewer communication rounds on the three model scales, respectively. For LLM fine-tuning, detailed metric results are deferred to the Appendix due to space limitations. The training loss curves in Figure~\ref{fig:loss_sft_dpo} show that \method reaches lower training loss than the baselines on both FedIT and FedVA, with a more pronounced advantage on FedVA. These results indicate that coordinate-level trust allocation remains effective in both federated LLM pre-training and parameter-efficient federated LLM fine-tuning.

\begin{table}[!ht]
    \centering
    \caption{Federated LLM pre-training results on C4-en.}
    \label{tab:main_results_llama}
    \vskip -0.2cm
    \resizebox{1.0\linewidth}{!}{
        {
\scriptsize
\setlength{\tabcolsep}{2.2pt}
\begin{tabular}{lcccccc}
  \toprule
  \multirow{2}{*}{\textbf{Method}}
  & \multicolumn{2}{c}{\textbf{Llama2-60M}}
  & \multicolumn{2}{c}{\textbf{Llama2-130M}}
  & \multicolumn{2}{c}{\textbf{Llama2-250M}} \\
  \cmidrule(lr){2-3} \cmidrule(lr){4-5} \cmidrule(lr){6-7}
  & \textbf{Train Loss $\downarrow$} & \textbf{Val. PPL $\downarrow$}
  & \textbf{Train Loss $\downarrow$} & \textbf{Val. PPL $\downarrow$}
  & \textbf{Train Loss $\downarrow$} & \textbf{Val. PPL $\downarrow$} \\
  \midrule
  LocalAdamW
    & 3.70 & 37.09
    & 3.36 & 26.33
    & 3.09 & 19.30 \\
  FedProx
    & 4.32 & 69.66
    & 3.97 & 48.38
    & 3.70 & 35.98 \\
  SCAFFOLD
    & 3.81 & 41.53
    & 3.40 & 27.44
    & 3.10 & 19.71 \\
  FedLADA
    & 3.72 & 37.87
    & 3.31 & 25.06
    & 3.46 & 27.36 \\
  FedAdamW
    & 3.65 & 35.64
    & 3.23 & 22.82
    & 2.95 & 17.95 \\
  \rowcolor{mygray}
  \textbf{Ours}
    & \textbf{3.57} & \textbf{32.57}
    & \textbf{3.18} & \textbf{21.48}
    & \textbf{2.94} & \textbf{16.80} \\
  \bottomrule
\end{tabular}
}
    }
    \vskip -0.2cm
\end{table}

\begin{figure}[!ht]
    \centering
    \includegraphics[width=1\linewidth]{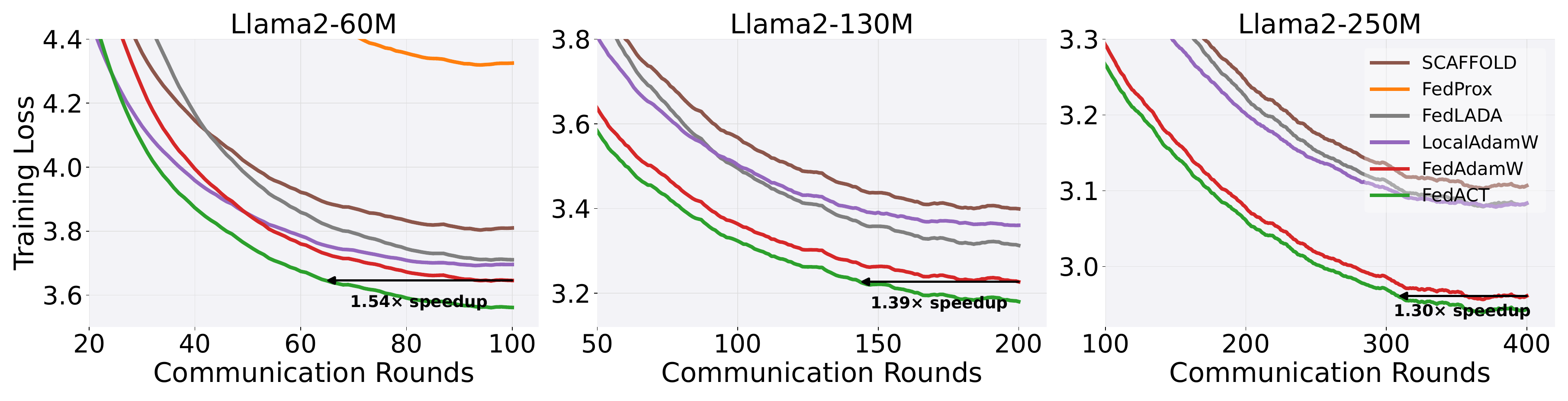}
    \vskip -0.2cm
    \caption{Training dynamics of different methods for federated LLM pre-training.}
    \label{fig:loss_llama}
    \vskip -0.2cm
\end{figure}

\begin{figure}[!ht]
    \centering
    \includegraphics[width=1\linewidth]{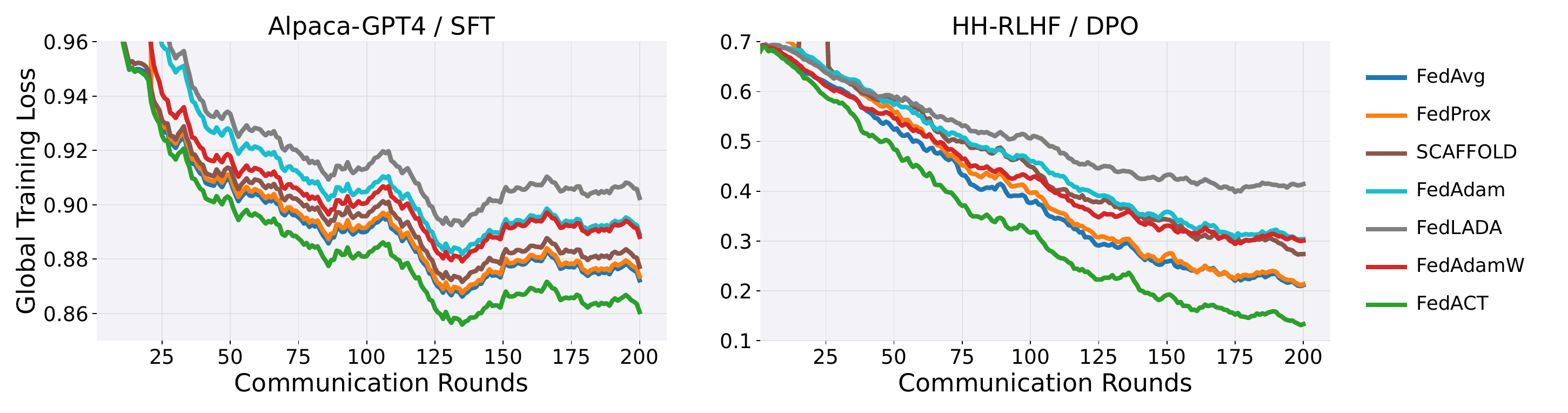}
    \vskip -0.2cm
    \caption{Training dynamics of different methods for FedIT and FedVA.}
    \label{fig:loss_sft_dpo}
    \vskip -0.2cm
\end{figure}


\subsection{Mechanism Analysis}
\label{subsec:mechanism}

We analyze whether the global-aware coordinate-level trust allocation of \method leads to more consistent client updates after local training. As shown in Figure~\ref{fig:mechanism_analysis}, \method consistently achieves higher direction consistency than both Local AdamW and FedAdamW across ViT-Tiny and Swin-Lite under mild and strong heterogeneity. Local AdamW suffers from lower consistency because its adaptive updates are mainly driven by client-specific data. FedAdamW improves consistency through global-local alignment and second-moment synchronization, but still applies the corrected direction densely and uniformly across all coordinates. In contrast, \method further modulates the corrected direction at the coordinate level: high-trust coordinates receive larger update magnitudes, while weakly supported coordinates are softly attenuated rather than discarded.


\subsection{Ablation Study}
\label{subsec:ablation}

\paragraph{Effect of Global-Aware Coordinate Trust Modulation.}
We examine whether the gains of \method come from coordinate-wise modulation alone or from making the modulation global-aware. To this end, we compare with FedACT-Local, which applies MGUP-style coordinate modulation using only local optimizer information. As shown in Table~\ref{tab:abs_global_local}, FedACT-Local improves over FedAdamW in most settings, indicating that coordinate-level trust non-uniformity is useful. However, the full \method consistently performs better, improving FedACT-Local by 1.17/2.66 points on ViT-Tiny under Dir-0.6/0.1, and by 1.20/2.34 points on Swin-Lite. Moreover, Figure~\ref{fig:mechanism_analysis} shows that FedACT-Local has direction consistency close to FedAdamW, while \method achieves clearly higher consistency. These results confirm that the key benefit does not come from local coordinate reweighting alone, but from using the global correction signal to select coordinates that are both locally descent-consistent and globally aligned.

\begin{table}[!ht]
    \vskip -0.2cm
    \centering
    \caption{Accuracy (\%) of FedAdamW, FedACT-Local, and FedACT on CIFAR-100.}
    \label{tab:abs_global_local}
    \vskip -0.2cm
    \resizebox{0.6\linewidth}{!}{
        \begin{tabular}{lcccc}
  \toprule
  \multirow{2}{*}{\textbf{Method}}
  & \multicolumn{2}{c}{\textbf{ViT-Tiny}}
  & \multicolumn{2}{c}{\textbf{Swin-Lite}} \\
  \cmidrule(lr){2-3} \cmidrule(lr){4-5}
  & Dir-0.6 & Dir-0.1
  & Dir-0.6 & Dir-0.1 \\
  \midrule
  FedAdamW
    & 41.76 & 39.08
    & 49.79 & 42.75 \\
  FedACT-Local
    & 42.41 & 39.37
    & 51.03 & 46.16 \\
  \rowcolor{mygray}
  \textbf{Ours}
    & \textbf{43.58} & \textbf{42.03}
    & \textbf{52.23} & \textbf{48.50} \\
  \bottomrule
\end{tabular}
    }
    \vskip -0.2cm
\end{table}

\paragraph{Sensitivity to Trust Ratio $\tau$ and Decay Factor $\gamma$.}
We study the sensitivity of \method to the trust ratio $\tau$ and decay factor $\gamma$ in Table~\ref{tab:abs_hyper_sensitivity}. The trust ratio $\tau$ controls the fraction of coordinates receiving amplified updates, while $\gamma$ controls the attenuation strength for the remaining coordinates, with $\gamma=0$ corresponding to hard masking. The results show that $\tau=0.5$ consistently provides the most reliable performance across ViT-Tiny and Swin-Lite under both Dir-0.6 and Dir-0.1. In contrast, $\tau=0.3$ severely degrades performance because too few coordinates are selected, while $\tau=0.7$ weakens the selectivity of trust modulation. The results also show that soft attenuation is preferable to hard masking: when $\tau=0.5$, non-zero $\gamma$ values generally outperform $\gamma=0$, indicating that preserving small updates on non-selected coordinates helps retain useful local information. Although $\gamma=0.9$ is slightly better in a few cases, $\tau=0.5$ and $\gamma=0.5$ give the most balanced and robust performance, and are therefore used as the default setting.

\begin{table}[!ht]
    \centering
    \caption{\textbf{Sensitivity Analysis.} Accuracy (\%) of FedACT with different trust ratios $\tau$ and decay factors $\gamma$ on CIFAR-100.}
    \label{tab:abs_hyper_sensitivity}
    \vskip -0.2cm
    \resizebox{\linewidth}{!}{
        \begin{tabular}{cc|cccc|cccc}
  \toprule
  \multirow{2}{*}{\textbf{Dir.}}
  & \multirow{2}{*}{\textbf{$\bm{\tau}$}}
  & \multicolumn{4}{c|}{\textbf{ViT-Tiny} ($\bm{\gamma}$)}
  & \multicolumn{4}{c}{\textbf{Swin-Lite} ($\bm{\gamma}$)} \\
  \cmidrule(lr){3-6} \cmidrule(lr){7-10}
  &
  & 0 & 0.1 & 0.5 & 0.9
  & 0 & 0.1 & 0.5 & 0.9 \\
  \midrule
  \multirow{3}{*}{Dir-0.6}
  & 0.3
    & 21.43 & 20.77 & 17.79 & 14.86
    & 12.90 & 12.04 & 10.24 & 8.81 \\
  & 0.5
    & 41.47 & 41.99 & \textbf{43.58} & \underline{43.16}
    & 51.82 & 51.82 & \underline{52.23} & \textbf{52.32} \\
  & 0.7
    & 38.93 & 39.30 & 41.45 & 41.36
    & 49.78 & 50.49 & 51.33 & 51.98 \\
  \midrule
  \multirow{3}{*}{Dir-0.1}
  & 0.3
    & 12.58 & 12.28 & 11.61 & 10.17
    & 6.55  & 6.32  & 6.00  & 5.61 \\
  & 0.5
    & 39.51 & 39.44 & \textbf{42.03} & \underline{41.30}
    & 46.69 & 46.75 & \textbf{48.50} & \underline{48.46} \\
  & 0.7
    & 36.66 & 37.48 & 38.18 & 39.72
    & 43.76 & 44.32 & 46.00 & 46.69 \\
  \bottomrule
\end{tabular}
    }
    \vskip -0.2cm
\end{table}

\section{Conclusion}
\label{sec:conclusion}

In this paper, we introduced \method, a global-aware coordinate trust modulation method for federated AdamW training under data heterogeneity. Our key observation is that, even after round-level global-local correction, corrected adaptive local updates can remain highly uneven in coordinate-wise reliability. \method addresses this coordinate trust mismatch by preserving the dense globally corrected AdamW direction while reallocating update magnitudes according to global-aware coordinate trust. Extensive experiments on federated vision Transformers, CNNs, LLM pre-training, and LLM fine-tuning show that \method consistently improves strong federated adaptive baselines, with especially clear gains on Transformer models under stronger heterogeneity. Mechanism analyses further show that \method improves cross-client direction consistency, highlighting coordinate-level trust allocation as a useful new design dimension for federated adaptive optimization.


\bibliography{main}


\clearpage
\onecolumn
\appendix

\section{Appendix A: Details of Evaluation Metrics}
\label{appendix:eval_metrics}

This section provides the detailed definitions of the diagnostic metrics used in the motivation analysis section. These metrics are designed to characterize federated adaptive optimization from two complementary perspectives. Direction Consistency measures whether local client updates are aligned at the communication-round level, while Positive Score Ratio and Top-$p$ Score Mass further inspect the internal coordinate-level trust structure of the corrected adaptive update. Together, these metrics help us distinguish whether the remaining mismatch comes from disagreement across clients or from highly non-uniform reliability across coordinates after global correction.

Let $S_r$ denote the set of clients participating in communication round $r$, and let $S=|S_r|$. For each client $i\in S_r$, let $\Delta_i^r$ be the accumulated local model update in round $r$. The averaged update direction in this round is defined as
\begin{equation}
    \bar{\Delta}^r = \frac{1}{S}\sum_{j\in S_r}\Delta_j^r .
\end{equation}
For coordinate-level diagnostics, let $K$ denote the number of local steps, and define the set of observed client-step pairs in round $r$ as
\begin{equation}
    \mathcal{O}_r = \{(i,k): i\in S_r,\; k=1,\ldots,K\}.
\end{equation}
At local step $k$ of client $i$, let $g_i^{r,k}$ denote the stochastic gradient and let $u_i^{r,k}$ denote the corrected adaptive update direction. 
We define the coordinate-wise trust score as
\begin{equation}
    s_i^{r,k} = u_i^{r,k} \odot g_i^{r,k},
\end{equation}
where $s_{i,j}^{r,k}$ is the trust score of coordinate $j$. Since the local parameter update moves along $-u_i^{r,k}$, a positive score $s_{i,j}^{r,k}>0$ means that the corrected adaptive direction is descent-consistent with the current stochastic gradient on coordinate $j$ under the first-order approximation. In contrast, a small or negative score indicates that the coordinate is weakly supported or potentially conflicting.

\paragraph{Direction Consistency.}
Direction Consistency measures how well each participating client's update agrees with the averaged update direction in the same communication round:
\begin{equation}
    \operatorname{DC}^r = 
    \frac{1}{S}\sum_{i \in S_r} \cos(\Delta_i^r, \bar{\Delta}^r)
    =
    \frac{1}{S}\sum_{i \in S_r} 
    \cos \left(
    \Delta_i^r, 
    \frac{1}{S}\sum_{j \in S_r} \Delta_j^r
    \right).
\end{equation}
A larger value indicates that local updates are more consistently aligned with the round-level aggregated direction, while a smaller value suggests stronger client drift or update disagreement across participating clients. Because this metric is based on cosine similarity, it focuses on update directions rather than update magnitudes. In the motivation analysis, we use this metric to evaluate whether global correction improves round-level alignment across clients.

\paragraph{Positive Score Ratio.}
Direction Consistency captures the aggregate alignment of client updates, but it does not reveal whether different coordinates of the corrected adaptive direction are equally reliable. To inspect the update at a finer granularity, we compute the fraction of coordinates whose trust scores are positive:
\begin{equation}
    \operatorname{PosRatio}^{r} = 
    \frac{1}{|\mathcal{O}_r|} 
    \sum_{(i,k)\in \mathcal{O}_r} 
    \frac{1}{d}\sum_{j=1}^d 
    \mathbf{1}\left(s_{i,j}^{r,k}>0\right).
\end{equation}
This metric measures the breadth of descent-consistent coordinates in the corrected adaptive update. A larger value means that more coordinates are at least positively supported by the current stochastic gradient.

However, this metric only counts how many coordinates have positive scores; it does not indicate whether the positive trust signal is uniformly distributed or dominated by a small subset of coordinates. Therefore, we interpret it together with the Top-$p$ Score Mass.

\paragraph{Top-$p$ Score Mass.}
To quantify the concentration of coordinate-wise trust, we compute the fraction of positive score mass carried by the top-$p$ fraction of coordinates:
\begin{equation}
    \operatorname{TopMass}_{p}^{r} = 
    \frac{1}{|\mathcal{O}_r|} 
    \sum_{(i,k) \in \mathcal{O}_r} 
    \frac{
    \sum_{j \in \operatorname{Top}_{p}(s_i^{r,k})}
    [s_{i,j}^{r,k}]_+
    }{
    \sum_{j=1}^{d} [s_{i,j}^{r,k}]_+
    },
\end{equation}
where $[a]_+=\max(a,0)$, and $\operatorname{Top}_{p}(s_i^{r,k})$ denotes the index set of the largest $\lceil p d\rceil$ entries in the score vector $s_i^{r,k}$. For example, $p=0.01$ corresponds to the top $1\%$ coordinates. If the denominator is zero, we define the corresponding ratio as zero.

This metric measures whether the positive trust signal is broadly distributed or concentrated in a small subset of coordinates. A larger Top-$p$ Score Mass indicates that a small fraction of coordinates carries most of the positive score mass, suggesting strong coordinate-level trust heterogeneity. In federated adaptive optimization, this is important because dense uniform updating implicitly assigns the same update scale to all coordinates of the corrected direction. When the Top-$p$ Score Mass is high, the corrected adaptive update contains a long tail of weakly supported coordinates that may be over-trusted by dense uniform updating. This motivates \method to assign larger update magnitudes to high-trust coordinates and smaller but non-zero magnitudes to the remaining coordinates.

\section{Appendix B: Convergence Analysis}
\label{appendix:convergence_analysis}

{
In this section, we analyze \method with $\lambda=0$ under full client participation. The analysis tracks the averaged local trajectory and separates three sources of error: stochastic gradient variance, client heterogeneity, and the modulation bias introduced by ACT. The bounded-gradient condition is used only to instantiate finite constants for the AdamW-style scaling and the stepsize range.

For clarity, we analyze the virtual averaged sequence of \method. Let $T=RK$ be the total number of local update steps. We use a single index $t=0, \ldots, T-1$ to denote the flattened local steps. Since we consider full participation in the analysis, all clients participate in each communication round, i.e.,
\begin{align}
    \mathcal{S}_r=\{1,\ldots,N\},\quad S=N.
\end{align}

Let $\bm{x}_{i,t}$ be the local iterate of client $i$ at step $t$, and define the averaged virtual iterate as
\begin{align}
    \bar{\bm{x}}_t=\frac{1}{N}\sum_{i=1}^{N}\bm{x}_{i,t}.
\end{align}

The averaged update is
\begin{equation}
    \bar{\bm{x}}_{t+1}=\bar{\bm{x}}_t-\eta\bar{\bm{p}}_t,\quad
    \bar{\bm{p}}_t=\frac{1}{N}\sum_{i=1}^{N}\bm{p}_{i,t}.
    \label{eq:virtual_avg_update}
\end{equation}

For each client,
\begin{equation}
    \bm{p}_{i,t}=\bm{\phi}_{i,t}\odot\bm{u}_{i,t},\quad
    \bm{u}_{i,t}=(1-\rho)\bm{H}_t\bm{g}_{i,t}+\rho\bm{\Delta}_t.
    \label{eq:fedact_direction_virtual}
\end{equation}

Here $\bm{H}_t$ is the synchronized AdamW-style diagonal adaptive scaling matrix and $\bm{\Delta}_t$ is the global correction direction inherited from the previous communication round. The ACT coefficient vector $\bm{\phi}_{i,t}\in\mathbb{R}^d$ is constructed from the trust score $\bm{s}_{i,t}=\bm{u}_{i,t}\odot\bm{g}_{i,t}$ as described in the Global-Aware Coordinate Trust Modulation subsection~\ref{subsec:trust_modulation} of the main paper. When $\rho=0$ and $\tau=1$, we have $\bm{\phi}_{i,t}=\bm{1}$ and \method reduces to the base AdamW-style preconditioned stochastic update.

\begin{remark}[Predictable preconditioner]
\label{rem:predictable_precond}
The analysis treats $\bm{H}_t$ as $\mathcal{F}_t$-measurable. Equivalently, $\bm{H}_t$ is formed before sampling the current stochastic gradients $\{\bm{g}_{i,t}\}_{i=1}^N$. This predictable-preconditioner convention isolates the effect of ACT from the additional dependence between the current gradient and the current second-moment estimate.
\end{remark}

\subsection{Assumptions}

\begin{assumption}[Lower bounded objective]
\label{ass:lower_bound}
The global objective $f$ is bounded from below. That is, there exists $f^\star>-\infty$ such that
\begin{align}
    f(\bm{x})\ge f^\star, \quad \forall \bm{x}\in\mathbb{R}^d .
\end{align}
We define $\Delta_f=f(\bm{x}_0)-f^\star$.
\end{assumption}

\begin{assumption}[Smoothness]
\label{ass:smoothness}
Each local objective $f_i$ and the global objective $f$ are $L$-smooth, i.e., for any $\bm{x},\bm{y}\in\mathbb{R}^d$,
\begin{align}
    f_i(\bm{y})\le f_i(\bm{x})+\left\langle\nabla f_i(\bm{x}),\bm{y}-\bm{x}\right\rangle+\frac{L}{2}\|\bm{y}-\bm{x}\|^2 .
\end{align}
\end{assumption}

\begin{assumption}[Unbiased stochastic gradients with bounded variance]
\label{ass:bounded_variance}
For every client $i$ and every step $t$, the stochastic gradient satisfies
\begin{align}
    \mathbb{E}\left[\bm{g}_{i,t}\mid\mathcal{F}_t\right]=\nabla f_i(\bm{x}_{i,t}),
\end{align}
and there exists $\sigma^2>0$ such that
\begin{align}
    \mathbb{E}\left[\left\|\bm{g}_{i,t}-\nabla f_i(\bm{x}_{i,t})\right\|^2\mid\mathcal{F}_t\right]\le \sigma^2 .
\end{align}
\end{assumption}

\begin{assumption}[Bounded data heterogeneity]
\label{ass:bounded_heterogeneity}
There exists $\zeta^2>0$ such that, for all $\bm{x}\in\mathbb{R}^d$,
\begin{align}
    \frac{1}{N}\sum_{i=1}^{N}\left\|\nabla f_i(\bm{x})-\nabla f(\bm{x})\right\|^2\le \zeta^2 .
\end{align}
\end{assumption}

\begin{assumption}[Uniformly bounded stochastic gradients]
\label{ass:bounded_gradient}
There exists $G>0$ such that
\begin{align}
    \|\bm{g}_{i,t}\|\le G \quad \text{almost surely}, \quad \forall i,t .
\end{align}
\end{assumption}

\begin{assumption}[Full client participation]
\label{ass:full_participation}
All clients participate in every communication round, i.e.,
\begin{align}
   \mathcal{S}_r=\{1,\ldots,N\},\quad S=N .
\end{align}
\end{assumption}

\subsection{Consequences of the Algorithm}

We first record several bounds implied by the algorithm. These bounds use full participation, bounded stochastic gradients, and the predictable preconditioner convention in Remark~\ref{rem:predictable_precond}. For the AdamW-style scaling,
\begin{align}
    \bm{H}_t=\operatorname{diag}\left(\frac{1}{\sqrt{\hat{\bm{v}}_t}+\epsilon}\right),
\end{align}
Assumption~\ref{ass:bounded_gradient} implies
\begin{equation}
    \mu\bm{I}\preceq \bm{H}_t\preceq M\bm{I},\quad
    \mu=\frac{1}{G/\sqrt{1-\beta_2}+\epsilon},\quad M=\frac{1}{\epsilon}.
    \label{eq:H_bound_new}
\end{equation}
Here the lower bound uses the conservative estimate $\sqrt{\hat{v}_{t,j}}\le G/\sqrt{1-\beta_2}$.

\begin{lemma}[Global correction as averaged previous directions]
\label{lem:correction_average}
For every communication round $r$, the global correction signal satisfies
\begin{equation}
    \bm{\Delta}_G^r=\frac{1}{NK}\sum_{i=1}^{N}\sum_{k=0}^{K-1}\bm{p}_i^{r,k}.
    \label{eq:delta_as_average}
\end{equation}
In the flattened virtual sequence, the correction $\bm{\Delta}_t$ used at any local step in round $r$ equals $\bm{\Delta}_G^{r-1}$, which is constructed from directions $\{\bm{p}_i^{r-1,k}\}$ in round $r-1$. Equivalently,
\begin{align}
    \bm{\Delta}_t=\frac{1}{K}\sum_{k=0}^{K-1}\bar{\bm{p}}_{s_k},
\end{align}
where
\begin{align}
    \bar{\bm{p}}_{s_k}=\frac{1}{N}\sum_{i=1}^{N}\bm{p}_i^{r-1,k}
\end{align}
is the virtual averaged direction at the $k$-th local step of round $r-1$. The maximum staleness is
\begin{align}
    \Gamma=2K-1 .
\end{align}
\end{lemma}

\begin{proof}
Since $\lambda=0$, the local update is
\begin{align}
    \bm{x}_i^{r,k+1}=\bm{x}_i^{r,k}-\eta\bm{p}_i^{r,k}.
\end{align}
Summing over $k=0,\ldots,K-1$ gives
\begin{align}
    \bm{x}_i^{r,K}-\bm{x}_i^{r,0}=-\eta\sum_{k=0}^{K-1}\bm{p}_i^{r,k}.
\end{align}
Using the server definition under full participation,
\begin{align}
    \bm{\Delta}_G^r=\frac{-1}{NK\eta}\sum_{i=1}^{N}\left(\bm{x}_i^{r,K}-\bm{x}_i^{r,0}\right),
\end{align}
we obtain \eqref{eq:delta_as_average}. The second representation follows by swapping the order of summation. The staleness bound follows because $\bm{\Delta}_G^{r-1}$ is constructed from directions in round $r-1$ and is used in round $r$; hence the maximum time gap between the current step and the oldest contributing direction is at most $2K-1$ local steps.
\end{proof}

\begin{lemma}[Uniform boundedness of \method directions]
\label{lem:direction_bounded_new}
Let $0<\rho<\tau\le 1$. Then every effective direction satisfies
\begin{equation}
    \|\bm{p}_{i,t}\|\le P_{\rho,\tau},\quad P_{\rho,\tau}=\frac{(1-\rho)MG}{\tau-\rho}.
    \label{eq:P_new}
\end{equation}
\end{lemma}

\begin{proof}
We prove the claim by induction on the flattened step index $t$. Since every ACT coefficient satisfies $\tau\le \phi_{i,t,j}\le 1/\tau$, we have
\begin{align}
    \|\bm{\phi}_{i,t}\odot\bm{u}\|\le \frac{1}{\tau}\|\bm{u}\|,\quad \forall\bm{u}\in\mathbb{R}^d .
\end{align}
Therefore,
\begin{equation}
\begin{aligned}
    \|\bm{p}_{i,t}\|&=\|\bm{\phi}_{i,t}\odot\bm{u}_{i,t}\|
    \le \frac{1}{\tau}\|\bm{u}_{i,t}\| \\
    &\le \frac{1-\rho}{\tau}\|\bm{H}_t\bm{g}_{i,t}\|+\frac{\rho}{\tau}\|\bm{\Delta}_t\|
    \le \frac{(1-\rho)MG}{\tau}+\frac{\rho}{\tau}\|\bm{\Delta}_t\|.
\end{aligned}
\label{eq:p_step_bound}
\end{equation}
At $t=0$, $\bm{\Delta}_0=\bm{0}$, hence $\|\bm{p}_{i,0}\|\le(1-\rho)MG/\tau\le P_{\rho,\tau}$. Suppose $\|\bm{p}_{i,s}\|\le P_{\rho,\tau}$ for all $s<t$ and all clients $i$. By Lemma~\ref{lem:correction_average}, $\bm{\Delta}_t$ is a convex combination of previous effective directions. Therefore,
\begin{align}
    \|\bm{\Delta}_t\|\le \max_{s<t,i}\|\bm{p}_{i,s}\|\le P_{\rho,\tau}.
\end{align}
Substituting this into \eqref{eq:p_step_bound} and using $P_{\rho,\tau}=(1-\rho)MG/(\tau-\rho)$ gives
\begin{align}
    \frac{(1-\rho)MG}{\tau}+\frac{\rho}{\tau}P_{\rho,\tau}
    =\frac{(1-\rho)MG}{\tau}\cdot\frac{\tau-\rho+\rho}{\tau-\rho}=P_{\rho,\tau}.
\end{align}
Hence $\|\bm{p}_{i,t}\|\le P_{\rho,\tau}$, completing the induction.
\end{proof}

\begin{lemma}[Bounded corrected direction]
\label{lem:u_bounded}
Under the conditions of Lemma~\ref{lem:direction_bounded_new}, the corrected direction satisfies
\begin{equation}
    \|\bm{u}_{i,t}\|\le \tau P_{\rho,\tau}.
    \label{eq:u_bounded}
\end{equation}
\end{lemma}

\begin{proof}
By the definition of $\bm{u}_{i,t}$ and Lemma~\ref{lem:direction_bounded_new},
\begin{align}
\begin{aligned}
    \|\bm{u}_{i,t}\|&\le (1-\rho)\|\bm{H}_t\bm{g}_{i,t}\|+\rho\|\bm{\Delta}_t\| \\
    &\le (1-\rho)MG+\rho P_{\rho,\tau}
    =(1-\rho)MG\cdot\frac{\tau-\rho+\rho}{\tau-\rho}
    =\frac{\tau(1-\rho)MG}{\tau-\rho}=\tau P_{\rho,\tau}.
\end{aligned}
\end{align}
\end{proof}

\begin{lemma}[Local drift bound]
\label{lem:local_drift_bound}
Define
\begin{align}
    D_t=\frac{1}{N}\sum_{i=1}^{N}\mathbb{E}\|\bm{x}_{i,t}-\bar{\bm{x}}_t\|^2 .
\end{align}
Then
\begin{equation}
    D_t\le 4\eta^2K^2P_{\rho,\tau}^2 .
    \label{eq:drift_bound}
\end{equation}
\end{lemma}

\begin{proof}
At the beginning of each communication round, all clients are initialized from the same global model $\bm{x}^{r}$. During at most $K$ local steps, each local update has norm at most $\eta P_{\rho,\tau}$ by Lemma~\ref{lem:direction_bounded_new}. Hence, for every client $i$,
\begin{equation}
    \|\bm{x}_{i,t}-\bm{x}^{r}\|\le \eta KP_{\rho,\tau}.
    \label{eq:client_drift_to_anchor}
\end{equation}
Since $\bar{\bm{x}}_t=N^{-1}\sum_{i=1}^{N}\bm{x}_{i,t}$, we have
\begin{align}
    \|\bar{\bm{x}}_t-\bm{x}^{r}\|\le \frac{1}{N}\sum_{i=1}^{N}\|\bm{x}_{i,t}-\bm{x}^{r}\|\le \eta KP_{\rho,\tau}.
\end{align}
Combining the two bounds,
\begin{align}
    \|\bm{x}_{i,t}-\bar{\bm{x}}_t\|\le \|\bm{x}_{i,t}-\bm{x}^{r}\|+\|\bar{\bm{x}}_t-\bm{x}^{r}\|\le 2\eta KP_{\rho,\tau}.
\end{align}
Squaring and averaging over all clients gives \eqref{eq:drift_bound}.
\end{proof}

\subsection{Second-Moment Bound Using Variance and Heterogeneity}

The following bound controls the averaged update norm using stochastic variance, data heterogeneity, and local drift.

\begin{lemma}[Second moment of \method directions]
\label{lem:second_moment_variance_heterogeneity}
Suppose Assumptions~\ref{ass:smoothness}--\ref{ass:bounded_heterogeneity} hold. Then
\begin{equation}
    \mathbb{E}\|\bar{\bm{p}}_t\|^2\le q_g\mathbb{E}\|\nabla f(\bar{\bm{x}}_t)\|^2+q_0,
    \label{eq:p_second_moment}
\end{equation}
where
\begin{equation}
    q_g=\frac{8(1-\rho)^2M^2}{\tau^2},
    \label{eq:qg_def}
\end{equation}
and
\begin{equation}
    q_0=\frac{8(1-\rho)^2M^2}{\tau^2}\left(\sigma^2+\zeta^2+4L^2\eta^2K^2P_{\rho,\tau}^2\right)+\frac{2\rho^2}{\tau^2}P_{\rho,\tau}^2 .
    \label{eq:q0_def}
\end{equation}
\end{lemma}

\begin{proof}
By Jensen's inequality,
\begin{align}
    \mathbb{E}\|\bar{\bm{p}}_t\|^2=\mathbb{E}\left\|\frac{1}{N}\sum_{i=1}^{N}\bm{p}_{i,t}\right\|^2\le \frac{1}{N}\sum_{i=1}^{N}\mathbb{E}\|\bm{p}_{i,t}\|^2 .
\end{align}
Using $\|\bm{\phi}_{i,t}\odot\bm{u}_{i,t}\|\le\tau^{-1}\|\bm{u}_{i,t}\|$ and $\|a+b\|^2\le2\|a\|^2+2\|b\|^2$, we obtain
\begin{align}
\begin{aligned}
    \|\bm{p}_{i,t}\|^2
    &\le \frac{1}{\tau^2}\left\|(1-\rho)\bm{H}_t\bm{g}_{i,t}+\rho\bm{\Delta}_t\right\|^2 \\
    &\le \frac{2(1-\rho)^2M^2}{\tau^2}\|\bm{g}_{i,t}\|^2+\frac{2\rho^2}{\tau^2}P_{\rho,\tau}^2 .
\end{aligned}
\end{align}
It remains to bound the averaged second moment of $\bm{g}_{i,t}$. We decompose
\begin{align}
\begin{aligned}
    \bm{g}_{i,t}
    &=\nabla f(\bar{\bm{x}}_t)+\left(\nabla f_i(\bar{\bm{x}}_t)-\nabla f(\bar{\bm{x}}_t)\right) \\
    &\quad+\left(\nabla f_i(\bm{x}_{i,t})-\nabla f_i(\bar{\bm{x}}_t)\right)
    +\left(\bm{g}_{i,t}-\nabla f_i(\bm{x}_{i,t})\right).
\end{aligned}
\end{align}
Applying $\|a+b+c+d\|^2\le4(\|a\|^2+\|b\|^2+\|c\|^2+\|d\|^2)$, taking expectation, and averaging over all clients gives
\begin{align}
\begin{aligned}
    \frac{1}{N}\sum_{i=1}^{N}\mathbb{E}\|\bm{g}_{i,t}\|^2
    &\le 4\mathbb{E}\|\nabla f(\bar{\bm{x}}_t)\|^2
    +4\mathbb{E}\left[\frac{1}{N}\sum_{i=1}^{N}\left\|\nabla f_i(\bar{\bm{x}}_t)-\nabla f(\bar{\bm{x}}_t)\right\|^2\right] \\
    &\quad+4L^2\mathbb{E}\left[\frac{1}{N}\sum_{i=1}^{N}\|\bm{x}_{i,t}-\bar{\bm{x}}_t\|^2\right]+4\sigma^2 \\
    &\le 4\mathbb{E}\|\nabla f(\bar{\bm{x}}_t)\|^2+4\zeta^2+4L^2D_t+4\sigma^2 \\
    &\le 4\mathbb{E}\|\nabla f(\bar{\bm{x}}_t)\|^2+4\left(\sigma^2+\zeta^2+4L^2\eta^2K^2P_{\rho,\tau}^2\right),
\end{aligned}
\end{align}
where we used Assumptions~\ref{ass:bounded_heterogeneity} and \ref{ass:bounded_variance}, $L$-smoothness, and Lemma~\ref{lem:local_drift_bound}. Substituting this bound back gives \eqref{eq:p_second_moment}.
\end{proof}

\subsection{Descent Lower Bound}

We next lower bound the descent produced by the averaged \method direction. Since ACT selects coordinates using the stochastic trust score $\bm{s}_{i,t}=\bm{u}_{i,t}\odot\bm{g}_{i,t}$, the selected coordinates need not be aligned with $\nabla f(\bar{\bm{x}}_t)$ in the worst case. We therefore use a deterministic bound that only relies on the range of the ACT coefficients.

\begin{lemma}[ACT inner-product bound]
\label{lem:act_conservative_bound}
Let $\bm{r},\bm{u}\in\mathbb{R}^d$ and let $\bm{\phi}\in\mathbb{R}^d$ be any ACT coefficient vector satisfying
\begin{align}
    \tau\le \phi_j\le \frac{1}{\tau},\quad j=1,\ldots,d .
\end{align}
Then
\begin{equation}
    \left\langle \bm{r},\bm{\phi}\odot\bm{u}\right\rangle
    \ge \tau\langle\bm{r},\bm{u}\rangle-\left(\frac{1}{\tau}-\tau\right)\|\bm{r}\|\|\bm{u}\|.
    \label{eq:act_conservative_bound}
\end{equation}
\end{lemma}

\begin{proof}
Write $\bm{\phi}=\tau\bm{1}+\bm{a}$ with $0\le a_j\le 1/\tau-\tau$ for all $j$. Then
\begin{align}
\begin{aligned}
    \langle\bm{r},\bm{\phi}\odot\bm{u}\rangle
    &=\tau\langle\bm{r},\bm{u}\rangle+\langle\bm{r},\bm{a}\odot\bm{u}\rangle \\
    &\ge \tau\langle\bm{r},\bm{u}\rangle-\|\bm{r}\|\|\bm{a}\odot\bm{u}\|.
\end{aligned}
\end{align}
Using $\|\bm{a}\odot\bm{u}\|\le\|\bm{a}\|_\infty\|\bm{u}\|\le(1/\tau-\tau)\|\bm{u}\|$, we obtain \eqref{eq:act_conservative_bound}.
\end{proof}

\begin{lemma}[Descent lower bound]
\label{lem:descent_lower_bound_new}
Suppose Assumptions~\ref{ass:smoothness}--\ref{ass:full_participation} hold. Let $\lambda=0$ and $0<\rho<\tau\le 1$. Define
\begin{align}
    C_\tau=1-\tau^2,
\end{align}
and
\begin{align}
    A_t=\mathbb{E}\left[\left\langle\nabla f(\bar{\bm{x}}_t),\bar{\bm{p}}_t\right\rangle\right].
\end{align}
Then
\begin{equation}
    A_t\ge a_{\rho,\tau}\mathbb{E}\|\nabla f(\bar{\bm{x}}_t)\|^2-b_{\rho,\tau}(\eta),
    \label{eq:descent_lower_final}
\end{equation}
where
\begin{equation}
    a_{\rho,\tau}=\frac{\tau(1-\rho)\mu}{4},
    \label{eq:a_def}
\end{equation}
and
\begin{equation}
    b_{\rho,\tau}(\eta)=\tau(1-\rho)\frac{M^2L^2}{2\mu}\cdot 4\eta^2K^2P_{\rho,\tau}^2
    +\frac{2\tau\rho^2P_{\rho,\tau}^2}{(1-\rho)\mu}
    +\frac{2C_\tau^2P_{\rho,\tau}^2}{\tau(1-\rho)\mu}.
    \label{eq:b_def}
\end{equation}
\end{lemma}

\begin{proof}
Applying Lemma~\ref{lem:act_conservative_bound} with $\bm{r}=\nabla f(\bar{\bm{x}}_t)$ and $\bm{u}=\bm{u}_{i,t}$, and using Lemma~\ref{lem:u_bounded}, we obtain
\begin{align}
\begin{aligned}
    \langle\nabla f(\bar{\bm{x}}_t),\bm{p}_{i,t}\rangle
    &\ge \tau\langle\nabla f(\bar{\bm{x}}_t),\bm{u}_{i,t}\rangle
    -\left(\frac{1}{\tau}-\tau\right)\|\nabla f(\bar{\bm{x}}_t)\|\|\bm{u}_{i,t}\| \\
    &\ge \tau\langle\nabla f(\bar{\bm{x}}_t),\bm{u}_{i,t}\rangle
    -C_\tau P_{\rho,\tau}\|\nabla f(\bar{\bm{x}}_t)\|,
\end{aligned}
\end{align}
where $(1/\tau-\tau)\tau P_{\rho,\tau}=(1-\tau^2)P_{\rho,\tau}=C_\tau P_{\rho,\tau}$. Averaging over all clients and taking expectation gives
\begin{equation}
\begin{aligned}
    A_t
    &\ge \tau(1-\rho)\mathbb{E}\left\langle\nabla f(\bar{\bm{x}}_t),\bm{H}_t\bar{\bm{g}}_t\right\rangle
    +\tau\rho\mathbb{E}\left\langle\nabla f(\bar{\bm{x}}_t),\bm{\Delta}_t\right\rangle \\
    &\quad-C_\tau P_{\rho,\tau}\mathbb{E}\|\nabla f(\bar{\bm{x}}_t)\|,
\end{aligned}
\label{eq:A_decomp}
\end{equation}
where $\bar{\bm{g}}_t=N^{-1}\sum_{i=1}^{N}\bm{g}_{i,t}$.

We bound the three terms separately. First, by Remark~\ref{rem:predictable_precond}, $\bm{H}_t$ is $\mathcal{F}_t$-measurable. By unbiasedness and the tower property,
\begin{align}
    \mathbb{E}\left[\bar{\bm{g}}_t\mid\mathcal{F}_t\right]=\frac{1}{N}\sum_{i=1}^{N}\nabla f_i(\bm{x}_{i,t}).
\end{align}
Therefore,
\begin{align}
    \mathbb{E}\left\langle\nabla f(\bar{\bm{x}}_t),\bm{H}_t\bar{\bm{g}}_t\right\rangle
    =\mathbb{E}\left\langle\nabla f(\bar{\bm{x}}_t),\bm{H}_t\frac{1}{N}\sum_{i=1}^{N}\nabla f_i(\bm{x}_{i,t})\right\rangle .
\end{align}
Adding and subtracting $\nabla f_i(\bar{\bm{x}}_t)$ and using $N^{-1}\sum_{i=1}^{N}\nabla f_i(\bar{\bm{x}}_t)=\nabla f(\bar{\bm{x}}_t)$, we obtain
\begin{align}
\begin{aligned}
    \mathbb{E}\left\langle\nabla f(\bar{\bm{x}}_t),\bm{H}_t\bar{\bm{g}}_t\right\rangle
    &=\mathbb{E}\left\langle\nabla f(\bar{\bm{x}}_t),\bm{H}_t\nabla f(\bar{\bm{x}}_t)\right\rangle \\
    &\quad+\mathbb{E}\left\langle\nabla f(\bar{\bm{x}}_t),\bm{H}_t\frac{1}{N}\sum_{i=1}^{N}\left(\nabla f_i(\bm{x}_{i,t})-\nabla f_i(\bar{\bm{x}}_t)\right)\right\rangle .
\end{aligned}
\end{align}
The first term is lower bounded by $\mu\mathbb{E}\|\nabla f(\bar{\bm{x}}_t)\|^2$. For the second term, by Cauchy--Schwarz, $\|\bm{H}_t\|\le M$, the triangle inequality, and $L$-smoothness,
\begin{align}
\begin{aligned}
    &\left|\mathbb{E}\left\langle\nabla f(\bar{\bm{x}}_t),\bm{H}_t\frac{1}{N}\sum_{i=1}^{N}\left(\nabla f_i(\bm{x}_{i,t})-\nabla f_i(\bar{\bm{x}}_t)\right)\right\rangle\right| \\
    &\le ML\,\mathbb{E}\left[\|\nabla f(\bar{\bm{x}}_t)\|\cdot\frac{1}{N}\sum_{i=1}^{N}\|\bm{x}_{i,t}-\bar{\bm{x}}_t\|\right].
\end{aligned}
\end{align}
Using Cauchy--Schwarz and Jensen's inequality,
\begin{align}
    ML\,\mathbb{E}\left[\|\nabla f(\bar{\bm{x}}_t)\|\cdot\frac{1}{N}\sum_{i=1}^{N}\|\bm{x}_{i,t}-\bar{\bm{x}}_t\|\right]
    \le ML\sqrt{\mathbb{E}\|\nabla f(\bar{\bm{x}}_t)\|^2}\sqrt{D_t}.
\end{align}
By Young's inequality,
\begin{align}
    ML\sqrt{\mathbb{E}\|\nabla f(\bar{\bm{x}}_t)\|^2}\sqrt{D_t}
    \le \frac{\mu}{2}\mathbb{E}\|\nabla f(\bar{\bm{x}}_t)\|^2+\frac{M^2L^2}{2\mu}D_t .
\end{align}
Combining the above bounds and using Lemma~\ref{lem:local_drift_bound} yields
\begin{equation}
    \mathbb{E}\left\langle\nabla f(\bar{\bm{x}}_t),\bm{H}_t\bar{\bm{g}}_t\right\rangle
    \ge \frac{\mu}{2}\mathbb{E}\|\nabla f(\bar{\bm{x}}_t)\|^2
    -\frac{M^2L^2}{2\mu}\cdot 4\eta^2K^2P_{\rho,\tau}^2 .
    \label{eq:local_term_lower}
\end{equation}

Second, by Cauchy--Schwarz and $\|\bm{\Delta}_t\|\le P_{\rho,\tau}$ from Lemma~\ref{lem:direction_bounded_new},
\begin{align}
    \tau\rho\mathbb{E}\left\langle\nabla f(\bar{\bm{x}}_t),\bm{\Delta}_t\right\rangle
    \ge -\tau\rho P_{\rho,\tau}\mathbb{E}\|\nabla f(\bar{\bm{x}}_t)\|.
\end{align}
Young's inequality with $\delta_1=\tau(1-\rho)\mu/4$ gives
\begin{equation}
\begin{aligned}
    \tau\rho P_{\rho,\tau}\mathbb{E}\|\nabla f(\bar{\bm{x}}_t)\|
    &\le \frac{\delta_1}{2}\mathbb{E}\|\nabla f(\bar{\bm{x}}_t)\|^2+\frac{\tau^2\rho^2P_{\rho,\tau}^2}{2\delta_1} \\
    &=\frac{\tau(1-\rho)\mu}{8}\mathbb{E}\|\nabla f(\bar{\bm{x}}_t)\|^2+\frac{2\tau\rho^2P_{\rho,\tau}^2}{(1-\rho)\mu}.
\end{aligned}
\label{eq:corr_young}
\end{equation}
Third, Young's inequality with $\delta_2=\tau(1-\rho)\mu/4$ gives
\begin{equation}
\begin{aligned}
    C_\tau P_{\rho,\tau}\mathbb{E}\|\nabla f(\bar{\bm{x}}_t)\|
    &\le \frac{\delta_2}{2}\mathbb{E}\|\nabla f(\bar{\bm{x}}_t)\|^2+\frac{C_\tau^2P_{\rho,\tau}^2}{2\delta_2} \\
    &=\frac{\tau(1-\rho)\mu}{8}\mathbb{E}\|\nabla f(\bar{\bm{x}}_t)\|^2+\frac{2C_\tau^2P_{\rho,\tau}^2}{\tau(1-\rho)\mu}.
\end{aligned}
\label{eq:act_residual_young}
\end{equation}
Combining \eqref{eq:A_decomp}, \eqref{eq:local_term_lower}, \eqref{eq:corr_young}, and \eqref{eq:act_residual_young}, the coefficient of $\mathbb{E}\|\nabla f(\bar{\bm{x}}_t)\|^2$ is
\begin{align}
    \tau(1-\rho)\cdot\frac{\mu}{2}-\frac{\tau(1-\rho)\mu}{8}-\frac{\tau(1-\rho)\mu}{8}
    =\frac{\tau(1-\rho)\mu}{4}=a_{\rho,\tau}.
\end{align}
The remaining constant terms give $b_{\rho,\tau}(\eta)$ in \eqref{eq:b_def}. This proves \eqref{eq:descent_lower_final}.
\end{proof}

\begin{remark}[Correction term]
\label{rem:bypass_recursion}
The correction term $\bm{\Delta}_t$ is an average of previous local directions. One can alternatively exploit this structure through a delayed-recursion argument. We use the direct Cauchy--Schwarz bound instead, which gives a simpler proof and leads to the $\mathcal{O}(\rho^2P_{\rho,\tau}^2)$ term in $b_{\rho,\tau}(\eta)$.
\end{remark}

\subsection{Convergence Rate}

\begin{theorem}[Nonconvex convergence to a stationary neighborhood]
\label{thm:fedact_convergence_final}
Suppose Assumptions~\ref{ass:lower_bound}--\ref{ass:bounded_gradient} hold. Let $\lambda=0$ and $0<\rho<\tau\le 1$. Let $a_{\rho,\tau}$, $b_{\rho,\tau}(\eta)$, $q_g$, and $q_0$ be defined in \eqref{eq:a_def}, \eqref{eq:b_def}, \eqref{eq:qg_def}, and \eqref{eq:q0_def}. If
\begin{equation}
    \eta\le \min\left\{\frac{1}{LK},\frac{a_{\rho,\tau}}{Lq_g}\right\},
    \label{eq:stepsize_condition}
\end{equation}
then
\begin{equation}
    \frac{1}{T}\sum_{t=0}^{T-1}\mathbb{E}\|\nabla f(\bar{\bm{x}}_t)\|^2
    \le \frac{2\Delta_f}{\eta T a_{\rho,\tau}}+\frac{2b_{\rho,\tau}(\eta)}{a_{\rho,\tau}}+\frac{L\eta q_0}{a_{\rho,\tau}} .
    \label{eq:final_convergence_bound}
\end{equation}
With $\eta=\Theta(1/\sqrt{T})$, the bound becomes
\begin{equation}
    \frac{1}{T}\sum_{t=0}^{T-1}\mathbb{E}\|\nabla f(\bar{\bm{x}}_t)\|^2
    =\mathcal{O}\!\left(\frac{1}{\sqrt{T}}\right)
    +\mathcal{O}\!\left(
        \frac{\rho^2P_{\rho,\tau}^2}{(1-\rho)\mu}
        +\frac{(1-\tau^2)^2P_{\rho,\tau}^2}{\tau(1-\rho)\mu}
    \right).
    \label{eq:neighborhood_rate}
\end{equation}
The hidden constant in $\mathcal{O}(1/\sqrt{T})$ depends on $L,\sigma^2,\zeta^2,K,\rho,\tau,G,\epsilon,\mu,M$; the dependence on $\rho$ and $\tau$ in the stationary-neighborhood term is shown explicitly. Since $T=RK$, the first term is equivalently $\mathcal{O}(1/\sqrt{RK})$.
\end{theorem}

\begin{proof}
By $L$-smoothness and the averaged update $\bar{\bm{x}}_{t+1}=\bar{\bm{x}}_t-\eta\bar{\bm{p}}_t$,
\begin{align}
    f(\bar{\bm{x}}_{t+1})\le f(\bar{\bm{x}}_t)-\eta\left\langle\nabla f(\bar{\bm{x}}_t),\bar{\bm{p}}_t\right\rangle+\frac{L\eta^2}{2}\|\bar{\bm{p}}_t\|^2 .
\end{align}
Taking expectation and applying Lemma~\ref{lem:descent_lower_bound_new} and Lemma~\ref{lem:second_moment_variance_heterogeneity},
\begin{align}
\begin{aligned}
    \mathbb{E}f(\bar{\bm{x}}_{t+1})
    &\le \mathbb{E}f(\bar{\bm{x}}_t)-\eta a_{\rho,\tau}\mathbb{E}\|\nabla f(\bar{\bm{x}}_t)\|^2+\eta b_{\rho,\tau}(\eta) \\
    &\quad+\frac{L\eta^2}{2}\left(q_g\mathbb{E}\|\nabla f(\bar{\bm{x}}_t)\|^2+q_0\right).
\end{aligned}
\end{align}
By \eqref{eq:stepsize_condition}, $\eta Lq_g/2\le a_{\rho,\tau}/2$, and hence
\begin{align}
    \eta a_{\rho,\tau}-\frac{L\eta^2q_g}{2}\ge \frac{\eta a_{\rho,\tau}}{2}.
\end{align}
Rearranging gives
\begin{align}
    \frac{\eta a_{\rho,\tau}}{2}\mathbb{E}\|\nabla f(\bar{\bm{x}}_t)\|^2
    \le \mathbb{E}f(\bar{\bm{x}}_t)-\mathbb{E}f(\bar{\bm{x}}_{t+1})+\eta b_{\rho,\tau}(\eta)+\frac{L\eta^2q_0}{2}.
\end{align}
Summing over $t=0,\ldots,T-1$, telescoping, and using $f(\bar{\bm{x}}_T)\ge f^\star$,
\begin{align}
    \frac{\eta a_{\rho,\tau}}{2}\sum_{t=0}^{T-1}\mathbb{E}\|\nabla f(\bar{\bm{x}}_t)\|^2
    \le \Delta_f+T\eta b_{\rho,\tau}(\eta)+\frac{TL\eta^2q_0}{2}.
\end{align}
Dividing by $\eta Ta_{\rho,\tau}/2$ yields \eqref{eq:final_convergence_bound}. When $\eta=\Theta(1/\sqrt{T})$, the first term is $\mathcal{O}(1/\sqrt{T})$. The local-drift contribution in $b_{\rho,\tau}(\eta)$ is $\mathcal{O}(\eta^2)$, hence contributes $\mathcal{O}(1/T)$ to \eqref{eq:final_convergence_bound}. The constant components of $b_{\rho,\tau}(\eta)$ contribute the stationary-neighborhood terms in \eqref{eq:neighborhood_rate}. The third term $L\eta q_0/a_{\rho,\tau}$ is $\mathcal{O}(1/\sqrt{T})$. Combining these terms gives \eqref{eq:neighborhood_rate}.
\end{proof}

\begin{proposition}[Choice of the trust ratio $\tau$]
\label{rem:tau_choice}
Theorem~\ref{thm:fedact_convergence_final} allows any fixed $0<\rho<\tau\le1$.
With $\eta=\Theta(1/\sqrt{T})$, \method satisfies
\begin{align}
    \frac{1}{T}\sum_{t=0}^{T-1}
    \mathbb{E}\|\nabla f(\bar{\bm{x}}_t)\|^2
    =
    \mathcal{O}\!\left(\frac{1}{\sqrt{T}}\right)
    +
    \mathcal{B}_{\rho,\tau},
\end{align}
where
\begin{align}
    \mathcal{B}_{\rho,\tau}
    =
    \mathcal{O}\!\left(
        \frac{\rho^2P_{\rho,\tau}^2}{(1-\rho)\mu}
        +
        \frac{(1-\tau^2)^2P_{\rho,\tau}^2}{\tau(1-\rho)\mu}
    \right),
    \quad
    P_{\rho,\tau}
    =
    \frac{(1-\rho)MG}{\tau-\rho}.
\end{align}
Thus, for fixed $\rho$ and $\tau$, \method converges to a stationary
neighborhood at the rate $\mathcal{O}(1/\sqrt{T})$, or equivalently
$\mathcal{O}(1/\sqrt{RK})$.

The parameter $\tau$ controls the strength of coordinate modulation and the
size of the worst-case neighborhood. A smaller $\tau$ yields stronger
coordinate modulation, while a larger $\tau$ shrinks
$\mathcal{B}_{\rho,\tau}$. If exact asymptotic convergence is desired, one may
choose $\rho=\rho_T$ and $\tau=\tau_T$ such that
\begin{align}
    \rho_T=\mathcal{O}(T^{-1/4}),
    \quad
    1-\tau_T=\mathcal{O}(T^{-1/4}),
\end{align}
with $\tau_T-\rho_T$ bounded away from zero. In this case
$P_{\rho_T,\tau_T}=\mathcal{O}(1)$ and
\begin{align}
    \mathcal{B}_{\rho_T,\tau_T}
    =
    \mathcal{O}\!\left(\frac{1}{\sqrt{T}}\right).
\end{align}
Consequently,
\begin{align}
    \frac{1}{T}\sum_{t=0}^{T-1}
    \mathbb{E}\|\nabla f(\bar{\bm{x}}_t)\|^2
    =
    \mathcal{O}\!\left(\frac{1}{\sqrt{T}}\right)
    =
    \mathcal{O}\!\left(\frac{1}{\sqrt{RK}}\right).
\end{align}
\end{proposition}

\begin{remark}[Degenerate case]
\label{rem:degenerate}
When $\rho=0$ and $\tau=1$, the correction term disappears and $\bm{\phi}_{i,t}=\bm{1}$. Hence
\begin{align}
    \bm{p}_{i,t}=\bm{H}_t\bm{g}_{i,t}.
\end{align}
Both neighborhood terms in \eqref{eq:neighborhood_rate} vanish, and the result reduces to the standard nonconvex preconditioned local-SGD bound with stochastic variance and client heterogeneity.
\end{remark}

}

\section{Appendix C: Additional Experimental Details}

\subsection{More Details on Datasets}
We evaluate \method on both vision and language tasks, covering image classification, federated LLM pre-training, and federated LLM fine-tuning. The vision benchmarks are used to evaluate federated image classification under controlled statistical heterogeneity, while the language benchmarks are used to assess the effectiveness of \method in more challenging federated LLM training scenarios.

\begin{itemize}
    \item \textbf{CIFAR-10}~\cite{krizhevsky2009learning}: 
    CIFAR-10 contains 10 object categories with 6,000 color images per class at a resolution of $32 \times 32$. It is a standard benchmark for image classification and is used to evaluate federated learning under relatively simple visual recognition settings.
    \item \textbf{CIFAR-100}~\cite{krizhevsky2009learning}: 
    CIFAR-100 has the same image resolution and total number of images as CIFAR-10, but contains 100 fine-grained classes with 600 images per class. Compared with CIFAR-10, it provides a more challenging classification task due to the larger label space and fewer images per class.
    \item \textbf{Tiny ImageNet}~\cite{le2015tiny}: 
    Tiny ImageNet is a compact subset of ImageNet with 200 classes and $64 \times 64$ images. Each class contains 500 training images, making it a more challenging vision benchmark than CIFAR-10 and CIFAR-100 in terms of both image resolution and number of categories.
    \item \textbf{C4-en}~\cite{raffel2020exploring}: 
    C4-en is a cleaned English corpus extracted from Common Crawl. We use it for federated LLM pre-training to evaluate whether \method remains effective when dense AdamW updates are repeatedly applied over long training horizons.
    \item \textbf{Alpaca-GPT4}\footnote{\url{https://huggingface.co/datasets/vicgalle/alpaca-gpt4}}: 
    Alpaca-GPT4 is an instruction-following dataset generated with GPT-4~\cite{achiam2023gpt} using the Self-Instruct procedure~\cite{wang2023self}. We use it for federated instruction tuning (FedIT), where clients fine-tune language models on decentralized instruction-following data.
    \item \textbf{HH-RLHF}\footnote{\url{https://huggingface.co/datasets/Anthropic/hh-rlhf}}: 
    HH-RLHF consists of human preference data for helpfulness and harmlessness~\cite{bai2022training}, together with red-teaming data~\cite{ganguli2022red}. We use it for federated value alignment (FedVA), where clients optimize models using decentralized preference or alignment data.
\end{itemize}

For vision tasks, we simulate statistical heterogeneity using Dirichlet partitioning~\cite{hsu2019measuring}. A smaller Dirichlet concentration parameter indicates stronger label distribution skew across clients. In particular, Dir-0.6 represents relatively mild heterogeneity, while Dir-0.1 corresponds to a highly non-IID setting. This allows us to evaluate whether \method remains robust as client data distributions become increasingly imbalanced.

\begin{table}[!ht]
    \centering
    \caption{A detailed summary of the image classification datasets, including the number of classes, image size, dataset splits, and number of training images per class.}
    \label{tab:dataset_summary}
    \vskip -0.2cm
    \resizebox{0.75\linewidth}{!}{%
    \begin{tabular}{lccccccc}
      \toprule
      \textbf{Dataset} & \textbf{\#Classes} & \textbf{Image Size}
      & \textbf{Train} & \textbf{Val} & \textbf{Test}
      & \textbf{Total} & \textbf{Train / Class} \\
      \midrule
      CIFAR-10 
        & 10  & $3 \times 32 \times 32$ & 50,000 & -- & 10,000 & 60,000 & 5,000 \\
      CIFAR-100
        & 100 & $3 \times 32 \times 32$ & 50,000 & -- & 10,000 & 60,000 & 500 \\
      Tiny-ImageNet
        & 200 & $3 \times 64 \times 64$ & 100,000 & 10,000 & 10,000 & 120,000 & 500 \\
      \bottomrule
    \end{tabular}
    }
\end{table}

\subsection{More Training Details}

\paragraph{Image Classification Task.} For image classification, we tune the learning rate of FedAvg, FedProx, SCAFFOLD, and FedAdam from $\{10^{-2}, 3\times10^{-2}, 5\times10^{-2}, 10^{-1}, 3\times10^{-1}\}$ with weight decay $0.001$. For FedLADA, LocalAdamW, FedAdamW, and \method, we tune the learning rate from $\{10^{-4}, 3\times10^{-4}, 5\times10^{-4}, 8\times10^{-4}, 10^{-3}\}$, use $\beta_1=0.9$ and $\beta_2=0.999$, and search weight decay from $\{0.001, 0.01\}$. We apply cosine learning-rate decay for all vision experiments. Unless otherwise specified, image classification experiments use 100 clients, 10\% client participation, batch size 50, 50 local steps, and 300 communication rounds.

\paragraph{Federated LLM Pre-Training Task.} For federated LLM pre-training, we follow the Chinchilla scaling law~\citep{hoffmann2022training} by allocating approximately 20 training tokens per model parameter. We use a learning rate of $3\times10^{-4}$ with AdamW momentum parameters $\beta_1=0.9$ and $\beta_2=0.99$. In the federated setting, we use 20 clients with 20\% client participation per communication round, a batch size of 64, a sequence length of 1024, and 50 local steps on each selected client. Since different model sizes require different token budgets, we adjust the number of communication rounds according to the model scale to match the Chinchilla-style token budget.

\paragraph{Federated LLM Fine-Tuning Task.} For federated instruction tuning, following prior federated LLM benchmarks~\cite{ye2024openfedllm, ye2024fedllm}, we fine-tune Llama2-7B with LoRA on Alpaca-GPT4 for 200 communication rounds. The learning rate is decayed from $5\times10^{-5}$ to $1\times10^{-6}$, and the LoRA rank and scaling factor are set to 32 and 64, respectively. We use the Alpaca template~\cite{taori2023stanford} for instruction formatting.

For federated value alignment, we fine-tune an uncensored Wizard-Vicuna-style instruction-following model~\footnote{\url{https://huggingface.co/QuixiAI/Wizard-Vicuna-7B-Uncensored}} on HH-RLHF for 200 communication rounds. The learning rate is decayed from $5\times10^{-4}$ to $1\times10^{-5}$, and the LoRA rank and scaling factor are set to 8 and 16, respectively. We use the Vicuna template~\cite{chiang2023vicuna} for dialogue formatting. 

Following common practice in LLM pre-training and fine-tuning, we use AdamW~\cite{loshchilov2017decoupled} as the local optimizer for all federated LLM methods, except FedLADA, which follows its original Adam-based update rule~\citep{sun2023efficient}. The key hyperparameters used in our experiments are summarized in Tables~\ref{tab:hyperparams_resnet_vit_swin},~\ref{tab:hyperparams_llama}, and~\ref{tab:hyperparams_fine_tuning}.

\paragraph{Experiment Environments.} All experiments were conducted on GPU servers equipped with NVIDIA RTX 4090 or A100 GPUs, using PyTorch 2.4.1 and CUDA 12.4. We set the random seed to 42 for all experiments.

\begin{table}[!ht]
    \centering
    \caption{Hyperparameter configuration of ResNet-18, ViT-Tiny, and Swin-Lite on CIFAR-10, CIFAR-100 and Tiny-ImageNet across different algorithms.}
    \label{tab:hyperparams_resnet_vit_swin}
    \vskip -0.2cm
    \resizebox{0.85\linewidth}{!}{%
    \begin{tabular}{lcccccccc}
      \toprule
      \textbf{Method} & \textbf{Local Optimizer} & \textbf{Global Optimizer}
      & \textbf{Local LR} & \textbf{Global LR}
      & $\boldsymbol{\rho}$ & $\boldsymbol{\beta_1}$ & $\boldsymbol{\beta_2}$
      & \textbf{Weight Decay} \\
      \midrule
      FedAvg (Local SGD) & SGD   & SGD  & 0.1  & 1.0  & --  & --  & --    & 0.001 \\
      SCAFFOLD           & SGD   & SGD  & 0.1  & 1.0  & --  & --  & --    & 0.001 \\
      FedProx            & SGD   & SGD  & 0.1  & 1.0  & --  & --  & --    & 0.001 \\
      FedAdam            & SGD   & Adam & 0.1  & 0.01 & --  & 0.9 & 0.98  & 0.001 \\
      LocalAdam          & Adam  & SGD  & 3e-4 & 1.0  & --  & 0.9 & 0.999 & 0.001 \\
      FedLADA            & Adam  & SGD  & 3e-4 & 1.0  & 0.5 & 0.9 & 0.999 & 0.001 \\
      LocalAdamW         & AdamW & SGD  & 3e-4 & 1.0  & --  & 0.9 & 0.999 & 0.01  \\
      FedAdamW           & AdamW & SGD  & 3e-4 & 1.0  & 0.5 & 0.9 & 0.999 & 0.01  \\
      \rowcolor{mygray}
      FedACT             & AdamW & SGD  & 3e-4 & 1.0  & 0.5 & 0.9 & 0.999 & 0.01  \\
      \bottomrule
    \end{tabular}
    }
\end{table}

\begin{table}[!ht]
    \vskip -0.2cm
    \centering
    \caption{Hyperparameter configuration of Llama2-60M/130M/250M on C4-en across different algorithms.}
    \label{tab:hyperparams_llama}
    \vskip -0.2cm
    \resizebox{0.8\linewidth}{!}{%
    \begin{tabular}{lcccccccc}
      \toprule
      \textbf{Method} & \textbf{Local Optimizer} & \textbf{Global Optimizer}
      & \textbf{Local LR} & \textbf{Global LR} & $\boldsymbol{\rho}$
      & $\boldsymbol{\beta_1}$ & $\boldsymbol{\beta_2}$ & \textbf{Weight Decay} \\
      \midrule
      SCAFFOLD           & AdamW & SGD & 3e-4 & 1.0 & --  & 0.9 & 0.99 & 0.01 \\
      FedProx            & AdamW & SGD & 3e-4 & 1.0 & --  & 0.9 & 0.99 & 0.01 \\
      FedLADA            & Adam  & SGD & 3e-3 & 1.0 & 0.5 & 0.9 & 0.99 & 0.0  \\
      LocalAdamW         & AdamW & SGD & 3e-4 & 1.0 & --  & 0.9 & 0.99 & 0.01 \\
      FedAdamW           & AdamW & SGD & 3e-4 & 1.0 & 0.5 & 0.9 & 0.99 & 0.01 \\
      \rowcolor{mygray}
      FedACT             & AdamW & SGD & 3e-4 & 1.0 & 0.5 & 0.9 & 0.99 & 0.01 \\
      \bottomrule
    \end{tabular}
    }
\end{table}

\begin{table}[!ht]
    \centering
    \caption{Hyperparameter configuration of FedIT and FedVA on Alpaca-GPT4 and HH-RLHF across different algorithms.}
    \label{tab:hyperparams_fine_tuning}
    \vskip -0.2cm
    \resizebox{0.8\linewidth}{!}{%
    \begin{tabular}{lcccccccc}
      \toprule
      \textbf{Method} & \textbf{Local Optimizer} & \textbf{Global Optimizer}
      & \textbf{Local LR} & \textbf{Global LR} & $\boldsymbol{\rho}$
      & $\boldsymbol{\beta_1}$ & $\boldsymbol{\beta_2}$ & \textbf{Weight Decay} \\
      \midrule
      \multicolumn{9}{c}{\textbf{Alpaca-GPT4/SFT}} \\
      \midrule
      SCAFFOLD           & AdamW & SGD  & 5e-5 & 1.0  & --  & 0.9 & 0.999 & 0.01  \\
      FedProx            & AdamW & SGD  & 5e-5 & 1.0  & --  & 0.9 & 0.999 & 0.01  \\
      FedAdam            & AdamW & Adam & 5e-5 & 1e-3 & --  & 0.9 & 0.999 & 0.01  \\
      FedLADA            & Adam  & SGD  & 5e-5 & 1.0  & 0.5 & 0.9 & 0.999 & 0.001 \\
      LocalAdamW         & AdamW & SGD  & 5e-5 & 1.0  & --  & 0.9 & 0.999 & 0.01  \\
      FedAdamW           & AdamW & SGD  & 5e-5 & 1.0  & 0.5 & 0.9 & 0.999 & 0.01  \\
      \rowcolor{mygray}
      FedACT             & AdamW & SGD  & 5e-5 & 1.0  & 0.5 & 0.9 & 0.999 & 0.01  \\
      \midrule
      \multicolumn{9}{c}{\textbf{HH-RLHF/DPO}} \\
      \midrule
      SCAFFOLD           & AdamW & SGD  & 5e-4 & 1.0  & --  & 0.9 & 0.999 & 0.01  \\
      FedProx            & AdamW & SGD  & 5e-4 & 1.0  & --  & 0.9 & 0.999 & 0.01  \\
      FedAdam            & AdamW & Adam & 5e-4 & 1e-3 & --  & 0.9 & 0.999 & 0.01  \\
      FedLADA            & Adam  & SGD  & 5e-4 & 1.0  & 0.5 & 0.9 & 0.999 & 0.001 \\
      LocalAdamW         & AdamW & SGD  & 5e-4 & 1.0  & --  & 0.9 & 0.999 & 0.01  \\
      FedAdamW           & AdamW & SGD  & 5e-4 & 1.0  & 0.5 & 0.9 & 0.999 & 0.01  \\
      \rowcolor{mygray}
      FedACT             & AdamW & SGD  & 5e-4 & 1.0 & 0.5  & 0.9 & 0.999 & 0.01  \\
      \bottomrule
    \end{tabular}
    }
\end{table}

\section{Appendix D: Additional Ablation Study}

\paragraph{Sensitivity to Global Guidance Strength $\rho$.}
We study the sensitivity of \method to the global guidance strength $\rho$ in Table~\ref{tab:abs_global_corr_coef}. The parameter $\rho$ controls the interpolation between the local AdamW direction and the global correction signal: $\rho=0$ reduces the method to local-only trust modulation, while larger values introduce stronger global guidance. The results show that a moderate value is consistently preferable, with $\rho=0.50$ achieving the best accuracy across both ViT-Tiny and Swin-Lite under Dir-0.6 and Dir-0.1. Compared with $\rho=0$, using $\rho=0.50$ clearly improves performance, e.g., from 40.85\% to 43.58\% on ViT-Tiny under Dir-0.6 and from 44.21\% to 48.50\% on Swin-Lite under Dir-0.1, confirming the importance of global guidance. However, overly large $\rho$ values cause severe degradation, suggesting that the global correction signal should guide rather than dominate local adaptive updates. We therefore use $\rho=0.50$ as the default setting.

\begin{table}[!ht]
    \centering
    \caption{Accuracy (\%) of FedACT with different global guidance strengths $\rho$ on ViT-Tiny and Swin-Lite over 300 communication rounds on CIFAR-100.}
    \label{tab:abs_global_corr_coef}
    \vskip -0.2cm
    \resizebox{0.5\linewidth}{!}{
        \begin{tabular}{lcccccccc}
  \toprule
  \multirow{2}{*}{\textbf{Setting}}
  & \multicolumn{8}{c}{\textbf{$\bm{\rho}$}} \\
  \cmidrule(lr){2-9}
  & \textbf{0.00}
  & \textbf{0.10}
  & \textbf{0.25}
  & \textbf{0.50}
  & \textbf{0.60}
  & \textbf{0.75}
  & \textbf{0.80}
  & \textbf{0.90} \\
  \midrule
  \multicolumn{9}{l}{\textbf{ViT-Tiny}} \\
  Dir-0.6
    & 40.85 & 41.35 & 42.08 & \textbf{43.58} & 40.93 & 18.11 & 16.44 & 12.57 \\
  Dir-0.1
    & 39.24 & 39.22 & 40.59 & \textbf{42.03} & 24.06 & 11.83 & 9.79 & 8.01 \\
  \midrule
  \multicolumn{9}{l}{\textbf{Swin-Lite}} \\
  Dir-0.6
    & 48.40 & 49.69 & 50.63 & \textbf{52.23} & 51.87 & 11.93 & 10.00 & 8.08 \\
  Dir-0.1
    & 44.21 & 44.90 & 45.84 & \textbf{48.50} & 18.19 & 7.22 & 6.93 & 4.92 \\
  \bottomrule
\end{tabular}
    }
\end{table}

\section{Appendix E: Additional Detailed Experimental Results}
\label{appendix:more_vis_results}

\paragraph{Detailed Results on CNN.} Table~\ref{tab:main_results_resnet18} reports the results on ResNet-18. \method also improves over FedAdamW across all nine ResNet-18 settings, indicating that coordinate-level trust allocation remains compatible with CNN-based federated training. However, the margins are smaller than those observed on ViT-Tiny and Swin-Lite. For example, the improvements over FedAdamW on CIFAR-10 are 0.27, 0.46, and 0.22 points under Dir-0.6, Dir-0.3, and Dir-0.1, respectively, while the corresponding improvements on CIFAR-100 are 1.41, 0.76, and 1.18 points. On Tiny-ImageNet, \method improves over FedAdamW by 0.40, 0.91, and 0.44 points under Dir-0.6, Dir-0.3, and Dir-0.1, respectively. This contrast is consistent with our motivation: coordinate-level trust allocation can refine corrected AdamW updates beyond Transformer architectures, but its advantage becomes more pronounced when local adaptive updates are stronger, denser, and more heterogeneous, as in federated Transformer training.

\begin{table}[!ht]
    \centering
    \caption{Accuracy (\%) of all methods on ResNet-18 over 300 communication rounds under different datasets and heterogeneity settings.}
    \label{tab:main_results_resnet18}
    \vskip -0.2cm
    \resizebox{0.7\linewidth}{!}{
        {
\begin{tabular}{lccccccccc}
  \toprule
  \multirow{2}{*}{\textbf{Method}}
  & \multicolumn{3}{c}{\textbf{CIFAR-10}}
  & \multicolumn{3}{c}{\textbf{CIFAR-100}}
  & \multicolumn{3}{c}{\textbf{Tiny-ImageNet}} \\
  \cmidrule(lr){2-4} \cmidrule(lr){5-7} \cmidrule(lr){8-10}
  & Dir-0.6 & Dir-0.3 & Dir-0.1
  & Dir-0.6 & Dir-0.3 & Dir-0.1
  & Dir-0.6 & Dir-0.3 & Dir-0.1 \\
  \midrule
  FedAvg
    & 89.76 & 88.29 & 83.52
    & 64.93 & 64.50 & 60.88
    & 53.59 & 51.79 & 48.38 \\
  SCAFFOLD
    & 89.45 & 88.30 & 83.69
    & 64.96 & 64.58 & 61.43
    & 53.78 & 52.41 & 48.22 \\
  FedProx
    & 89.64 & 87.76 & 82.91
    & 64.54 & 64.28 & 60.33
    & 53.07 & 51.46 & 47.69 \\
  LocalAdam
    & 88.59 & 87.29 & 82.05
    & 63.95 & 63.70 & 59.69
    & 50.82 & 48.99 & 45.13 \\
  FedAdam
    & 87.60 & 86.87 & 77.92
    & 62.55 & 62.21 & 57.82
    & 50.87 & 49.61 & 45.36 \\
  FedLADA
    & 89.78 & 88.77 & 83.82
    & 64.30 & 63.95 & 61.31
    & 54.04 & 53.01 & 50.50 \\
  LocalAdamW
    & 89.06 & 87.08 & 80.86
    & 64.23 & 63.75 & 59.78
    & 50.62 & 49.46 & 44.95 \\
  FedAdamW
    & 90.69 & 89.37 & 84.27
    & 65.43 & 65.86 & 63.24
    & 53.66 & 52.71 & 50.65 \\
  \rowcolor{mygray}
  \textbf{Ours}
    & \textbf{90.96} & \textbf{89.83} & \textbf{84.49}
    & \textbf{66.84} & \textbf{66.62} & \textbf{64.42}
    & \textbf{54.06} & \textbf{53.62} & \textbf{51.09} \\
  \bottomrule
\end{tabular}
}
    }
\end{table}

\paragraph{Detailed Results on Federated LLM Fine-Tuning.} Table~\ref{tab:main_results_llm_sft_dpo} evaluates \method on federated LLM fine-tuning using held-out/test metrics that are directly aligned with the supervised fine-tuning (SFT) and direct preference optimization (DPO) training objectives. On Alpaca-GPT4/SFT, \method achieves the best loss, perplexity (PPL), token accuracy, and response negative log-likelihood (NLL) among all compared methods. Compared with FedAdamW, \method reduces loss from 0.9047 to 0.8979, reduces PPL from 2.4712 to 2.4544, improves token accuracy from 73.31 to 73.47, and reduces response NLL from 0.8735 to 0.8613. On HH-RLHF/DPO, \method also obtains the best loss, preference accuracy, reward margin, and policy gap, improving over FedAdamW on all four metrics.

Figure~\ref{fig:loss_sft_dpo_appendix} further compares the training loss trajectories. On Alpaca-GPT4/SFT, \method reaches a lower global training loss than the baselines after the early training stage and maintains a clear advantage in later rounds. On HH-RLHF/DPO, the advantage is more pronounced: \method descends faster and remains below the other baselines for most of training. These curves indicate that \method improves not only held-out/test metrics, but also the optimization of the federated training objective.

These results suggest that \method can be applied to parameter-efficient federated LLM fine-tuning without degrading training stability, and can provide consistent objective-level improvements.

\begin{table*}[!ht]
  \centering
  \caption{Federated LLM fine-tuning performance of all methods after 200 communication rounds. Supervised fine-tuning (SFT) results are evaluated on held-out Alpaca-GPT4 under federated instruction tuning (FedIT; 20 clients, 10\% participation, batch size 16, $K=10$), and direct preference optimization (DPO) results are evaluated on the HH-RLHF test set under federated value alignment (FedVA; 5 clients, 40\% participation, batch size 16, $K=10$). Base denotes the corresponding task-specific base model for each setting.}
  \label{tab:main_results_llm_sft_dpo}
  \vskip -0.2cm
  \resizebox{0.8\linewidth}{!}{
    \begin{tabular}{lcccccccc}
  \toprule
  \multirow{2}{*}{\textbf{Method}}
  & \multicolumn{4}{c}{\textbf{Alpaca-GPT4 / SFT}}
  & \multicolumn{4}{c}{\textbf{HH-RLHF / DPO}} \\
  \cmidrule(lr){2-5} \cmidrule(lr){6-9}
  & \textbf{Loss $\downarrow$}
  & \textbf{PPL $\downarrow$}
  & \textbf{Token Acc. $\uparrow$}
  & \textbf{Resp. NLL $\downarrow$}  
  & \textbf{Loss $\downarrow$}
  & \textbf{Pref. Acc. $\uparrow$}   
  & \textbf{Reward Marg. $\uparrow$} 
  & \textbf{Policy Gap $\uparrow$} \\
  \midrule
  Base
    & 1.0682 & 2.9102 & 70.32 & 1.3236
    & 0.6931 & -     & 0.0000 & 0.1182 \\
  FedAvg
    & 0.8997 & 2.4589 & 73.42 & 0.8650
    & 0.6880 & 55.59 & 0.0113 & 0.2308 \\
  FedProx
    & 0.9002 & 2.4600 & 73.41 & 0.8653
    & 0.6881 & 55.54 & 0.0113 & 0.2309 \\
  SCAFFOLD
    & 0.9007 & 2.4614 & 73.40 & 0.8674
    & 0.6888 & 55.61 & 0.0094 & 0.2122 \\
  FedAdam
    & 0.9046 & 2.4708 & 73.31 & 0.8749
    & 0.6895 & 55.04 & 0.0080 & 0.1977 \\
  FedLADA
    & 0.9112 & 2.4873 & 73.16 & 0.8850
    & 0.6902 & 54.31 & 0.0064 & 0.1820 \\
  FedAdamW
    & 0.9047 & 2.4712 & 73.31 & 0.8735
    & 0.6891 & 54.90 & 0.0090 & 0.2080 \\
  \rowcolor{mygray}
  \textbf{Ours}
    & \textbf{0.8979} & \textbf{2.4544} & \textbf{73.47} & \textbf{0.8613}
    & \textbf{0.6866} & \textbf{55.95} & \textbf{0.0147} & \textbf{0.2656} \\
  \bottomrule
\end{tabular}

  }
\end{table*}

\begin{figure}[!ht]
    \centering
    \includegraphics[width=1\linewidth]{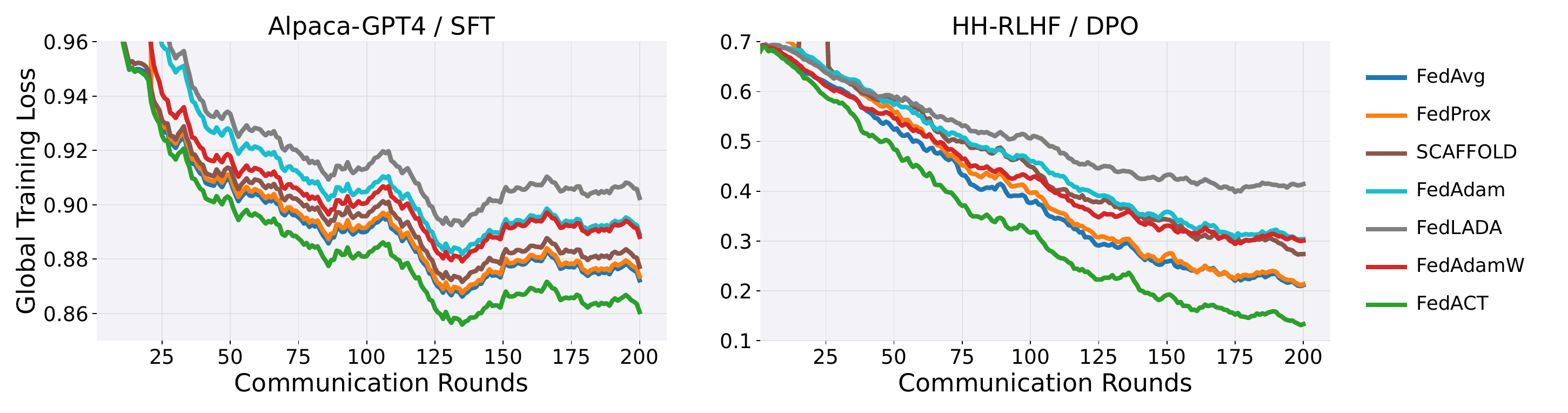}
    \vskip -0.2cm
    \caption{Training dynamics of different methods for federated instruction tuning (FedIT) and federated value alignment (FedVA).}
    \label{fig:loss_sft_dpo_appendix}
\end{figure}

Figures~\ref{fig:appendix_swin_lite_training},~\ref{fig:appendix_vit_tiny_training}, and~\ref{fig:appendix_resnet18_training} provide additional training dynamics in terms of test Top-1 accuracy over communication rounds.

\section{Appendix F: Limitations and Future Work}
\label{app:limitations}

\paragraph{Scope of Optimizer Applicability.}
\method is designed for AdamW-style diagonal adaptive optimization, where each coordinate has an independent adaptive scale and coordinate-wise reliability is therefore a natural object to estimate. This design matches the main setting studied in this work, namely federated Transformer training with local AdamW. However, \method should not be viewed as an optimizer-agnostic update rule. For matrix-structured optimizers such as Muon~\cite{jordan2024muon}, the update direction is obtained through matrix orthogonalization, and the useful optimization geometry is encoded at the matrix or block level rather than in independent coordinates. Applying coordinate-wise reweighting after orthogonalization may disrupt this structured geometry and weaken the benefits of the base optimizer. Extending \method to such optimizers likely requires block-wise or matrix-aware trust modulation, where reliability is estimated and applied at the level of rows, columns, blocks, or full matrices. We leave this direction for future work.

\paragraph{Scale of LLM Pre-Training Experiments.}
Due to computational resource limitations, our federated LLM pre-training experiments are conducted on relatively small-scale Llama2 models with 60M, 130M, and 250M parameters. These experiments allow us to systematically evaluate training dynamics, validation perplexity, and communication-round efficiency under controlled federated settings. However, they do not fully cover the regime of larger foundation models with billions of parameters or substantially longer pre-training horizons. Extending the evaluation of \method to billion-scale LLMs, larger token budgets, and more diverse federated language corpora is an important direction for future work.

\paragraph{Hyperparameter Adaptivity.}
In this work, we use fixed trust-modulation hyperparameters, including the global guidance strength $\rho$, the trust ratio $\tau$, and the attenuation factor $\gamma$. Our sensitivity studies show that moderate global guidance and soft attenuation are important for stable performance, and the default setting performs consistently in the evaluated settings. Nevertheless, the optimal trust allocation may depend on the model architecture, heterogeneity level, training stage, and optimizer dynamics. A promising future direction is to develop adaptive schedules for $\rho$, $\tau$, and $\gamma$, or to estimate them online from client update statistics, direction consistency, or coordinate-level trust distributions.

\paragraph{Theoretical Assumptions.}
Our convergence analysis focuses on a simplified but standard setting for local adaptive federated optimization. In particular, the analysis considers full client participation, removes decoupled weight decay, and uses a predictable diagonal preconditioner assumption. These assumptions make the analysis tractable and help isolate the effect of global-aware coordinate trust modulation. However, practical federated training may involve partial participation, client sampling variability, stale or noisy global correction signals, and additional system constraints. Extending the theory to these more realistic settings would further strengthen the understanding of \method.

\paragraph{Systems and Privacy Considerations.}
\method introduces coordinate-level modulation on top of corrected local AdamW updates and can be implemented without changing the overall federated communication protocol. However, this work does not specifically study interactions with systems-level techniques such as gradient compression, quantization, secure aggregation, or differential privacy. Since these techniques may alter update magnitudes or inject noise into local updates, they could affect the quality of the coordinate-wise trust signal. Studying how global-aware trust modulation interacts with communication-efficient and privacy-preserving federated training is an important direction for future work.

\begin{figure*}[!ht]
    \centering
    \begin{subfigure}{0.33\textwidth}
        \centering
        \includegraphics[width=1\textwidth]{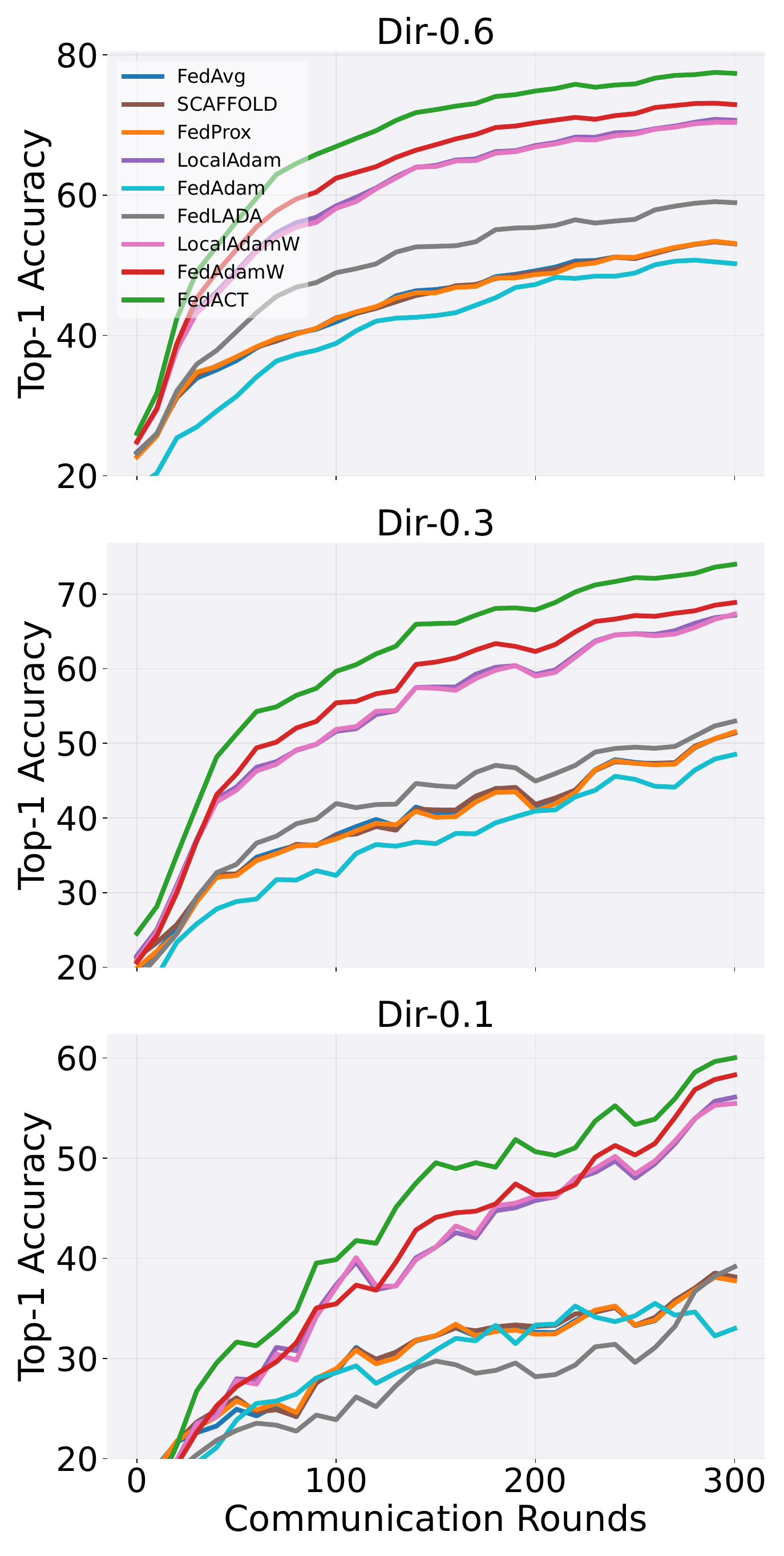}
        \caption{CIFAR-10}
    \end{subfigure}\hfill
    \begin{subfigure}{0.33\textwidth}
        \centering
        \includegraphics[width=1\textwidth]{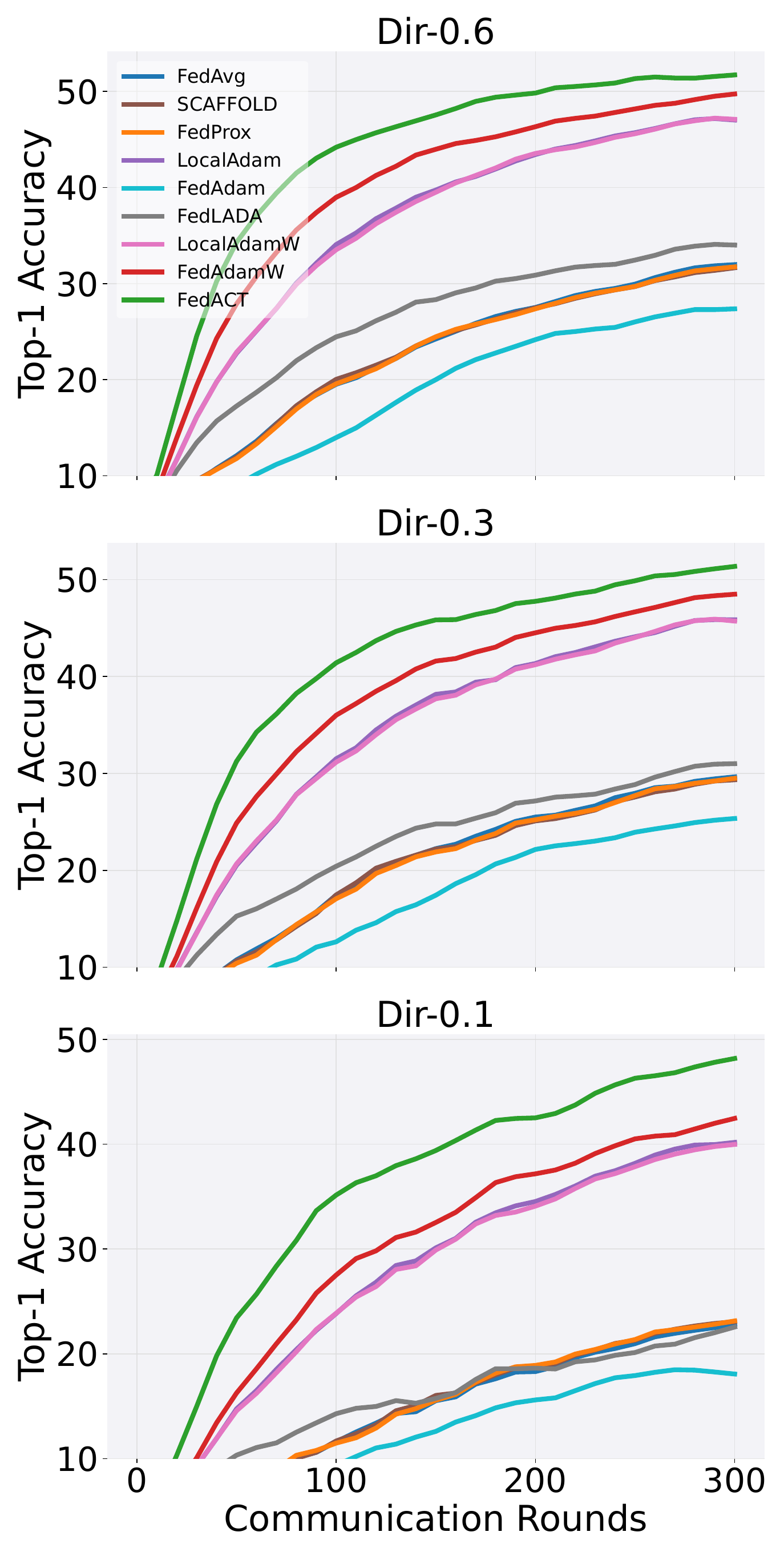}
        \caption{CIFAR-100}
    \end{subfigure}\hfill
    \begin{subfigure}{0.33\textwidth}
        \centering
        \includegraphics[width=1\textwidth]{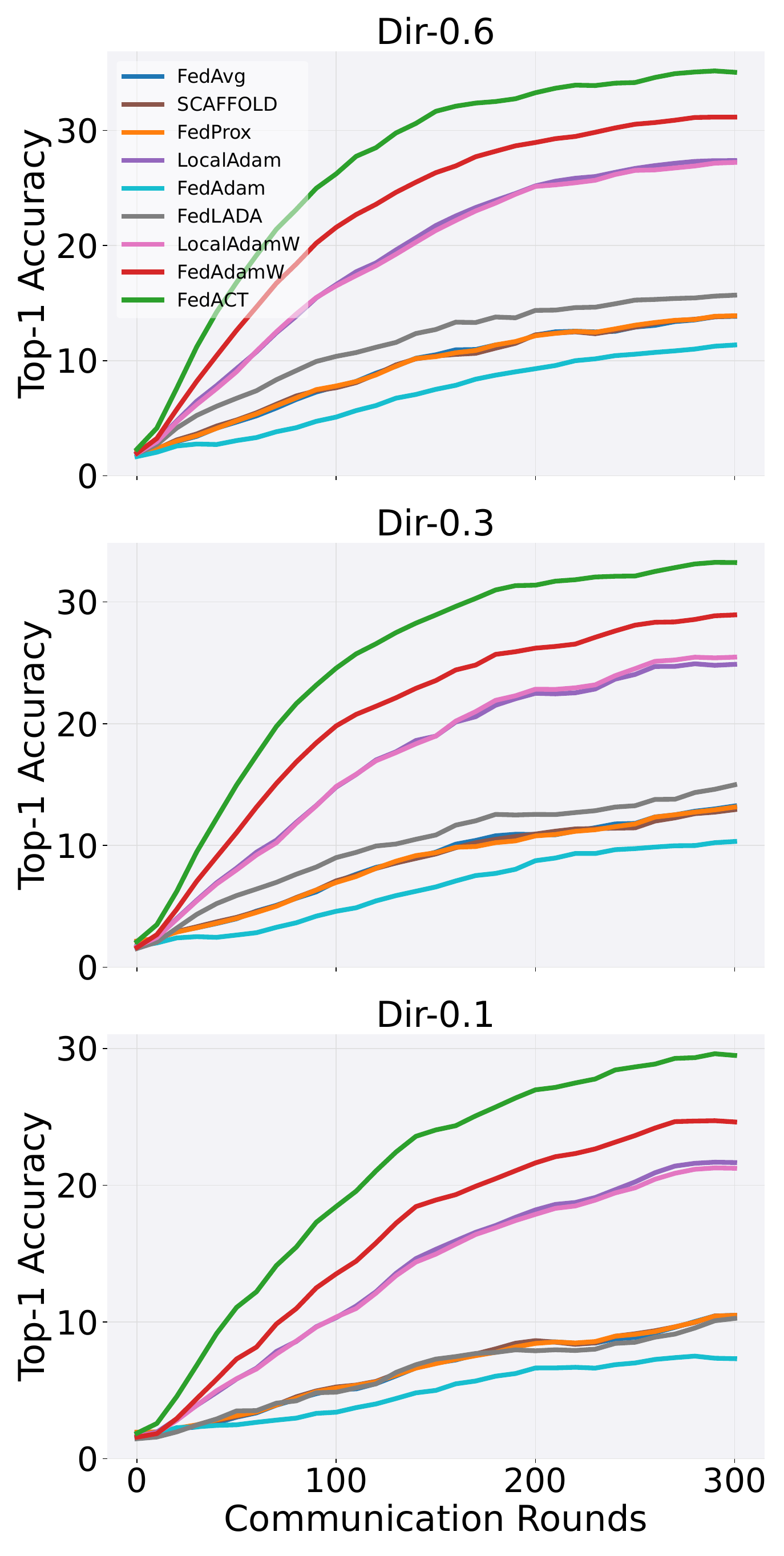}
        \caption{Tiny-ImageNet}
    \end{subfigure}
    \caption{\textbf{Test Top-1 accuracy over communication rounds for federated Swin-Lite training.} Results are reported on CIFAR-10, CIFAR-100, and Tiny-ImageNet under Dirichlet heterogeneity levels Dir-0.6, Dir-0.3, and Dir-0.1. \method consistently improves convergence and final accuracy across datasets and heterogeneity levels.}
    \label{fig:appendix_swin_lite_training}
\end{figure*}

\begin{figure*}[!ht]
    \centering
    \begin{subfigure}{0.33\textwidth}
        \centering
        \includegraphics[width=1\textwidth]{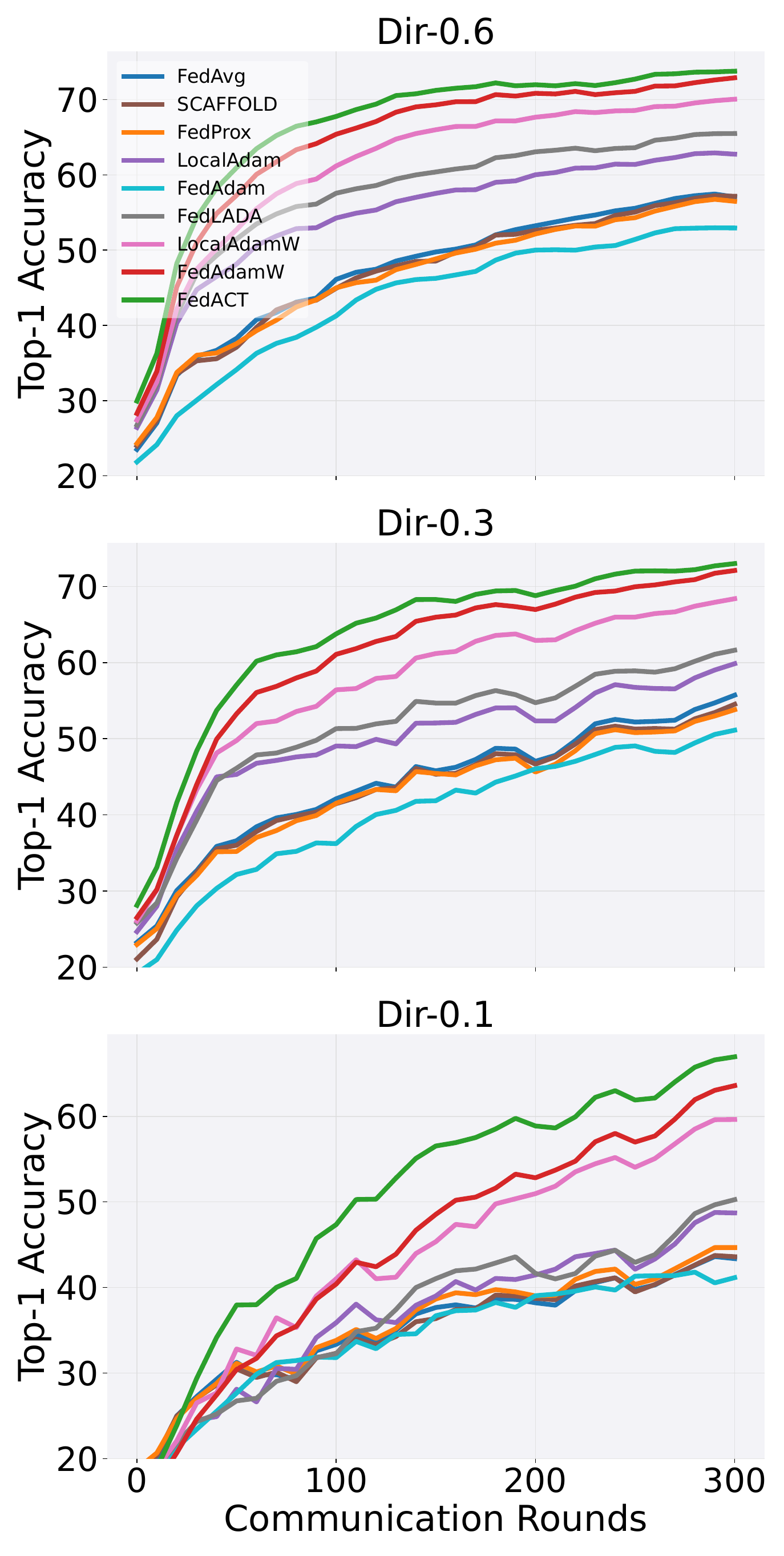}
        \caption{CIFAR-10}
    \end{subfigure}\hfill
    \begin{subfigure}{0.33\textwidth}
        \centering
        \includegraphics[width=1\textwidth]{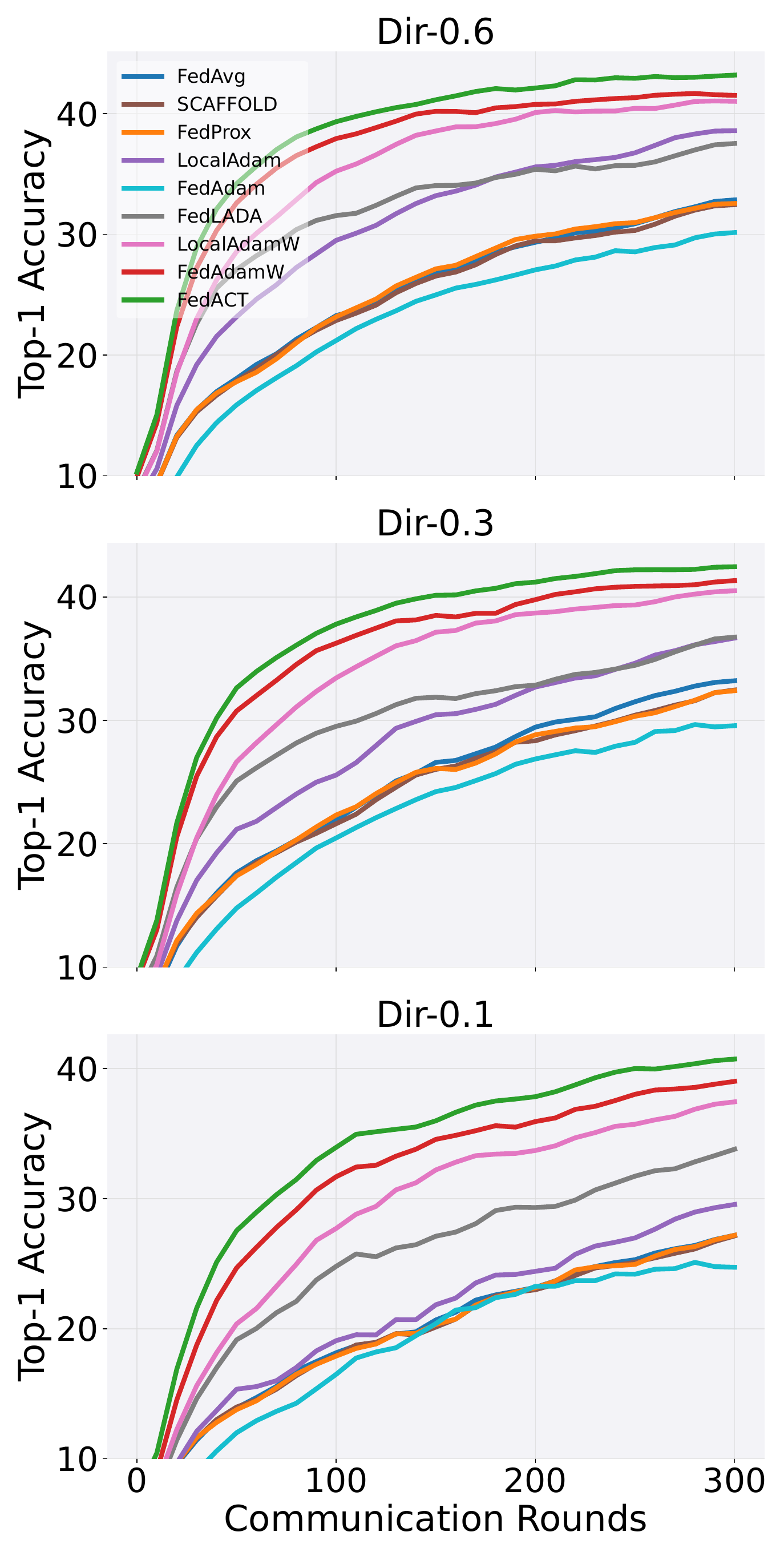}
        \caption{CIFAR-100}
    \end{subfigure}\hfill
    \begin{subfigure}{0.33\textwidth}
        \centering
        \includegraphics[width=1\textwidth]{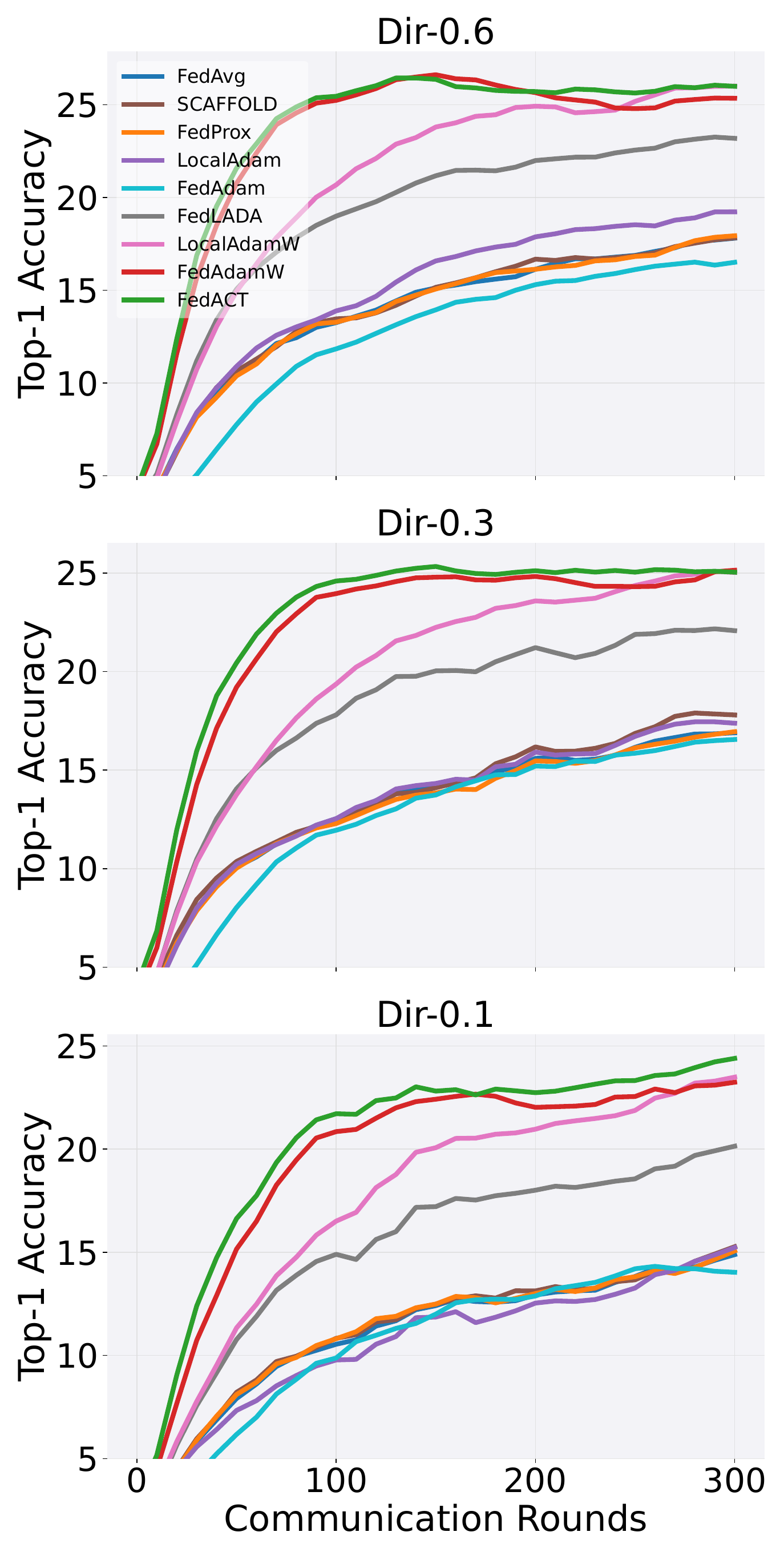}
        \caption{Tiny-ImageNet}
    \end{subfigure}
    \caption{\textbf{Test Top-1 accuracy over communication rounds for federated ViT-Tiny training.} Results are reported on CIFAR-10, CIFAR-100, and Tiny-ImageNet under Dirichlet heterogeneity levels Dir-0.6, Dir-0.3, and Dir-0.1. \method also shows improved convergence and higher final accuracy.}
    \label{fig:appendix_vit_tiny_training}
\end{figure*}

\begin{figure*}[!ht]
    \centering
    \begin{subfigure}{0.33\textwidth}
        \centering
        \includegraphics[width=1\textwidth]{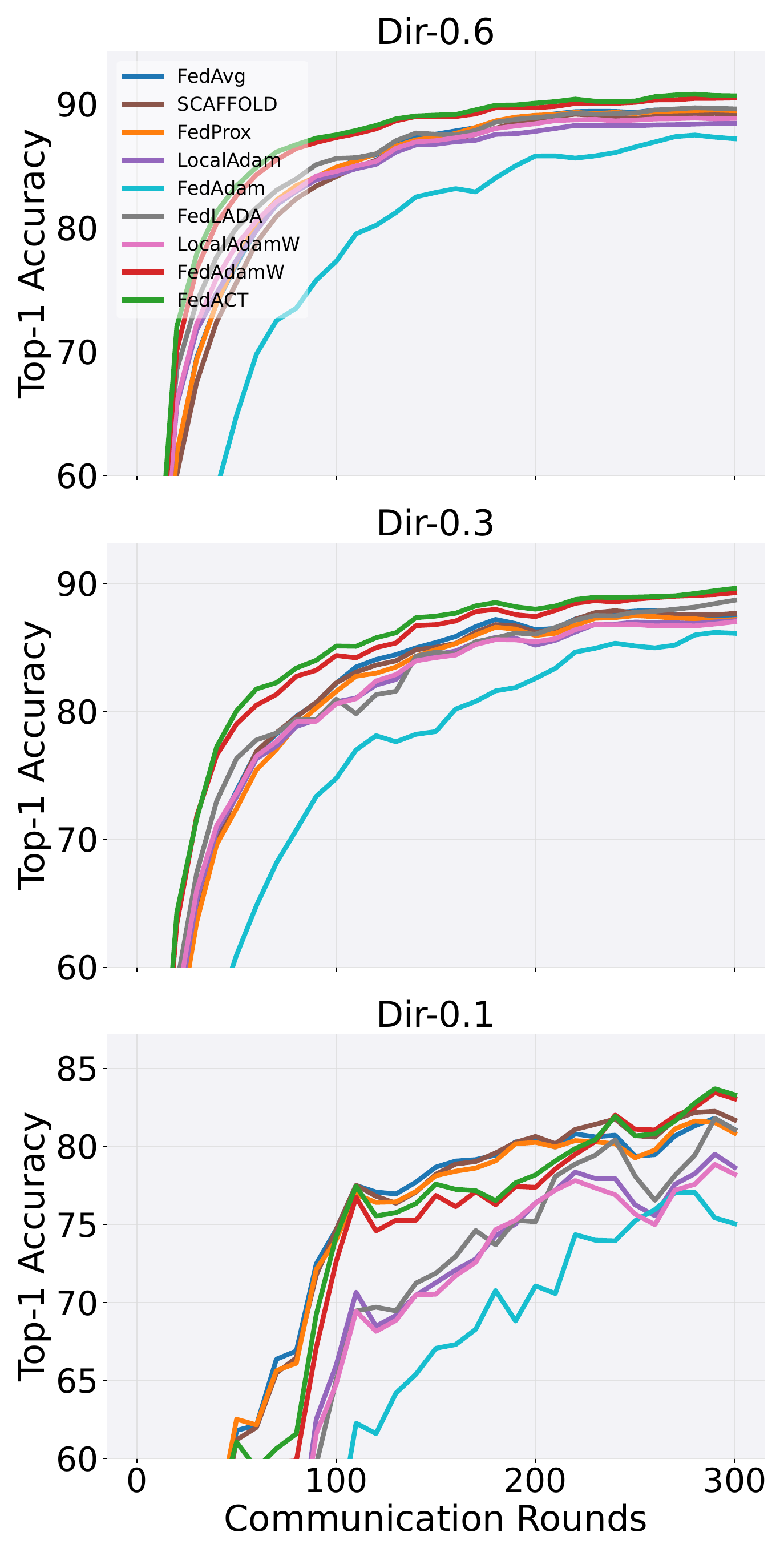}
        \caption{CIFAR-10}
    \end{subfigure}\hfill
    \begin{subfigure}{0.33\textwidth}
        \centering
        \includegraphics[width=1\textwidth]{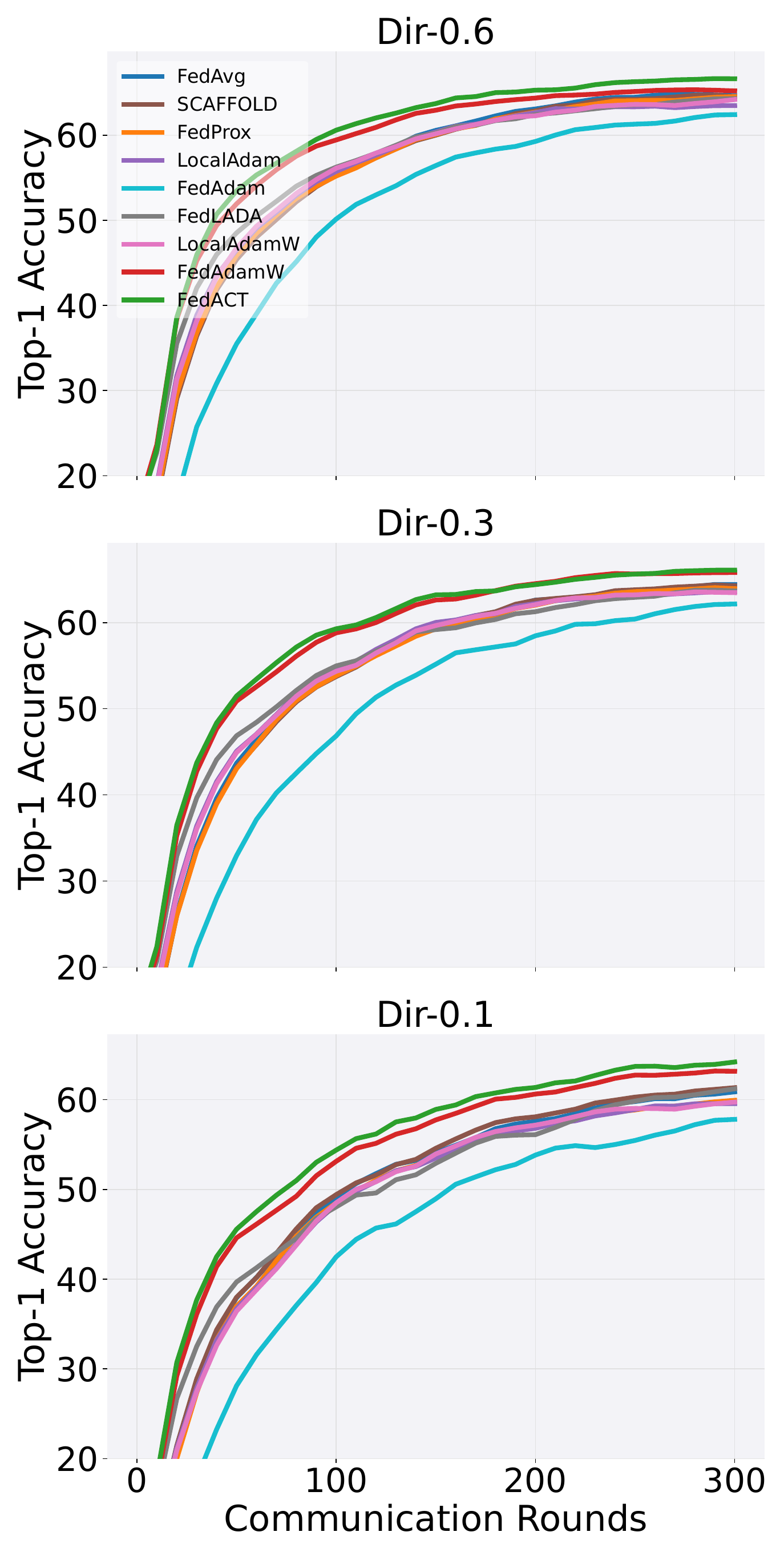}
        \caption{CIFAR-100}
    \end{subfigure}\hfill
    \begin{subfigure}{0.33\textwidth}
        \centering
        \includegraphics[width=1\textwidth]{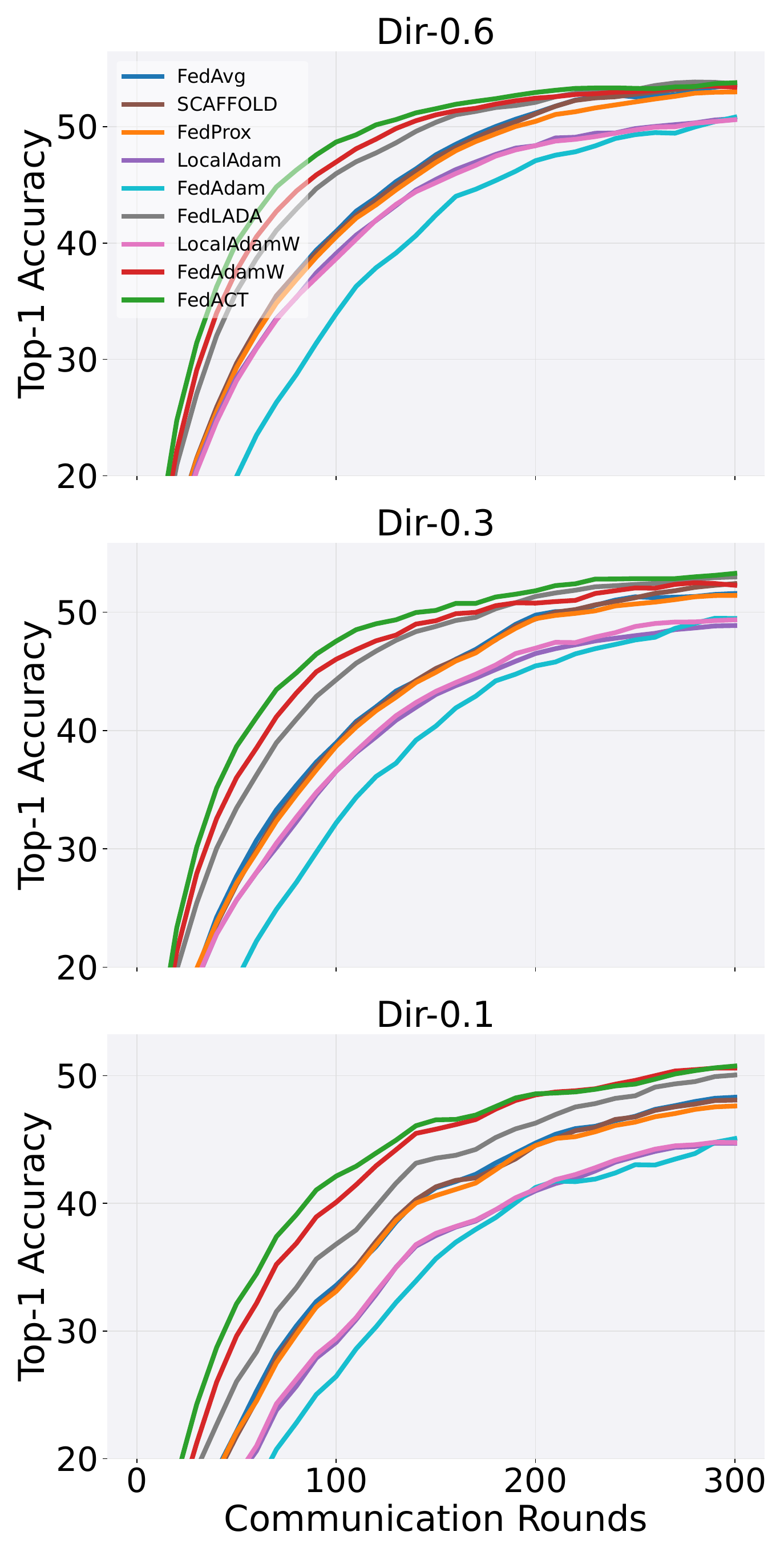}
        \caption{Tiny-ImageNet}
    \end{subfigure}
    \caption{\textbf{Test Top-1 accuracy over communication rounds for ResNet-18.} Results are reported on CIFAR-10, CIFAR-100, and Tiny-ImageNet under Dirichlet heterogeneity levels Dir-0.6, Dir-0.3, and Dir-0.1. \method remains effective on CNN training, but the improvements are generally smaller than those on vision Transformers, supporting our view that coordinate-level trust modulation is most beneficial when local adaptive updates are stronger and more heterogeneous.}
    \label{fig:appendix_resnet18_training}
\end{figure*}

\end{document}